\documentclass[11pt,a4paper]{article}
\usepackage{mathrsfs}
\usepackage{amsmath}
\usepackage{amsfonts}
\usepackage{amsthm}
\usepackage[dvipsnames]{xcolor}
\usepackage{geometry}
\usepackage{enumitem} 
\usepackage{hyperref}
\usepackage{graphicx}
\DeclareMathOperator{\Var}{Var}
\DeclareMathOperator{\Cov}{Cov}
\newcommand{\btheta}{\boldsymbol\theta}
\newcommand{\indep}{\perp \!\!\! \perp}
\newtheorem{theorem}{Theorem}
\newtheorem{lemma}{Lemma}
\newtheorem{definition}{Definition}
\newtheorem{proposition}{Proposition}
\newtheorem{remark}{Remark}

\newcounter{savecntr}

\date{\today}
\newcommand{\di}{\mathrm{d}} 
\title{Central Limit Theorem for Bayesian Neural Network trained with Variational Inference}

\author{Arnaud Descours\setcounter{savecntr}{\value{footnote}}\thanks{INRIA Lille, Magnet Team, Lille, France. E-mail: arnaud.descours@inria.fr},\, Tom Huix\setcounter{savecntr}{\value{footnote}}\thanks{Centre de Mathématiques Appliquées, UMR 7641, \'Ecole polytechnique, Palaiseau, France. E-mail: \{tom.huix,eric.moulines\}@polytechnique.edu }, \, Arnaud Guillin\setcounter{savecntr}{\value{footnote}}\thanks{Laboratoire de Math\'ematiques Blaise Pascal UMR 6620, Université Clermont-Auvergne, Aubière, France,  and Institut Universitaire de France. E-mail:  arnaud.guillin@uca.fr}, \,  Manon Michel\thanks{CNRS, Laboratoire de Math\'ematiques Blaise Pascal UMR 6620, Université Clermont-Auvergne, Aubière, France. E-mail:  manon.michel@uca.fr},\\\ \'Eric Moulines\footnotemark[2]   and Boris Nectoux\thanks{Laboratoire de Math\'ematiques Blaise Pascal UMR 6620, Université Clermont-Auvergne, Aubière, France. E-mail: boris.nectoux@uca.fr}} 

\begin{document}
\maketitle

\begin{abstract} 
  In this paper, we rigorously derive Central Limit Theorems (CLT) for
  Bayesian two-layer neural networks in the infinite-width limit and
  trained by variational inference on a regression task. The different
  networks are trained via different maximization schemes of the
  regularized evidence lower bound: (i) the idealized case with exact
  estimation of a multiple Gaussian integral from the
  reparametrization trick, (ii) a minibatch scheme using Monte Carlo
  sampling, commonly known as \emph{Bayes-by-Backprop}, and (iii) a
  computationally cheaper algorithm named \emph{Minimal VI}. The
  latter was recently introduced by leveraging the information
  obtained at the level of the mean-field limit. Laws of large numbers
  are already rigorously proven for the three schemes that admits the same
  asymptotic limit. By deriving CLT, this work shows
  that the idealized and Bayes-by-Backprop schemes have similar
  fluctuation behavior, that is different from the Minimal VI
  one. Numerical experiments then illustrate that the Minimal VI
  scheme is still more efficient, in spite of bigger variances, thanks
  to its important gain in computational complexity.
  \end{abstract}
  \tableofcontents
  
  \section{Introduction}

  Neural networks (NN), especially with a deep learning architecture, are
  one of the most powerful function approximators, in particular in a
  regime of abundant data. Their flexibility may however lead to some
  overfitting issues, which justify the introduction of a
  regularization term in the loss. Therefore, Bayesian Neural Networks
  (BNN) are an interesting alternative. Thanks to a full probabilistic
  approach, they directly model the uncertainty on the learnt weights
  through the introduction of a prior distribution, which acts as some
  natural regularization. Thus, BNN combine the expressivity power of
  NN, while showing more robustness, in particular when
  dealing with small datasets, and providing predictive uncertainty
  \cite{blundell2015weight,michelmore2020uncertainty,mcallister2017concrete,filos2019systematic}.
  During training, the probabilistic modelling however requires to
  compute integrals over the posterior distribution. This can be
  computationally demanding, as these integrals are most of the time
  not tractable. Alternative techniques as Markov-chain Monte Carlo
  methods and variational inference are most commonly used
  instead. The convergence time of the former may prove too
  prohibitively long in large-dimensional cases
  \cite{cobb2021scaling}. Therefore variational inference
  \cite{Hinton93keepingneural,mackay1995probable,mackay1995ensemble}
  comes often as the most efficient alternative, especially while
  using the reparametrization trick and the Bayes-by-backprop (BbB) approach. The
  variational approach relies on an approximation of the posterior
  distribution by the closest realization of a parametric one,
  according to a Kullback-Leibler (KL) divergence. Using a
  generalisation of the reparametrization trick \cite{kingma2014}, the
  \emph{Bayes-by-Backprop} approach \cite{blundell2015weight} leads to
  an unbiased estimator of the gradient of the ELBO, which enables
  training by stochastic gradient descent (SGD).

  There are now many successful applications of this approach, e.g.
  \cite{gal2016dropout,louizos2017multiplicative,khan2018fast}. This
  comes in contrast with the lack of analytical understanding of the
  behavior of BNN trained with variational inference, especially
  regarding their overparametrized limit. For instance, it was but
  only recently shown in \cite{colt} what is the appropriate balance
  in the ELBO of the integrated log-likelihood term and of the KL
  regularizer, in order to avoid a trivial Bayesian posterior
  \cite{izmailov2021bayesian}.  To achieve such results, a proper
  limiting theory was rigorously derived \cite{colt}. Such mean-field
  analysis, as done in
  \cite{Vanden2,chizat2018global,mei2018mean,sirignano2020lln,jmlr},
  enables the determination of the limiting nonlinear evolution of the
  weights of the NN, trained by a gradient descent or some
  variants. It then allows the derivation of a Law of Large Numbers
  (LLN) and a Central Limit Theorem (CLT). The main practical goal of
  such asymptotic analysis is to show convergence towards some global
  minimizer, it however remains an open and highly-challenging
  question. Nevertheless, such asymptotic analysis can still be of
  direct and practical relevance. On top of the proper balance in
  ELBO, it was recently shown in \cite{colt} for BNN on a regression
  task that the mean-field limit can be leveraged to develop a new SGD
  training scheme, named \emph{Minimal VI} (MiVI). Indeed, in this
  limit, the microcospic correlations between each pair of neurons can
  be shown to be equivalent to some averaged effect of the whole
  system. Therefore, the \emph{Minimal VI} scheme, which
  backpropagates only these average fields, is proven to follow the
  same LLN as standard SGD schemes, but only requires a fraction of
  the previously needed computations to recover the same limit
  behavior. Furthermore, numerical experiments showed that the
  convergence to the mean-field limit arises quite fast with the
  number of neurons ($N=300$ \cite{colt}). The \emph{Minimal VI}
  scheme would emerge as a genuinely competitive alternative under
  these conditions. However, unsurprisingly, numerical experiments
  also showed a larger variance for the \emph{Minimal VI} scheme,
  compared to others. Therefore the work presented here directly deals
  with a precise study of the fluctuation behaviors present at finite
  width $N$, as done in \cite{jmlr} for a two-layer NN, but here for
  the different variational training schemes of a BNN. Independently
  from the question of scheme comparison, the issue of quantifying the
  deviations of finite-width BNN from their infinite-width limit is of
  direct and fundamental relevance.
  
  In more details, we push on the analytical effort to further
  characterize the limiting behaviors of the three schemes and derive
  CLT. By framing the fluctuation behaviors of the different schemes,
  this work is thus of practical and direct relevance for a robust and
  efficient variational inference framework. More specifically, we
  consider a two-layer BNN trained by variational inference on a
  regression task and our contributions are as follows:
\begin{itemize}
\item We derive a CLT for the idealized SGD algorithm, where the
  variational expectations of the derivative of the loss from the
  reparametrization trick of \cite{blundell2015weight} are computed
  exactly. More precisely, we prove that with the number of neurons
  $N\to +\infty$, the sequence of trajectories of the scaled centered
  empirical distributions of the parameters satisfies a CLT, namely
  the limit satisfies a stochastic partial differential equation
  (SPDE) whose leading process is a $\mathfrak G$-process with known
  covariance structure (see Definition~\ref{d-G-process}). This is the
  first purpose of Theorem \ref{thm-clt-ideal}.
\item We derive the exact same CLT for the Bayes-by-Backprop (BbB) SGD, i.e. when the
  integrals of the idealized case are obtained by a Monte Carlo approximation, see
  \cite{blundell2015weight}. This justifies even further than at the
  LLN level the use of such an approximation procedure.
\item We derive a CLT with a  $\mathfrak G$-process of a different covariance
  structure, for the \emph{Minimal VI} (MiVI) scheme. This is the second
  purpose of Theorem \ref{thm-clt-ideal}. In comparison to the BbB
  scheme, which requires $O(N)$ Gaussian random variables and can
  become prohibitively expensive, the MiVI scheme only requires two
  Gaussian random variables and achieves the same first order
  limit. Considering scalar test function, one can show that the
  variance of the  $\mathfrak G$-process for the MiVI is greater than the one of
  the BbB.
  \item We numerically investigate the fluctuations of the three methods
  on a toy example. We observe that the scheme MiVI is
  still more efficient, as the gain in computational complexity
  outweights the increase in the observed variances.
\end{itemize}

The paper is organized as follows: Section~\ref{sec:setting} presents
the BNN setting as well as the different training algorithms,
i.e. idealized, BbB and MiVI, as well as recalls the LLN derived in
\cite{colt}, that shows their asymptotic equivalence at first order. Then, in
Section~\ref{sec:clt}, we prove for each algorithm a CLT for the
rescaled and centered empirical measure with identified covariance
based on non trivial extensions of \cite{jmlr}. Whereas, covariances
of the  $\mathfrak G$-process driving the limit SPDE may be compared, the
asymptotic variances of the rescaled centered empirical process are
not easily comparable. Therefore we produce numerical experiments in
Section \ref{sec:numerics} showing the good performance of MiVI
needing few additional neurons to get comparable variances with less
complexity. The proofs for CLT can be found in the supplementary
material.

\paragraph{Related works.}
The derivation of LLN and CLT for mean-field interacting particle
systems have garnered significant attention; refer to, for
instance,~\cite{hitsuda1986tightness,sznitman_topics_1991,
  fernandez1997hilbertian, jourdainAIHP,delarue, delmoral,
  kurtz2004stochastic} and references therein. The use of such
approaches to study the asymptotic limit of two-layer NN were
introduced in \cite{mei2018mean} (see also \cite{mei2}), which
establishes a LLN on the empirical measure of the weights at fixed
times. Formal arguments in \cite{Vanden2} led to conditions to achieve
a global convergence of Gradient Descent for exact mean-square loss
and online SGD with mini-batches. Regarding fluctuation behaviors,
they observe with increasing mini-batch size in the SGD the reduction
of the variance of the process leading the fluctuations of the
empirical measure of the weights (see \cite{Vanden2} (Arxiv-V2. Sec
3.3)). See also \cite{Vanden1} for a dynamical CLT and  \cite{durmus-neural} on
  propagation of chaos for SGD on a two-layer NN with different
  step-size schemes, however limited to finite time horizon. In
\cite{jmlr}, a LLN and CLT for the entire trajectory, and not only at
fixed times, of the empirical measure of a two-layer NN are rigorously
derived, especially when proving the uniqueness of the limit
PDE. These results are obtained for a large class of variants of SGD
(minibatches, noise), that extend in addition to rigorize the work
done in \cite{sirignano2020lln} and \cite{sirignano2020clt}. Regarding
the fluctuation behavior, the results in \cite{jmlr} agree with the
observations of \cite{Vanden2} on the minibatch impact and further
exhibit a possible particular fluctuation behavior in a large noise
regime. Finally, regarding BNN, \cite{colt} rigorously prove a LLN for
the entire trajectory for a two-layer BNN trained on a regression task
with three different schemes (idealized, BbB, MiVI).

We rigorously prove a CLT for the entire trajectory of the empirical
measure of the weights of a two-layer BNN trained by three different
maximization schemes (idealized, BbB, MiVI) of a regularized version
of ELBO. Remark that a trajectorial CLT is necessary
to understand the evolution of the variance of the scaled centered
covariance.

  \section{Setting and proven mean-field limit}\label{sec:setting}
  
   \subsection{Variational Inference and  Evidence Lower Bound} 
   In this section, we first recall the setting of Bayesian neural
   networks as well as the minimization problem in Variational
   Inference. We then introduce the three maximization algorithms of
   the ELBO and recall the respective Law of Large Numbers which were derived in
   \cite{colt}, which are the starting points of this work.

  \paragraph{The Evidence Lower Bound} 
  Let $\mathsf X$ and $\mathsf Y$ be subsets of $\mathbf R^n$
  ($n\ge 1$) and $\mathbf R$ respectively.  For $N\ge1$ and
  $\boldsymbol{w}=(w^1,\dots,w^N)\in(\mathbf R^d)^N$, we consider the
  following two-layer neural network
  $f_{\boldsymbol{w}}^N: \mathsf X\to \mathbf R$ defined by:
  \begin{equation*}
  f_{\boldsymbol{w}}^N(x):=\frac 1N\sum_{i=1}^Ns(w_i,x)\in\mathbf R,
  \end{equation*}
  where $x\in\mathsf X$ and
  $s:\mathbf R^d\times \mathsf X\to \mathbf R$ is the so-called
  activation function.  In a Bayesian setting, one needs to be able to
  efficiently sample according to the posterior distribution
  $\mathfrak P^N$ of the latent variable $\boldsymbol{w}$
  ($\boldsymbol{w}$ are the weights of the neural network). The
  classical issue in Bayesian inference over complex models is that
  the posterior distribution $\mathfrak P^N$ is quite hard to
  sample. For that reason, in variational inference, one looks for the
  closest distribution to $\mathfrak P^N$ in a family of distributions
  $\mathcal Q^N=\{ q_{\boldsymbol\theta}^N, \boldsymbol\theta\in
  \Xi^N\}$ which are much easier to sample than $\mathfrak P^N$. Here,
  $\Xi$ is the parameter space. To measure the distance between
  $ q\in \mathcal Q^N$ and $\mathfrak P^N$, one typically considers
  the KL divergence distance, denoted
  by $\mathscr D_{{\rm KL}}$ in the following.  In other words, this
  minimization problem writes:
  $${\rm argmin}_{q  \in \mathcal Q^N}\mathscr D_{{\rm KL}}(q |\mathfrak P^N) .$$ 
  \begin{sloppypar}
  \noindent
  This minimization problem is hard to solve since the KL is not
  easily computable in practice.  A routine computation shows that the
  above minimization problem, which also writes
  ${\rm argmin}_{\boldsymbol \theta\in \Xi^N}\mathscr D_{{\rm
      KL}}(q_{\boldsymbol\theta}^N|\mathfrak P^N) $, is equivalent to
  the maximization of the Evidence Lower Bound over $\theta\in
  \Xi^N$. In practice $N\gg 1$, and in this regime, it has been shown
  in \cite{coker2021wide} and \cite{huix} that optimizing the ELBO
  leads to the collapse of the variational posterior to the prior.  It
  has been suggested in \cite{huix} to rather consider a regularized
  version of the ELBO, which consists in multiplying the KL term by a
  parameter which is scaled by the inverse of the number of neurons:
  \begin{align}
  \nonumber
  \mathrm{E}_{{\rm lbo}}^N(\boldsymbol\theta,x,y) &=- \int_{(\mathbf R^d)^N}\!\!\mathfrak L\big(y,f_{\boldsymbol{w}}^N(x)\big) q_{\btheta}^N(\di w)    -\frac 1N \mathscr D_{{\rm KL}}(q_{\btheta}^N|P_0^N).
  \end{align}
  In conclusion, the maximization problem we will consider in this work is 
  $${\rm argmax}_{ \boldsymbol \theta\in \Xi^N}\mathrm{E}_{{\rm lbo}}^N(\boldsymbol\theta,x,y).$$
  \end{sloppypar}
  

  \paragraph{Loss function and prior distribution} 
  The variational family $\mathcal Q^N$ we consider is a Gaussian
  family of distributions. More precisely, it is assumed throughout
  this work that for any $\btheta=(\theta^1,\dots,\theta^N)\in\Xi^N$,
  the variational distribution $q_{\btheta}^N$ factorizes over the
  neurons: for all $\boldsymbol{w}=(w^1,\dots,w^N)\in(\mathbf R^d)^N$,
  $q_{\btheta}^N(\boldsymbol{w})=\prod_{i=1}^Nq^1_{\theta^i}(w^i)$,
  where $\theta^i=(m^i,\rho^i)\in\Xi:=\mathbf R^d\times\mathbf R$ and
  $q^1_{\theta^i}$ is the probability density function (pdf) of
  $\mathfrak N(m^i,g(\rho^i)^2 I_d)$, with
  $g(\rho)=\log(1+e^{\rho}), \ \rho \in \mathbf R$.  Let us simply
  write $\mathbf R^{d+1}$ for $\mathbf R^d\times\mathbf R$.  Following
  the reparameterisation trick of \cite{blundell2015weight},
  $q^1_\theta(w) \di w$ is the pushforward of a reference probability
  measure with density $\gamma$ by $\Psi_\theta$ (see Assumption {\rm
    \textbf{A1}}).  In practice, $\gamma$ is the pdf of
  $\mathfrak N(0,I_d)$ and $\Psi_\theta(z)=m+g(\rho)z$.  In addition,
  in all this work, we consider the regression problem, i.e.
  $\mathfrak L$ is the Mean Square Loss: for
  $\mathfrak a,\mathfrak b\in\mathbf R$,
  $\mathfrak L(\mathfrak a,\mathfrak b)=\frac 12|\mathfrak a-\mathfrak
  b|^2.$
  
  Set $\phi:(\theta,z,x)\in \mathbf R^{d+1} \times\mathbf R^d\times\mathsf X\mapsto s(\Psi_\theta(z),x)$. 
  Throughout this work, we assume that the  prior distribution $P_0^N$ is  the function defined by:
  \begin{equation}\label{eq.PriorN}
  \forall \boldsymbol{w}\in(\mathbf R^d)^N, \ P_0^N(\boldsymbol{w})=\prod_{i=1}^NP_0^1(w^j),
  \end{equation}
   where $P_0^1:\mathbf R^d\to\mathbf R_+$  is the pdf of $\mathfrak N(m_0,\sigma^2_0I_d)$, and $\sigma_0>0$. 
   With all these assumptions and notations,  we have:
  \begin{align}\label{eq.ELL}
  & \quad \mathrm{E}_{{\rm lbo}}^N(\boldsymbol\theta,x,y)= \nonumber\\
&  -\frac 12 \int \Big|y-\frac 1N\sum_{i=1}^Ns(\Psi_{\theta^i}(z^i),x)\Big|^2 \gamma(z^1)\dots\gamma(z^N)\di z_1\dots\di z_N 
   -\frac 1N \sum_{i=1}^N\mathscr D_{{\rm KL}}(q_{\theta^i}^1|P_0^1).
  \end{align}
   \begin{remark}\label{re.KL}
    We recall that \eqref{eq.PriorN} implies that $\mathscr D_{{\rm KL}}(q_{\btheta}^N|P_0^N)$ has a rather nice expression, given by:  
    $\mathscr D_{{\rm KL}}(q_{\btheta}^N|P_0^N)=\sum_{i=1}^N\mathscr D_{{\rm KL}}(q_{\theta^i}^1|P_0^1)$ and, for $\theta=(m,\rho)\in \mathbf R^{d+1}$,
  \begin{align*}
  \mathscr D_{{\rm KL}}(q_\theta^1|P_0^1)=\int_{\mathbf R^d} q^1_\theta(x) \log(q^1_\theta(x)/P_0^1(x))\di x
  =\frac{\|m-m_0\|_2^2}{2\sigma_0^2}+\frac d2\Big(\frac{g(\rho)^2}{\sigma_0^2}-1\Big)
  +\frac d2\log\Big(\frac{\sigma_0^2}{g(\rho)^2}\Big).
  \end{align*}
  We also note that $\mathscr D_{{\rm KL}}$ has at most a
  quadratic growth in $m$ and $\rho$. In addition, for
  $\theta \in \mathbf R^{d+1}$, we have
  \begin{align}\label{eq.kl_1}
  \nabla_{\theta}\mathscr D_{{\rm KL}}(q_{\theta}^1|P_0^1)=
  \begin{pmatrix}
    \nabla_{m}\mathscr D_{{\rm KL}}(q_{\theta}^1|P_0^1)    \\
      \partial_{\rho}\mathscr D_{{\rm KL}}(q_{\theta}^1|P_0^1)
  \end{pmatrix}
    =
  \begin{pmatrix}
    \frac{1}{\sigma_0^2}(m-m_0)    \\
        \frac{d}{\sigma_0^2}g'(\rho)g(\rho)-d\frac{g'(\rho)}{g(\rho)}
  \end{pmatrix}.
  \end{align}
   \end{remark}
   We assume here a Gaussian prior to get an explicit
   expression of the Kullback-Leibler divergence. Most arguments
   extend to sufficiently regular densities and are essentially the
   same for exponential families, using conjugate families for the
   variational approximation.

  \subsection{Stochastic Gradient Descent  and maximization algorithms}
  
  In this section, we present the three different maximization
  algorithms of the ELBO we are going to consider. In what
  follows, $(\Omega, \mathcal F,\mathbf P)$ is a probability space and
  we write $\langle U,\nu \rangle=\int_{\mathbf R^q}U(z)\nu(\di z) $
  for any integrable function $U:\mathbf R^q\to \mathbf R$ w.r.t. a
  measure~$\nu$ (with a slight abuse of notation, we denote by
  $\gamma$ the measure $\gamma(z)\di z$). Also we define the
  $\sigma$-algebra
  $\mathcal F_0^N=\boldsymbol\sigma(\theta_{0}^i, 1\le i\le N)$.
  
  \paragraph{Idealized SGD}
  Consider a data set $\{(x_k,y_k)\}_{k\ge 0}$ i.i.d.
  w.r.t. $\pi\in\mathcal{P}(\mathsf X\times\mathsf Y)$, the space of
  probability measures over $\mathsf X\times\mathsf Y$. For $N\ge1$
  and given a learning rate $\kappa>0$, the maximization of
  $\theta\in \mathbf R^{d+1}\mapsto \mathrm{E}_{{\rm
      lbo}}^N(\boldsymbol\theta,x,y)$ with a SGD algorithm writes as
  follows: for $k\ge 0$,
  \begin{equation}\begin{cases}\label{eq.sgd}
   &\!\!\!\!\!\!\boldsymbol\theta_{k+1}=\boldsymbol\theta_k+ \kappa \nabla_{\boldsymbol\theta}\mathrm{E}_{{\rm lbo}}^N(\boldsymbol\theta_k,x_k,y_k) \\
   &\!\!\!\!\!\!\boldsymbol\theta_0 \sim \mu_0^{\otimes N},
  \end{cases}\end{equation}
   where $\mu_0\in \mathcal P(\mathbf R^{d+1})$ (the space of probability measures over $\mathbf R^{d+1}$)  and $\boldsymbol\theta_k=(\theta^1_k,\ldots, \theta^N_k)$.
  Using the computation of $\nabla_{\boldsymbol\theta}\mathrm{E}_{{\rm lbo}}^N(\boldsymbol\theta_k,x_k,y_k) $ performed in \cite{colt}, \eqref{eq.sgd} writes: for $k\ge 0$ and $i\in\{1,\dots,N\}$, 
  \begin{equation}\begin{cases}\label{eq.algo-ideal}
   &\!\!\!\!\!\!\theta_{k+1}^i=\theta_{k}^i-\frac{\kappa}{N^2}\sum_{j=1,j\neq i}^N\Big(\langle\phi(\theta_k^j,\cdot,x_k),\gamma\rangle-y_k\Big)\langle\nabla_\theta\phi(\theta_k^i,\cdot,x_k),\gamma\rangle \\
  & \quad \quad -\frac{\kappa}{N^2}\Big\langle(\phi(\theta_k^i,\cdot,x_k)-y_k)\nabla_\theta\phi(\theta_k^i,\cdot,x_k),\gamma\Big\rangle -\frac{\kappa}{N}\nabla_{\theta}\mathscr D_{{\rm KL}}(q_{\theta^i_k}^1|P_0^1),\\
   &\!\!\!\!\!\!\theta_{0}^i \sim \mu_0.
  \end{cases}\end{equation}
We shall call this algorithm \emph{idealised} SGD because it contains
an intractable term given by the integral w.r.t. the probability
distribution $\gamma$. This has motivated the development of methods
where this integral is replaced by an unbiased Monte Carlo estimator
(see \cite{blundell2015weight}) as detailed below with the BbB SGD
scheme.  For the Idealized SGD, and for later purposes, we set for
$N\ge1$ and $k\ge1$:
  \begin{align}\label{eq.Fk1} 
  \mathcal F_k^N=\boldsymbol\sigma(\theta_{0}^i,   (x_q,y_q),1\le i\le N, 0\le q\le k-1)
  \end{align}
  \paragraph{\textit{Bayes-by-Backprop} (BbB) SGD}
  For $N\ge 1$, given a dataset $(x_k,y_k)_{k\ge0}$, the
  maximization of
  $\theta\in \mathbf R^{d+1}\mapsto \mathrm{E}_{{\rm
      lbo}}^N(\boldsymbol\theta,x,y)$ with a
  BbB SGD algorithm is the following: for
  $k\ge 0$ and $i\in\{1,\dots,N\}$,
  \begin{equation}\begin{cases}\label{eq.algo-batch}
  &\!\!\!\!\!\!\theta_{k+1}^i=\theta_k^i -\frac{\kappa}{N^2}\sum_{j=1}^N\big (\phi(\theta_k^j,\mathsf Z^{j}_{k},x_k)-y_k\big )\nabla_\theta\phi(\theta_k^i,\mathsf Z^{i}_k,x_k)
 -\frac{\kappa}{N}\nabla_\theta \mathscr D_{{\rm KL}}(q_{\theta^i_k}^1|P_0^1),\\
   &\!\!\!\!\!\!\theta_{0}^i=(m_{0}^i,\rho_{0}^i)\sim \mu_0,
  \end{cases}\end{equation}
where $(\mathsf Z^{j}_k, 1\le j\le N, k\ge 0)$ is a i.i.d sequence of
random variables distributed according to~$\gamma$.  We recall that
this algorithm is based on the Monte Carlo approximation, for
$i\in\{1,\dots,N\}$, of the term
  $$\int_{(\mathbf R^d)^N}(y-\phi(\theta^j,z^j,x))\nabla_{\theta}\phi(\theta^i,z^i,x)\gamma(z^1)\dots\gamma(z^N)\di z^1\dots\di z^N$$
  which is the gradient w.r.t. to $\theta^i$ of the integral term in
  the left-hand-side of \eqref{eq.ELL}. We mention that we consider
  here in \eqref{eq.algo-batch} the BbB SGD with a batch size of $1$,
  corresponding to $|B|=1$ in \cite{colt}.
    
    For the BbB SGD, we set for $N\ge1$ and $k\ge1$: 
  \begin{equation}\label{eq.Fk2}
  \mathcal F_k^N=\boldsymbol\sigma \Big (\theta_{0}^i ,   \mathsf Z^{j}_q,(x_q,y_q),  1\leq i,j\leq N,   0\le q\le k-1\big \} \Big ).
  \end{equation}

  
  
  \paragraph{\textit{Minimal VI} (MiVI) SGD}
  The last algorithm studied, denoted MiVI SGD, was proposed in
  \cite{colt} as an efficient alternative to the first two algorithm
  above. It is the following: for $k\ge 0$ and
  $i\in\{1,\dots,N\}$,
  \begin{align}\begin{cases}\label{eq.algo-z1z2}
  &\!\!\!\!\!\! \theta_{k+1}^i=\theta_k^i -\frac{\kappa}{N^2}\sum_{j=1}^N \big (\phi(\theta_k^j,\mathsf Z^{1}_{k},x_k)-y_k\big )\nabla_\theta\phi(\theta_k^i,\mathsf Z^{2}_k,x_k)
   -\frac{\kappa}{N}\nabla_\theta \mathscr D_{{\rm KL}}(q_{\theta^i_k}^1|P_0^1)\\
   &\!\!\!\!\!\!\theta_{0}^i=(m_{0}^i,\rho_{0}^i)\sim \mu_0,
  \end{cases}\end{align}
where $(\mathsf Z^{p}_k, p\in \{1,2\}, k\ge 0)$ is a i.i.d sequence of
random variables distributed according to $\gamma^{\otimes2}$. Thus,
the MiVI descent backpropagates through two common Gaussian variables
$(\mathsf Z^{1}_k, \mathsf Z^{2}_k)$ to all neurons, instead of a
 different Gaussian random variable $\mathsf Z^{\cdot}_k$ for each
neuron.

  We finally set 
   for $N,k\ge 1$:
  \begin{align}
  &\mathcal F_k^N =  \boldsymbol\sigma \Big (\theta_{0}^i ,   \mathsf Z^{p}_q,(x_q,y_q),  i \in [1, N], p\in \{1,2\}, q \in [0, k-1] \Big ). \label{eq.Fk3}
  \end{align}

  \subsection{Mean-field limit and Law of Large Numbers}

  \paragraph{Empirical distributions  and assumptions} 
  We introduce the empirical distribution $\nu_k^N$ of the parameters
  $\{\theta^i_k, i\in \{1,\ldots,N\}\}$ at iteration $k\ge 0$ (where
  the $\theta^i_k$'s are generated either by the algorithm
  \eqref{eq.algo-ideal}, \eqref{eq.algo-batch}, or by
  \eqref{eq.algo-z1z2}) as well as its scaled version $\mu_t^N$, which
  are defined by:
  \begin{equation}\label{empirical_distrib}
  \nu_k^N:=\frac 1N\sum_{i=1}^N\delta_{\theta_k^i} \ \ \text{and} \ \ \mu_t^N:=\nu_{\lfloor Nt\rfloor}^N.
  \end{equation}
  Note that for all $N\ge1$, $\mu^N:=\{\mu_t^N, t\ge 0\}$ is a random
  element of the Skorokhod space $\mathcal D(\mathbf R_+,\mathcal P(\mathbf R^{d+1}))$,
  when $\mathcal P(\mathbf R^{d+1})$ is endowed with the weak
  convergence topology. Let us recall that for $q\ge 0$, the
  Wasserstein spaces $\mathcal P_q(\mathbf R^{d+1})$ are defined by
  $ \mathcal P_q(\mathbf R^{d+1})= \{ \mu \in \mathcal P(\mathbf
  R^{d+1}), \int_{\mathbf R^{d+1}} |\theta|^q \mu (\di \theta)<+\infty
  \}$.  The space $\mathcal P_q(\mathbf R^{d+1})$ is endowed with the
  standard Wasserstein metric $\mathsf W_q$.
  Note that for all  $q\ge 0$,  $(\mu^N)_{N\ge1}$ is also a random sequence of elements in $\mathcal D(\mathbf R_+,\mathcal P_q(\mathbf R^{d+1}))$.
  We denote by $\mathcal C^\infty_b(\mathbf R^d\times \mathsf X)$ the space of smooth functions over $\mathbf R^d\times\mathsf X$ whose derivatives of all order are bounded.
  
  We now introduce the assumptions \cite{colt} we will work with in this work:
  \begin{enumerate}[label=\textbf{A\arabic*.}]
  \item[\textbf{A1}.]
    There exists a pdf $\gamma:\mathbf R^d\to\mathbf R_+$ such that for all $\theta\in \mathbf R^{d+1}$, $q^1_\theta\di x=\Psi_\theta\#\gamma\di x$, where $\{\Psi_\theta, \theta\in\mathbf R^{d+1}\}$ is a family of $\mathcal C^1$-diffeomorphisms over $\mathbf R^d$ such that for all $z\in\mathbf R^d$, $\theta\in\mathbf R^{d+1}\mapsto \Psi_\theta(z)$ is of class $\mathcal C^\infty$.
  Finally, there exists   $\mathfrak p_0\in \mathbf N^*$   such that for all multi-index $\alpha \in \mathbf N^{d+1}$ with $|\alpha|\ge 1$, there exists $C_\alpha>0$, for all $z\in\mathbf R^d$ and $  \theta=(\theta_1,\ldots,\theta_{d+1})\in \mathbf R^{d+1}$,
  \begin{equation}\label{jac_T_bounded}
  \big| \partial_{\alpha}\Psi_\theta(z)\big|  \leq C_{\alpha} \mathfrak b(z) \ \ \text{ with } \forall q\ge 1, \  \langle  \mathfrak b^q, \gamma\rangle <+\infty,
  \end{equation}
   where $\partial_\alpha= \partial_{\theta_1}^{\alpha_1}\ldots \partial_{\theta_{d+1}}^{\alpha_{d+1}}$ and $\partial_{\theta_j}^{\alpha_j}$  is the partial derivatives of order $\alpha_j$ w.r.t. to $\theta_j$, and $ \mathfrak b(z)=1+|z|^{\mathfrak p_0}$.  
  \item[\textbf{A2}.]
  The sequence $\{(x_k,y_k)\}_{k\ge 0}$  is  i.i.d.   w.r.t. $  \pi\in\mathcal{P}(\mathsf X\times\mathsf Y)$.
  The set $\mathsf X\times\mathsf Y\subset \mathbf R^d\times \mathbf R$ is compact. For all $k\ge0$, $(x_k,y_k)\indep \mathcal F_{k}^N$ (where, depending on the considered algorithms, $\mathcal F_{k}^N$ is defined by \eqref{eq.Fk1}, \eqref{eq.Fk2}, or \eqref{eq.Fk3}).
  \item[\textbf{A3}.]
   The  (activation) function $s:\mathbf R^d\times \mathsf X\to\mathbf R$ belongs to  $\mathcal C^\infty_b(\mathbf R^d\times \mathsf X)$.
   
  \item[\textbf{A4}.]
   The initial parameters $(\theta_{0}^i)_{i=1}^N$ are i.i.d. w.r.t. $\mu_0\in \mathcal P(\mathbf R^{d+1})$. Furthermore, $\mu_0$  has compact support.  
  \end{enumerate}
  We moreover assume when considering the BbB algorithm
  \eqref{eq.algo-batch} (resp. the MiVI algorithm \eqref{eq.algo-z1z2}):
  \begin{enumerate} 
  \item[\textbf{A5}.]
  The sequences $(\mathsf Z^{j}_k,1\leq j\leq N, k\ge 0)$ (resp. $(\mathsf Z^{p}_k, p\in \{1,2\},  k\ge 0)$) and $((x_k,y_k), k\ge 0)$ are independent. For $k\ge 0$, $\big((x_k,y_k),\mathsf Z^{j}_k, 1\leq j\leq N\big)\indep  \mathcal F_k^N$, see \eqref{eq.Fk2} (resp. $\big((x_k,y_k),\mathsf Z^{p}_k, p\in \{1,2\}\big)\indep  \mathcal F_k^N$, see \eqref{eq.Fk3}).  
  \end{enumerate}
   In the following we simply denote all the above assumptions by {\rm \textbf{A}}.  Let us remark that \textbf{A3} may seem restrictive, see however Remark 4 in \cite{jmlr} to consider a more general setting.

  \paragraph{Law of Large Numbers for  the sequence of rescaled empirical distribution}
  As already explained, the starting points to derive Central Limit Theorems for the sequence $(\mu^N)_{N\ge1}$ defined in~\eqref{empirical_distrib} for  the three algorithms introduced above are  the Law of Large Numbers obtained in  \cite{colt} (see more precisely Theorems 1, 2, and 3 there), that we now recall.

  \begin{theorem}[{\cite{colt}}]\label{th.LLN}
    Let $\gamma_0> 1+ \frac{d+1}{2}$.  Assume {\rm \textbf{A}}.  Let
    the $\{\theta^i_k, k\ge 0, i\in \{1,\ldots,N\}\}$'s be generated
    either by the algorithm \eqref{eq.algo-ideal},
    \eqref{eq.algo-batch}, or \eqref{eq.algo-z1z2}.  Then,
    $(\mu^N)_{N\ge1}$ (see~\eqref{empirical_distrib}) converges in
    $\mathbf P$-probability in
    $\mathcal D(\mathbf R_+,\mathcal P_{\gamma_0}(\mathbf R^{d+1}))$
    to a deterministic element
    $\bar\mu\in \mathcal D(\mathbf R_+,\mathcal P_{\gamma_0}(\mathbf
    R^{d+1}))$. In addition,
    $\bar\mu\in \mathcal C(\mathbf R_+,\mathcal P_{1}(\mathbf
    R^{d+1}))$ and it is the unique solution in
    $\mathcal C(\mathbf R_+,\mathcal P_{1}(\mathbf R^{d+1}))$ to the
    following measure-valued evolution equation:
    $ \forall f\in \mathcal C^\infty_b(\mathbf R^{d+1})$ and
    $\forall t\in \mathbf R_+$:
  \begin{align}
  \label{eq.P2}
  \langle f,\bar\mu_t\rangle-\langle f,\mu_0\rangle=&- \kappa\int_{0}^t\int_{\mathsf X\times\mathsf Y}\big \langle\phi(\cdot,\cdot,x)-y,\bar\mu_s\otimes\gamma\big \rangle
   \big \langle\nabla_\theta f\cdot\nabla_\theta\phi( \cdot ,\cdot,x),\bar\mu_s\otimes\gamma\big \rangle  \pi(\di x,\di y)\di s\nonumber\\
  &\quad- \kappa\int_0^t\big \langle\nabla_\theta f\cdot \nabla_\theta \mathscr D_{{\rm KL}}(q_{\,_\cdot }^1|P_0^1),\bar\mu_s\big \rangle\di s.
  \end{align}
  \end{theorem}
  Let us mention that the statement of Theorem~\ref{th.LLN} differs
  slightly from the one of Th. 2 in \cite{colt} when the Idealized
  SGD~\eqref{eq.algo-ideal} is concerned. Since this was possible, we
  have decided here to work in
  $\mathcal P(\Theta)$ (with $\Theta\subset \mathbf R^{d+1}$ compact)
  instead of $\mathcal P_{\gamma_0}(\mathbf R^{d+1})$. Nevertheless,
  Theorem 3 in \cite{colt}, by following its proof, also holds for the
  scaled empirical measure $\mu^N$ of the parameters $\theta_k^i$'s
  generated by the Idealized SGD~\eqref{eq.algo-ideal}.
  
  %
  %
  %
  %
  %

  \section{Main results: Central Limit Theorems} \label{sec:clt}

  For $ \mathfrak J\in\textbf{N}$ and $\mathfrak j\geq0$, let  $\mathcal H^{\mathfrak J,\mathfrak  j}(\mathbf R^{d+1})$ be the closure of the set
  $\mathcal C_c^\infty(\mathbf R^{d+1})$ for the norm $\|f\|_{\mathcal H^{\mathfrak J, \mathfrak  j}}$ defined by 
  $$\|f\|_{\mathcal H^{\mathfrak J,\mathfrak  j}}^2= \sum_{|k|\leq \mathfrak J}\int_{\mathbf R^{d+1}}\frac{|\partial_kf(\theta)|^2}{1+|\theta|^{2\mathfrak j}}\di \theta .$$
   The space
  $\mathcal H^{\mathfrak J,\mathfrak  j}(\mathbf R^{d+1})$ was introduced  e.g.  in \cite{fernandez1997hilbertian,jourdainAIHP}. It  is a separable  Hilbert space. Its dual space is denoted by 
  $\mathcal H^{-\mathfrak J,\mathfrak  j}(\mathbf R^{d+1})$.
  The associated
  scalar product on $\mathcal H^{\mathfrak J,\mathfrak j}(\mathbf R^{d+1})$ will be denoted
  by $\langle\cdot,\cdot\rangle_{\mathcal H^{\mathfrak J,\mathfrak j}}$. For
  $\Phi \in \mathcal H^{-\mathfrak J,\mathfrak j}(\mathbf R^{d+1})$, we use the notation $\langle f,\Phi\rangle_{\mathfrak J,\mathfrak j}= \Phi[f], \ f\in \mathcal H^{\mathfrak J,\mathfrak j}(\mathbf R^{d+1})$. 
  We will simply 
  denote $\langle f,\Phi\rangle_{J,\beta}$ by $\langle f,\Phi\rangle$ when no confusion is possible.
  The set $\mathcal C^{\mathfrak J,\mathfrak j}(\mathbf R^{d+1})$ is defined as the space of
  functions $f:\mathbf R^{d+1}\rightarrow\mathbf{R}$ which have continuous
  partial derivatives up to the order $\mathfrak J\in\textbf{N}$ and satisfy, for all $|k|\le \mathfrak J$,   $\frac{|\partial_kf(\theta)|}{1+|\theta|^{\mathfrak j}}\to 0$ as  $|\theta|\to +\infty$. It is endowed with the norm 
   $\|f\|_{\mathcal C^{\mathfrak J,\mathfrak j}}:=\sum_{|k|\leq \mathfrak J}\ \sup_{\theta\in\mathbf R^{d+1}}\frac{|\partial_kf(\theta)|}{1+|\theta|^{\mathfrak j}}<+\infty$. 
   We denote by    $x\mapsto  \lceil x\rceil$ the ceiling function and we finally set:  $$\mathfrak j_3= \lceil \frac{d+1}{2}\rceil +1\text{ and } \mathfrak J_3 = 4\lceil \frac{d+1}{2}\rceil+8.$$ 

  \noindent
  The fluctuation process is defined by 
  \begin{equation}\label{def:eta}
    \eta^N: t \in \mathbf R_+\mapsto  \sqrt N(\mu^N_t-\bar\mu_t) \; ,
  \end{equation}
  where $\mu^N$ is defined in~\eqref{empirical_distrib}  and $\bar\mu_t$ is its limiting process, see  Theorem \ref{th.LLN}. We will show below that the three  fluctuation processes converge in law  to a limiting process which is the unique (weak) solution an equation (namely Equation  {\rm \textbf{(EqL)}} below). The  equation  {\rm \textbf{(EqL)}}  is fully characterizes  by the covariance structure of a so-called \textit{$\mathfrak G$-process}, a process we introduce now.

  \begin{definition}\label{d-G-process}
  We say that a $ \mathcal C(\mathbf R_+,\mathcal H^{-\mathfrak J_3,\mathfrak j_3}(\mathbf R^{d+1}))$-valued  process $\mathscr G$ is a $\mathfrak G$-process if for all $k\ge1$ and all  $f_1,\dots, f_k\in\mathcal H^{\mathfrak J_3,\mathfrak j_3}(\mathbf R^{d+1})$, $\{t\in\mathbf R_+\mapsto(\mathscr G_t[f_1],\dots,\mathscr G_t[f_k])^T\} $ is a $\mathcal C(\mathbf R_+,\mathbf R^k)$-valued  process with zero-mean, independent Gaussian increments (and thus a martingale) and with covariance structure prescribed by $\Cov(\mathscr G_t[f_i],\mathscr G_s[f_j])$, for $0\le s\le t$. 
  \end{definition}
  
  We mention that two $\mathfrak G$-processes are equal in law if and only if they have the same covariance structure (see~\cite{jmlr}). 
  For a  $\mathfrak G$-process  $\mathscr G\in  \mathcal C(\mathbf R_+,\mathcal H^{-\mathfrak J_3,\mathfrak j_3}(\mathbf R^{d+1}))$,  
  we say that a  $\mathcal C(\mathbf R_+,\mathcal H^{-\mathfrak J_3+1,\mathfrak j_3}(\mathbf R^{d+1}))$-valued process $\eta$ is a  solution of 
  {\rm \textbf{(EqL)}} if it satisfies a.s.  the equation:
  \begin{align*}
  &\forall  f\in\mathcal H^{-\mathfrak J_3,\mathfrak j_3-1}(\mathbf R^{d+1}),   \forall t\in\mathbf R_+, \\
  &\langle f,\eta_t\rangle-\langle f,\eta_0\rangle
  =-\kappa\int_0^t\int_{\mathsf X\times\mathsf Y}\langle\phi(\cdot,\cdot,x)-y,\bar\mu_s\otimes\gamma\rangle
   \langle\nabla_\theta f\cdot\nabla_\theta\phi(\cdot,\cdot,x),\eta_s\otimes\gamma\rangle\pi(\di x,\di y)\di s\nonumber\\
  &\ \quad\quad\quad\quad\quad\quad\quad -\kappa\int_0^t\int_{\mathsf X\times\mathsf Y}\langle\phi(\cdot,\cdot,x),\eta_s\otimes\gamma\rangle\langle\nabla_\theta f\cdot\nabla_\theta\phi(\cdot,\cdot,x),\bar\mu_s\otimes\gamma\rangle\pi(\di x,\di y)\di s  \quad\quad\quad\mathbf{(EqL)}\nonumber\\
  &\ \quad\quad\quad\quad\quad\quad\quad-\kappa\int_0^t\langle\nabla_\theta f\cdot\nabla_\theta\mathscr D_{\mathrm{KL}}(q^1_\cdot|P_0^1),\eta_s\rangle\di s +\mathscr G_t[f]. 
  \end{align*}
  
  We now define,   as in the classical theory of stochastic differential equations (see \cite{kallenberg2006foundations}), the notion of weak solution of {\rm \textbf{(EqL)}}.

  \begin{definition}
  Let $\nu$ be a $\mathcal H^{-\mathfrak J_3+1,\mathfrak j_3}(\mathbf R^{d+1})$-valued random variable.  We say that weak existence holds for {\rm \textbf{(EqL)}} with initial distribution $\nu$  if:  there exist  a probability space $\mathscr P$,  a  process $\eta\in \mathcal C(\mathbf R_+,\mathcal H^{-\mathfrak J_3+1,\mathfrak j_3}(\mathbf R^{d+1}))$  and a $\mathfrak G$-process $\mathscr G$ on $\mathscr P$ satisfying  {\rm \textbf{(EqL)}}   with in addition   $\eta_0=\nu$ in law. In this case, we will simply say that $\eta$ is a weak solution of {\rm \textbf{(EqL)}}.  In addition, we say that weak uniqueness holds if for any two weak solutions $\eta^\circ$ and $\eta^\star$ of {\rm \textbf{(EqL)}}  with the same initial distributions, it holds $\eta^\circ=\eta^\star$ in law. 
  \end{definition}
  
  We are now in position to state  the main theoretical result of this work:  Central Limit Theorems for the trajectory of the scaled empirical measures  $\mu^N$ of  the  $\{\theta^i_k, i\in \{1,\ldots,N\}\}$'s  generated either  by  the algorithm \eqref{eq.algo-ideal}, \eqref{eq.algo-batch}, or by \eqref{eq.algo-z1z2}.
  

  \begin{theorem}\label{thm-clt-ideal} Assume  {\rm \textbf{A}}. Then, 
  \begin{enumerate}
  \item  The sequence $(\eta^N)_{N\ge1}$  converges in distribution in $\mathcal D(\mathbf R_+,\mathcal H^{-\mathfrak J_3+1,\mathfrak j_3}(\mathbf R^{d+1}))$  to a $\mathcal C(\mathbf R_+,\mathcal H^{-\mathfrak J_3+1,\mathfrak j_3}(\mathbf R^{d+1}))$-valued process $\eta^\star$. 
  \item  The process $\eta^\star$ is the unique weak solution of {\rm \textbf{(EqL)}} with initial distribution $\nu_0$, where $\nu_0$ is the unique (in distribution) $\mathcal H^{-\mathfrak J_3+1,\mathfrak j_3}(\mathbf R^{d+1})$-valued random variable such that for all $k\ge1$  and $f_1,\dots,f_k\in \mathcal H^{J_3-1,j_3}(\mathbf R^{d+1})$, $(\langle f_1,\nu_0\rangle,\dots,\langle f_k,\nu_0\rangle)^T\sim\mathfrak N(0,\mathfrak C(f_1,\dots,f_k))$, where $\mathfrak C(f_1,\dots,f_k)$ is the covariance matrix of  $(f_1(\theta_0^1),\dots,f_k(\theta_0^1))^T$.   Moreover, the $\mathfrak G$-process $\mathscr G$ has covariance structure given by, for all $ f,g\in\mathcal H^{\mathfrak J_3,\mathfrak j_3}(\mathbf R^{d+1})$ and all  $ 0\le s\le t$:
  \begin{itemize}
  \item When the $\{\theta^i_k, i\in \{1,\ldots,N\}\}$'s are generated
    by the idealized algorithm \eqref{eq.algo-ideal} or by the BbB
    algorithm~\eqref{eq.algo-batch},
  \begin{align*}
  &\Cov(\mathscr G_t[f],\mathscr G_s[g])=\eta^2\!\!\!\int_0^s\Cov(\mathscr Q[f](x,y,\bar\mu_v),\mathscr Q[g](x,y,\bar\mu_v))\di v,
  \end{align*}
  where $\mathscr Q[f](x,y,\bar\mu_v)=\langle\phi(\cdot,\cdot,x)-y,\bar\mu_v\otimes\gamma\rangle\langle\nabla_\theta f\cdot\nabla_\theta\phi(\cdot,\cdot,x),\bar\mu_v\otimes\gamma\rangle$.
  
\item When the $\{\theta^i_k, i\in \{1,\ldots,N\}\}$'s are generated
  by the MiVI algorithm \eqref{eq.algo-z1z2},
  \begin{align*}
  &\Cov(\mathscr G_t[f],\mathscr G_s[g])=\eta^2\!\!\!\int_0^s\!\!\!\Cov(\mathscr Q[f](x,y,z^1,z^2,\bar\mu_v), \mathscr Q[g](x,y,z^1,z^2,\bar\mu_v))\di v,
  \end{align*}
  where   $\mathscr Q[f](x,y,z^1,z^2,\bar\mu_v)=\langle\phi(\cdot,z^1,x)-y,\bar\mu_v\rangle\langle\nabla_\theta f\cdot\nabla_\theta\phi(\cdot,z^2,x),\bar\mu_v\rangle$.
  \end{itemize}
  \end{enumerate} 
  \end{theorem}
  Let us begin by the following remark: when $f=g$ it follows
  directly from Jensen's inequality that the variance of the
  $\mathfrak G$-process leading the limiting SPDE of the CLT of the
  \emph{Minimal VI} algorithm is greater than the corresponding
  variance of the \emph{BbB} algorithm. It is however not clear if
  this hierarchy is conserved through the SPDE. However numerical
  experiments presented in Section \ref{sec:numerics} tend to this
  conclusion.
   
    The strategy of the proof of Theorem \ref{thm-clt-ideal} is the same whenever one considers that the    $\{\theta^i_k, i\in \{1,\ldots,N\}\}$'s are generated by \eqref{eq.algo-ideal}, \eqref{eq.algo-batch} or \eqref{eq.algo-z1z2}, except for the convergence of the martingale   sequence $(\sqrt N\mathbf M^N)_{N\ge1}$ towards a $\mathfrak G$-process which requires more inlvolved   analysis (see more precisely Section \ref{sec-conv-G-process}).  Appendix \ref{sec-proof-clt} below is dedicated to the detailed  proof of the Central Limit Theorem when  the    $\{\theta^i_k, i\in \{1,\ldots,N\}\}$'s are generated by \eqref{eq.algo-batch}. The other two cases are treated very similarly except, as already mentioned,  the convergence of the martingale term towards a $\mathfrak G$-process, which  is therefore proved for each of the three algorithms  in Section \ref{sec-conv-G-process}. The  proof  of  Theorem \ref{thm-clt-ideal} is inspired by the one made for  Th. 2 in \cite{jmlr}. Nonetheless, two   difficulties arise in the proof of Theorem \ref{thm-clt-ideal} compared   to \cite{jmlr}. The first one comes from the fact that  the term $\nabla_\theta \mathscr D_{{\rm KL}}(q_{\theta}^1|P_0^1)$, appearing in all of the three algorithms, is not bounded in $\theta$ (see indeed \eqref{eq.kl_1}). 
    The second difficulty deals with the convergence of the martingale   sequence $(\sqrt N\mathbf M^N)_{N\ge1}$, defined in  \eqref{eq.MN},   when the $\{\theta^i_k, i\in \{1,\ldots,N\}\}$'s are generated by \eqref{eq.algo-batch}. In this case, we have to introduce and study the convergence of the empirical distribution of both the $\{\theta^i_k, i\in \{1,\ldots,N\}\}$'s and the  $\mathsf Z^i$'s (see  \eqref{new_empir_distrib} and Lemma \ref{prop-conv_rhoN'}).


\section{Numerical simulations} \label{sec:numerics}

In this section, we begin by illustrating Theorem \ref{thm-clt-ideal}
of this paper, followed by a comparative analysis between MiVI SGD
algorithm and its two counterparts, idealized (I-SGD) and BbB SGD.

For our experimental setup, we draw uniformly the input data $x \sim \mathcal{U}([-1, 1]^{d_{in}})$. Then, the output data is given by $y = \text{tanh}(\langle x, w_{in}^{\star} \rangle) \cdot w_{out}^{\star} + \gamma \cdot \epsilon$. Here, $\gamma \in \mathbf R$ represents the noise level and $\epsilon \sim \mathcal{N}(0, \mathrm{I}_{d_{in}})$ is the Gaussian noise. Therefore, we are trying to learn the noisy prediction of a two-layer Neural Network with an hyperbolic tangent activation function. The true parameters of this network are defined by $w_{in}^{\star} \in \mathbf R^{d_{in}}$ and $w_{out}^{\star} \in \mathbf R^{d_{out}}$. These true parameters are initialized randomly, sampled from a standard Gaussian distribution. 

We consider two distinct settings in our evaluation. The first is a noiseless and low-dimensional scenario with parameters set to $\gamma = 0$, $d_{in} = 10$, and $d_{out} = 1$. In contrast, the second setting is more complex, involving noise with $\gamma = 1$, and higher dimensions with $d_{in} = 50$ and $d_{out} = 10$. 

For all algorithms (MiVI-SGD, BbB-SGD, and I-SGD), the prior
distribution is
$P_0^N = \mathcal{N}(0, \mathrm{I}_{N \times (d_{in} +
  d_{out})})$. The variational parameters $\mathbf{\theta}$ are
randomly initialized, centered around the prior distribution. Since
the I-SGD cannot be implemented due to intractable integral
calculation, we approximate it using Monte Carlo with a mini-batch of
100. For the algorithm BbB-SGD, we set the number of Monte Carlo
samples to $1$. The number of gradient descent steps used by all
algorithms is set to $\lfloor t \cdot N \rfloor$, where $t = 10$ for
the simple setting. However, due to computational limitations, we set
$t = 3$ for the complex setting.  For all experiments, we consider
three different test functions. If $\theta = (m, \rho)$, we define
$f_{mean}(\theta) = \Vert m \Vert_2$,
$f_{std}(\theta) = |g(\rho)|$, and
$f_{pred}(\theta) = \hat{\mathbb{E}}_{x} \Big[ \hat{\mathbb{V}}_{w
  \sim q_{\theta}^1}[s(w,x)]^{\frac{1}{2}}\Big]$. Here,
$\hat{\mathbb{E}}$ and $\hat{\mathbb{V}}$ represent the empirical mean
and variance over 100 samples, respectively. These functions are used
to compute $\langle f ,\; \mu_t^N \rangle$ and
$\langle f,\; \eta_t^N \rangle$.

\paragraph{Illustration of Theorem \ref{thm-clt-ideal} :} Using the
definition of $\eta_t^N$ in equation \ref{def:eta}, and that
$\bar{\mu_t}$ is deterministic, then we deduce that
$\mathbb{V}[\langle f, \; \eta_t^N \rangle ] = N \cdot
\mathbb{V}[\langle f , \; \mu_t^N \rangle]$. Figure \ref{fig:eta}
displays the convergence of
$N \cdot \mathbb{V}[\langle f , \; \mu_t^N \rangle]$ in the simple and
complex setting. The variance is estimated using its empirical version
with 300 samples, and the 95\% confidence interval is calculated based
on 10 samples. These plots clearly show that the $\mathfrak G$-process
associated with the limiting fluctuation process $\eta_t$ derived from
BbB-SGD shares the same covariance as the one derived from I-SGD, but
differs from the covariance derived from MiVI-SGD, which exhibit
larger values. These plots clearly illustrates the main result of
Theorem \ref{thm-clt-ideal} and the following remark.

\paragraph{Comparison MiVI-SGD, BbB-SGD and I-SGD:} The objective of this paragraph is to compare, at a fixed number of neurons $N$, the performances of algorithms MiVI-SGD, BbB-SGD and I-SGD. 
Recall that algorithm BbB-SGD randomly samples $N$ Gaussian vectors of dimension $d_{in} + d_{out}$ at each training step. Consequently, during the full training, this algorithm samples $\lfloor t \cdot N \rfloor N$ Gaussian vectors. In contrast, MiVI-SGD samples only $2$ Gaussian vectors per training step, resulting in a total of $2 \lfloor t \cdot N \rfloor$ sampled Gaussian vectors. Therefore, algorithm MiVI-SGD becomes more suitable (in terms of the number of Gaussian vectors sampled) for $N \geq 2$.
Figure \ref{fig:mu} show the variance of $\langle f, \; \mu_t^N \rangle$ with respect to $N$, in the simple and complex setting. 
Similarly to the previous paragraph, the variance is estimated using 300 samples, and the 95\% confidence interval is computed based on 10 samples.

\begin{figure}
     \centering
         \includegraphics[scale=0.4]{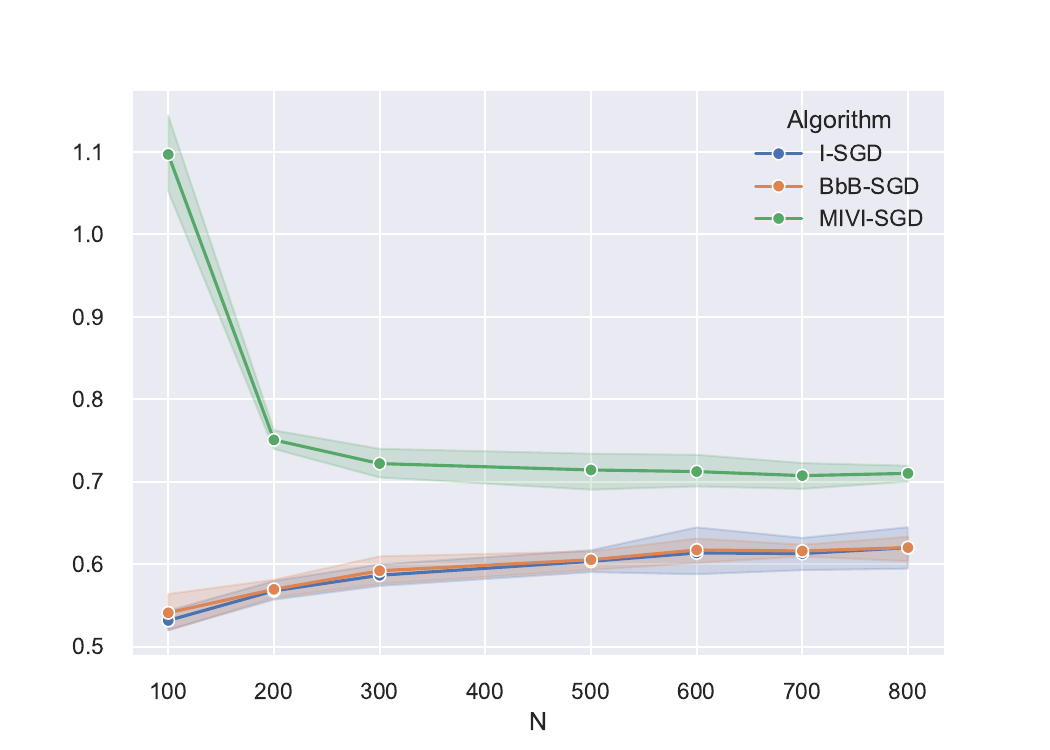}
         \includegraphics[scale=0.4]{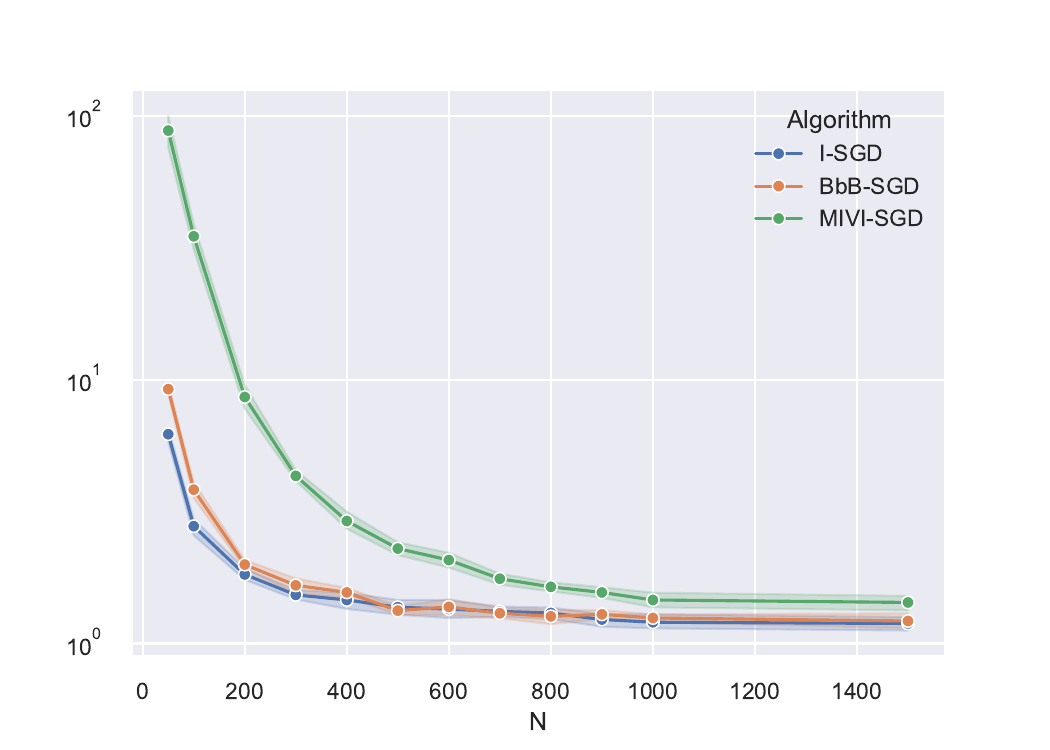}\\
         \includegraphics[scale=0.4]{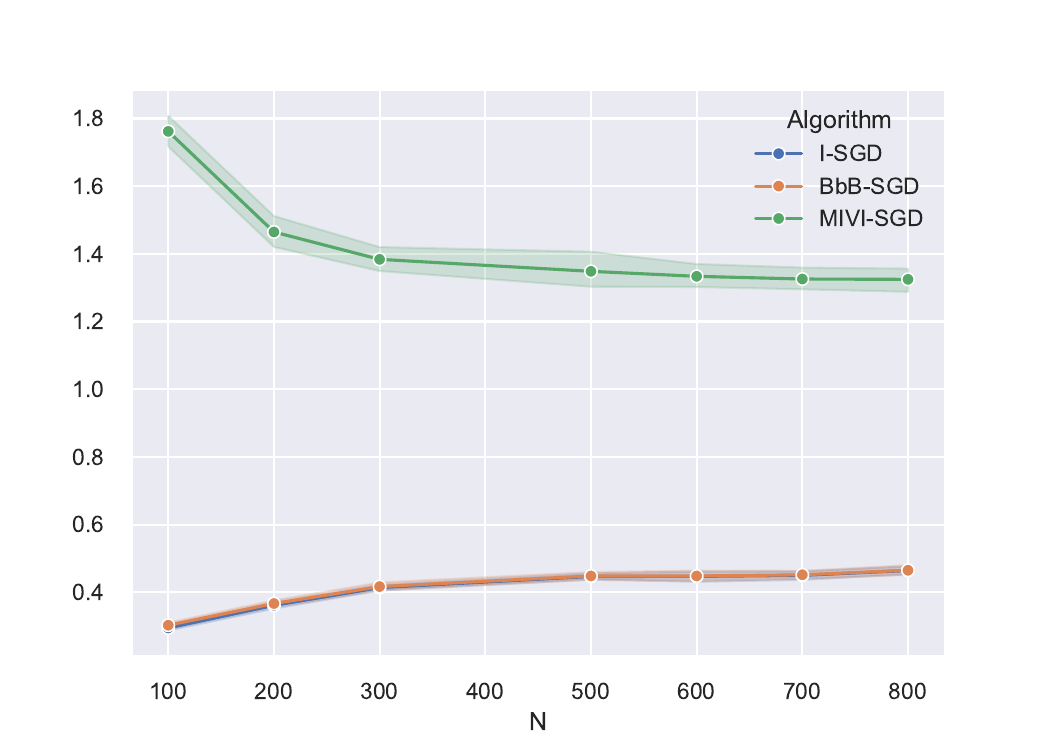}
         \includegraphics[scale=0.4]{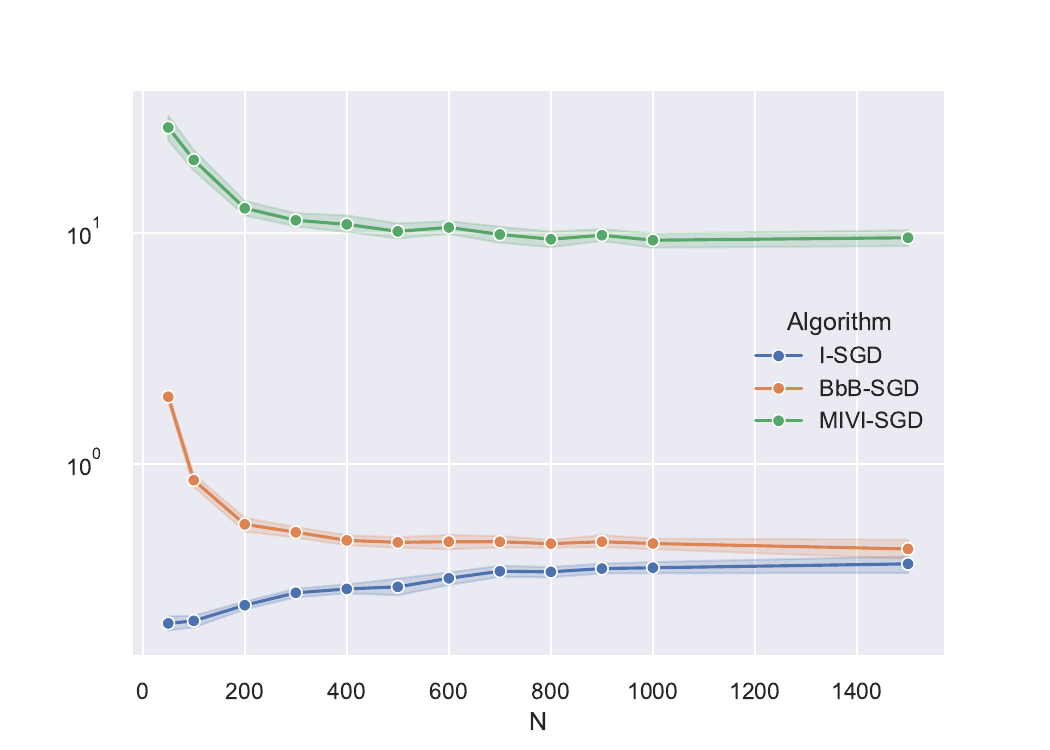}\\
         \includegraphics[scale=0.4]{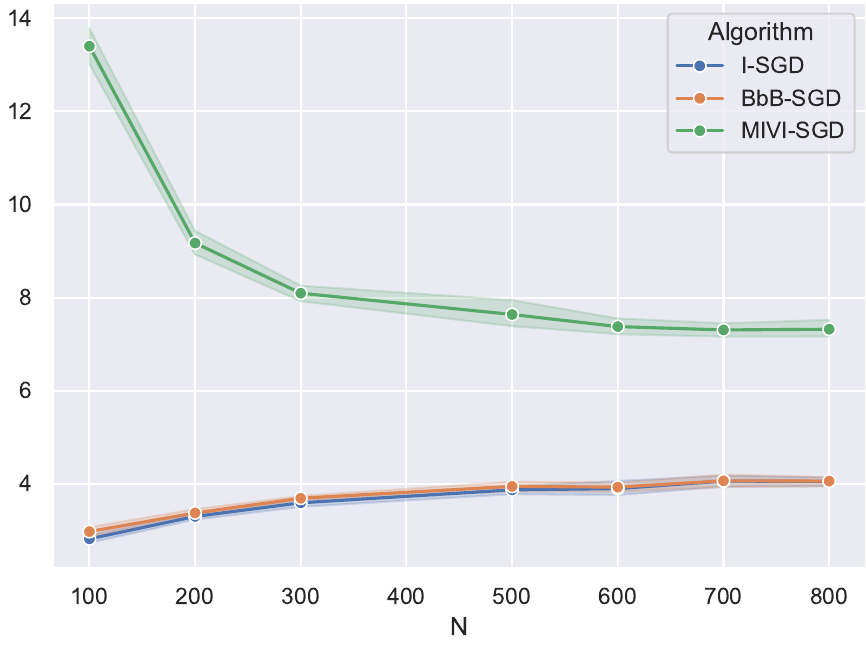}
         \includegraphics[scale=0.4]{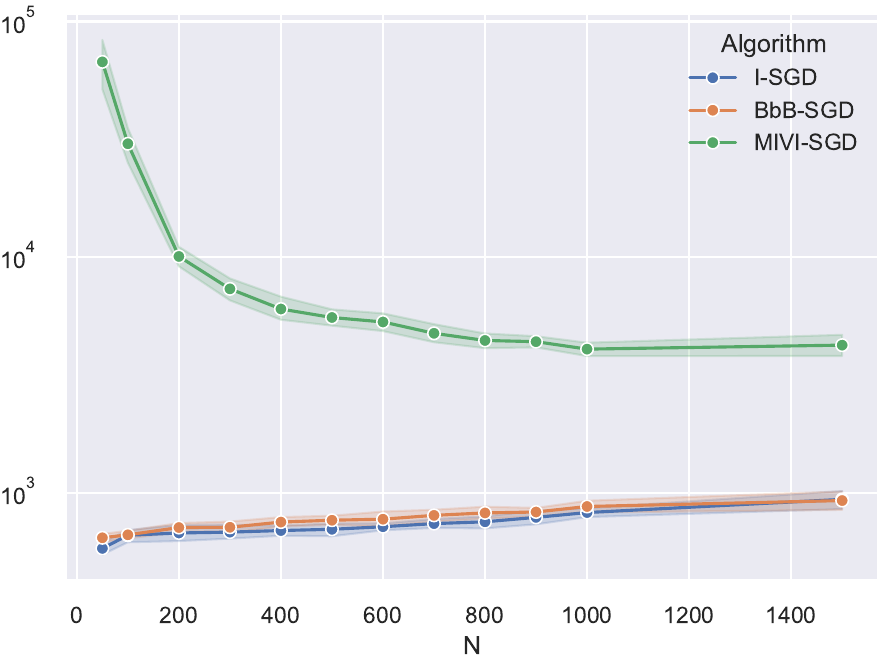}
        \caption{Convergence of $ \mathbb{V}[\langle f, \eta_t^N \rangle]$ in the simple (left column) and complex (right column) setting, for $f_{mean}$ ($1^{st}$ line), $f_{std}$ ($2^{nd}$ line) and $f_{pred}$ ($3^{rd}$ line).}
    \label{fig:eta}
\end{figure}

\begin{figure}
  \centering
      \includegraphics[scale=0.4]{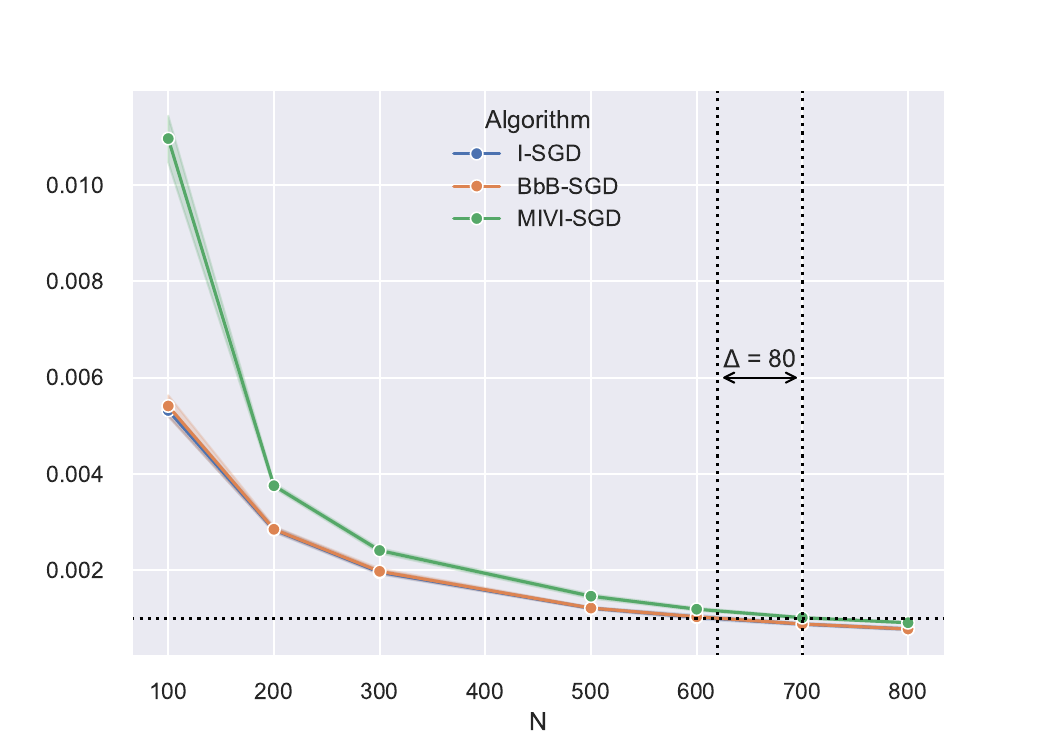}
      \includegraphics[scale=0.4]{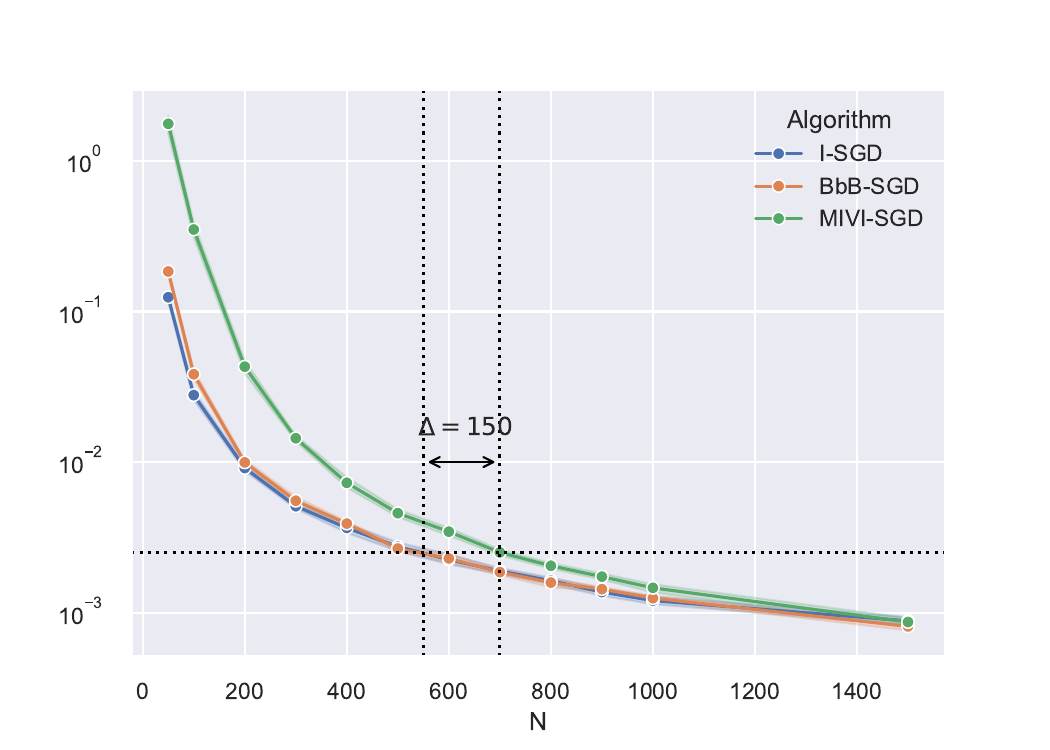}\\
      \includegraphics[scale=0.4]{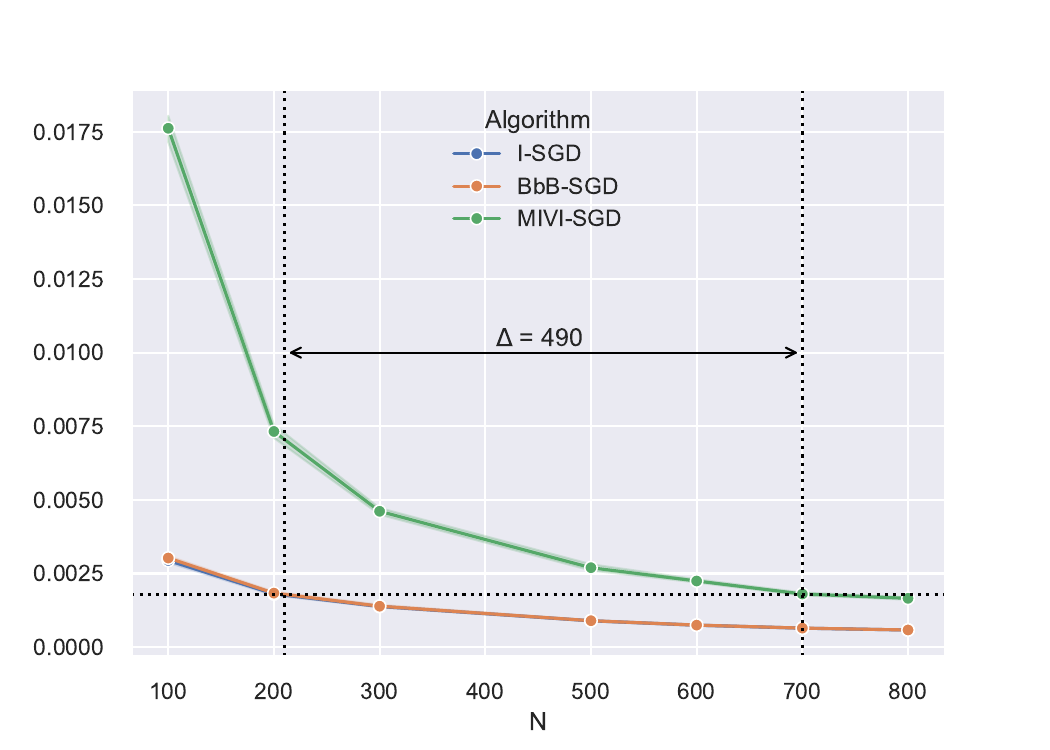}
      \includegraphics[scale=0.4]{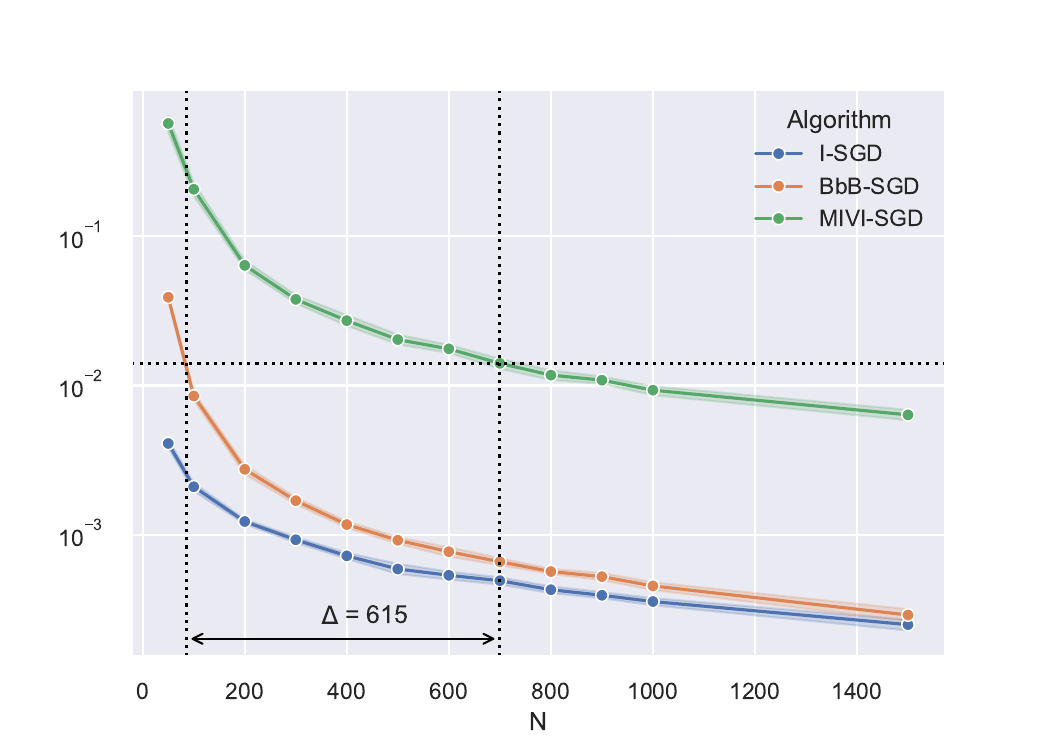}\\
      \includegraphics[scale=0.4]{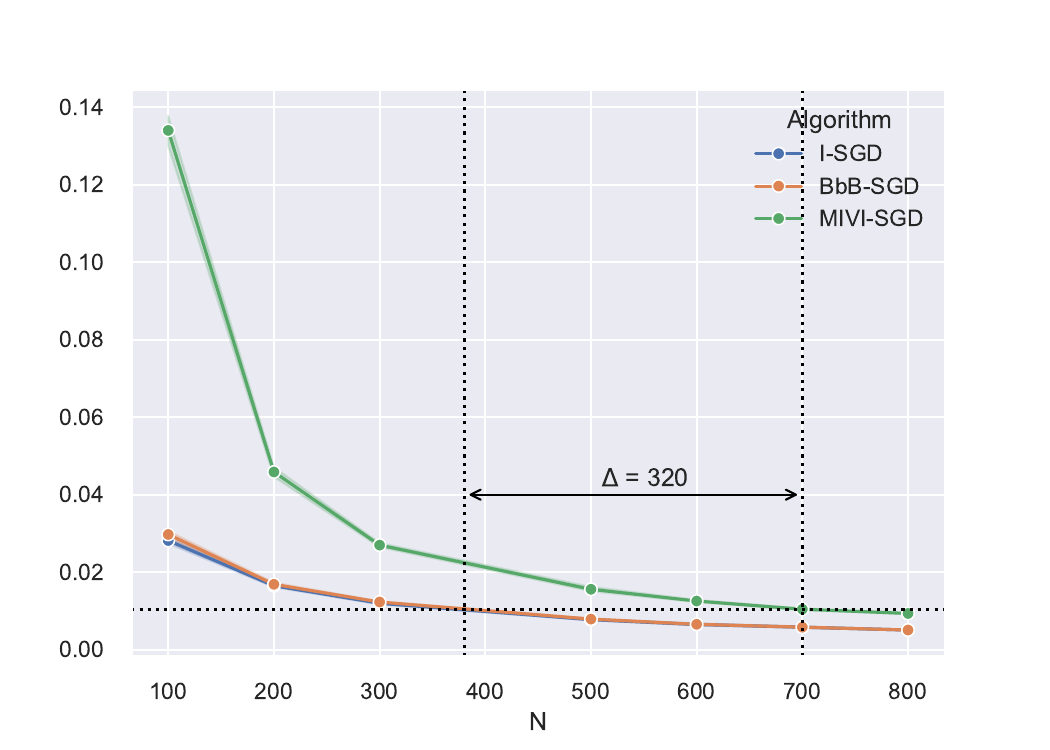}
      \includegraphics[scale=0.4]{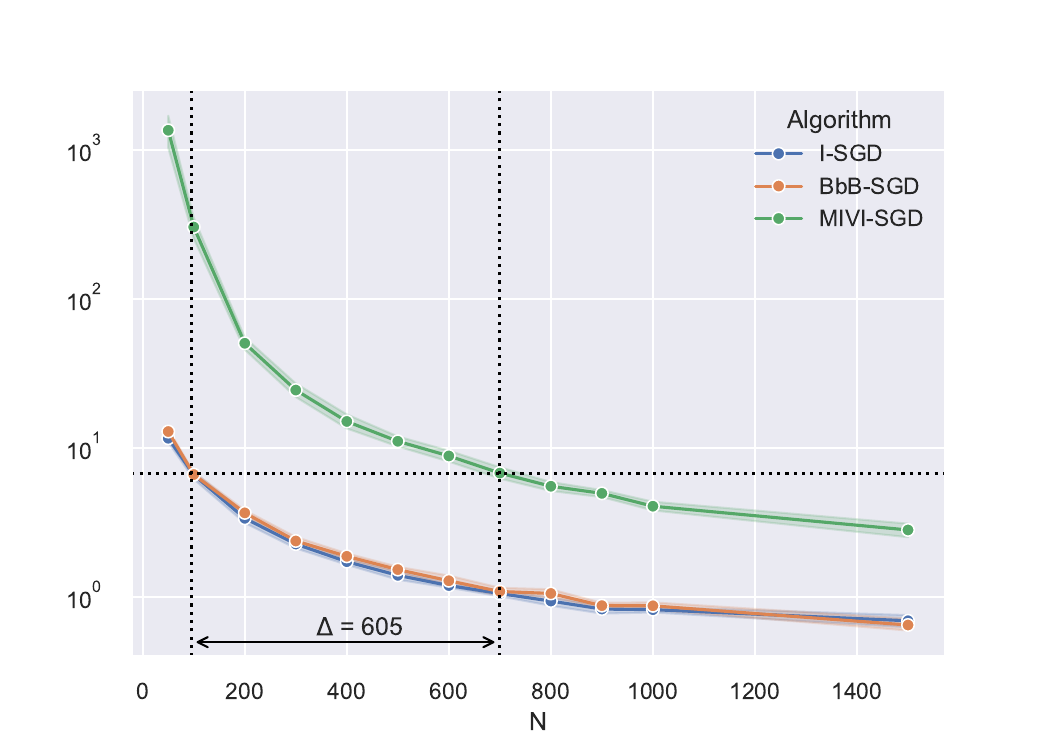}
     \caption{$\mathbb{V}[\langle f, \mu_t^N \rangle]$ with respect to $N$, in the simple (left column) and complex (right column) setting, for $f_{mean}$ ($1^{st}$ line), $f_{std}$ ($2^{nd}$ line) and $f_{pred}$ ($3^{rd}$ line).}
 \label{fig:mu}
\end{figure}

This figure shows that, in the simple setting MiVI-SGD with $N = 700$ obtains the same performance (in term of $\mathbb{V}[\langle f , \; \mu_t^N \rangle]$), than BbB-SGD and I-SGD with $N = 620$ for $f_{mean}$, $N = 210$ for $f_{std}$ and $N = 380$. 
Similarly, in the complex setting MiVI-SGD with $N = 700$ obtains the same performance (in term of $\mathbb{V}[\langle f , \; \mu_t^N \rangle]$), than BbB-SGD and I-SGD with $N = 550$ for $f_{mean}$, $N = 75$ for $f_{std}$ and $N = 95$. 

Consequently, in both settings, algorithm MiVI-SGD appears to be more efficient (in terms of the number of sampled vectors) than other algorithms for achieving the same value of $\mathbb{V}[\langle f , ; \mu_t^N \rangle]$.

\section{Conclusion}

In this work, we have rigorously shown CLT for a two-layer BNN trained
by variational inference with different SGD schemes. It appears that
the idealized SGD and the most-commonly used \emph{Bayes-by-Backprop}
SGD schemes have the same fluctuation behaviors. i.e. driven by a SPDE with a  $\mathfrak G$-process having
the same covariance structure, in addition to admitting the same
mean-field limit. Introduced in \cite{colt}, the less costly
\emph{Minimal VI} SGD scheme exhibits a different fluctuation
behavior, with a  $\mathfrak G$-process of different covariance structure, which
can be argued to lead to larger variances. Though, numerical
experiments show that the trade-off between computational complexity
and variance is still vastly in favour of the \emph{Minimal VI}
scheme. This opens the interesting perspective of exploring whether
additional practical improvements can be derived from the asymptotic
results at the mean-field level. This becomes even more intriguing and
a justified approach given that neural networks appear to reach such
limits rapidly.

\section*{Acknowledgements}
A.D. is grateful for the support received from the Agence Nationale de
la Recherche (ANR) of the French government through the program
"Investissements d'Avenir" (16-IDEX-0001 CAP 20-25) A.G. is supported
by the Institut Universtaire de France. M.M. acknowledges the support
of the the French ANR under the grant ANR-20-CE46-0007 (\emph{SuSa}
project). This work has been (partially) supported by the Project CONVIVIALITY ANR-23-CE40-0003 of the French National Research Agency (A.G, M.M.).  B.N. is supported by the grant IA20Nectoux from the Projet
I-SITE Clermont CAP 20-25. E.M. and T.H. acknowledge the support of
ANR-CHIA-002, "Statistics, computation and Artificial Intelligence";
Part of the work has been developed under the auspice of the Lagrange
Center for Mathematics and Calculus.

\bibliography{elbo_bib}
\bibliographystyle{alpha}

\newpage
\appendix

\section{Central Limit Theorem: proof of Theorem \ref{thm-clt-ideal}}
\label{sec-proof-clt}

 In all this section,  the $\{\theta^i_k, i\in \{1,\ldots,N\}\}$'s are  generated    by  the algorithm  \eqref{eq.algo-batch}, except in Section~\ref{sec-conv-G-process} which, we recall, is dedicated to the study of the convergence of the sequences of martingale  $(\sqrt N\mathbf M^N)_N$ (see \eqref{eq.MN}). Recall the definition of the $\sigma$-algebra $\mathcal F_k^N$ in \eqref{eq.Fk2}.  We also recall the following paramount result which aims at giving uniform bounds, see Lemma 17 in \cite{colt}   on the moments of the parameters $\{\theta^i_k, i\in \{1,\ldots,N\}\}$ up to iteration $\lfloor NT\rfloor$, for a fixed $T>0$.
\begin{lemma}\label{le.Bounds}
Assume  {\rm\textbf{A}}. Then, for all $T>0$ and all $p\ge 1$, there exists $C>0$ such that for all $N\ge1$, $i\in\{1,\dots,N\}$ and $0\leq k\leq \lfloor NT\rfloor$, $\mathbf E[|\theta_k^i|^p]\leq C$. 
\end{lemma}
\begin{sloppypar}
Let us now recall some Sobolev embeddings which will be also used in the proof of  Theorem~\ref{thm-clt-ideal}. For $\mathfrak M,\mathfrak j>(d+1)/2$  and  $\mathfrak k,\mathfrak J \ge 0$, $\mathcal H^{\mathfrak M+\mathfrak J,\mathfrak k}(\mathbf R^{d+1})\hookrightarrow \mathcal C^{\mathfrak J ,\mathfrak k}(\mathbf R^{d+1})$ and $\mathcal H^{\mathfrak M+\mathfrak J ,\mathfrak k}(\mathbf R^{d+1})\hookrightarrow_{\mathrm{H.S.}}\mathcal H^{\mathfrak J ,\mathfrak k+\mathfrak j}(\mathbf R^{d+1})$ (see Section 2 in \cite{fernandez1997hilbertian}).  
Recall $\mathfrak J_3 = 4\lceil \frac{d+1}{2}\rceil+8$ and  $\mathfrak j_3= \lceil \frac{d+1}{2}\rceil +1$. Set $
\mathfrak J_0=\lceil \frac{d+1}{2}\rceil+3,\ \mathfrak J_1=2\lceil \frac{d+1}{2}\rceil+4, \ \mathfrak J_2=3\lceil \frac{d+1}{2}\rceil+6$, 
and 
$
  \mathfrak  j_2 = 2\lceil \frac{d+1}{2}\rceil +2,\ \mathfrak j_1 = 3\lceil \frac{d+1}{2}\rceil +4, \ \mathfrak j_0 = 4\lceil \frac{d+1}{2}\rceil +5$. 
Hence,  the following Hilbert-Schmidt embeddings hold: 
 $\mathcal H^{\mathfrak J_3-1\mathfrak ,\mathfrak j_3}(\mathbf R^{d+1})\hookrightarrow_{\mathrm{H.S.}}\mathcal H^{\mathfrak J_2,\mathfrak j_2}(\mathbf R^{d+1})$, $\mathcal H^{\mathfrak J_2,\mathfrak j_2}(\mathbf R^{d+1})\hookrightarrow_{\mathrm{H.S.}}\mathcal H^{\mathfrak J_1+1,\mathfrak j_1-1}(\mathbf R^{d+1})$, 
$\mathcal H^{\mathfrak J_1,\mathfrak j_1}(\mathbf R^{d+1})\hookrightarrow_{\mathrm{H.S.}}\mathcal H^{\mathfrak J_0,\mathfrak j_0}(\mathbf R^{d+1})$. 
One  also has the following continuous embeddings: 
$\mathcal H^{\mathfrak J_0,\mathfrak j_0}(\mathbf R^{d+1})\hookrightarrow \mathcal C^{2,\mathfrak j_0}(\mathbf R^{d+1})$ and $ \mathcal H^{\mathfrak M,\mathfrak j}(\mathbf R^{d+1})\hookrightarrow \mathcal H^{\mathfrak M,\mathfrak j+\mathfrak k}(\mathbf R^{d+1})$, where $ \mathfrak M,\mathfrak j,\mathfrak k\ge 0$. 
\end{sloppypar}

 We finally recall some useful inequality which will be used throughout this work (see the proof of  Lemma 1 in \cite{colt}) and which are direct consequences of \textbf{A}: for all $\theta\in \mathbf R^{d+1}$, $z\in\mathbf R^d$, and $(x,y)\in\mathsf X\times\mathsf Y$, it holds:
\begin{enumerate}
\item[\textbf I.] $|\phi(\theta,z,x)-y|\leq C$ and $|\nabla_\theta\phi(\theta,z,x)|\leq C|\mathrm J_\theta\Psi_\theta(z)|\leq   C\mathfrak b(z)$ (where $\mathrm J_\theta$ denotes  the Jacobian operator w.r.t.  $\theta$). 
 \end{enumerate}
 In addition,
 \begin{enumerate}
 \item[\textbf{II}.]  For all $x\in\mathsf X$, 
\begin{equation}\label{eq.Sigma}
 \mathfrak H(\cdot ,x): \theta \mapsto \int_{\mathbf R^d} \phi(\theta,z,x)\gamma(z)\di z=\langle \phi(\theta,\cdot,x),\gamma \rangle
\end{equation} 
 is smooth  and all its derivatives of non negative order   are uniformly bounded  over $\mathbf R^{d+1}$ w.r.t $x\in \mathsf X$. 
\end{enumerate}
Moreover,  for any multi-index  $\alpha\in\mathbf N^{d+1}$, (see Remark \ref{re.KL}), it holds for some $C>0$ and all $\theta\in \mathbf R^{d+1}$:
\begin{equation}\label{eq.boundKL}
|\partial_\alpha\mathscr D_{\mathrm{KL}}(q^1_\theta|P_0^1)|\le C(1+|\theta|) \text{ if } |\alpha|=1 \text{ and } |\partial_\alpha\mathscr D_{\mathrm{KL}}(q^1_\theta|P_0^1)|\le C \text{ for } |\alpha|\ge 2.
\end{equation}

\subsection{Relative compactness of the fluctuation  sequence $(\eta^N)_{N\ge1}$}
Recall that the fluctuation process is defined by $\eta^N: t \in \mathbf R_+\mapsto  \sqrt N(\mu^N_t-\bar\mu_t)$, $N\ge 1$. The aim of this section is to prove the following relative compactness result on the sequence $(\eta^N)_{N\ge1}$. 
\begin{proposition}\label{p-rc-eta}
Assume  {\rm\textbf{A}}. 
Then,  $(\eta^N)_{N\ge1}$ is relatively compact in $\mathcal D(\mathbf R_+,\mathcal H^{-\mathfrak J_3+1,\mathfrak j_3}(\mathbf R^{d+1}))$. 
\end{proposition} 
\noindent
 We mention that Proposition \ref{p-rc-eta} also holds when the $\{\theta^i_k, i\in \{1,\ldots,N\}\}$'s are  generated    by  the two other algorithms \eqref{eq.algo-ideal} and \eqref{eq.algo-z1z2}. Before starting the proof of Proposition \ref{p-rc-eta}, we need to introduce an auxiliary system of particles, this is the purpose of the next lemma.
 
 For any $\mu \in \mathcal P(\mathcal C(\mathbf R_+,\mathbf R^{d+1}))$, we consider  $\mathscr P_\mu \in \mathcal P(\mathcal C(\mathbf R_+,\mathbf R^{d+1}))$ defined as the law of the process  $(X_t)_{t\ge 0}$ solution to
$$
  { \rm  \mathbf{(E_\mu)} }
 \begin{cases}
  \di  X_t=-\kappa\int_{\mathsf X\times\mathsf Y}\langle\phi(\cdot,\cdot,x)-y, \mu_t\otimes\gamma\rangle\langle\nabla_\theta\phi( X_t,\cdot,x),\gamma\rangle\pi(\di x,\di y)\di t -\kappa\nabla_\theta\mathscr D_{\mathrm{KL}}(q^1_{X_t}|P_0^1) \di t,\\
 X_0\sim \mu_0.
  \end{cases}
$$
We then denote by $\mathscr F(\mu)$ the function   $t\in \mathbf R_+\mapsto (\mathscr P_\mu)_t= \mathscr P_\mu\circ \pi_t^{-1}$ the law $(X_s)_{s\ge 0}$ at time $t$, where $\pi_t$ is the natural projection from $\mathcal C(\mathbf R_+,\mathbf R^{d+1})$ to $\mathbf R$ define by $\pi_t(f)=f(t)$.

 \begin{lemma}\label{le.PS}
 Assume {\rm \textbf{A}}. Then, $\bar \mu=\mathscr F( \bar  \mu)$ (where $\bar \mu$ is given by Theorem \ref{th.LLN}), i.e. for the solution  $(\bar X_t)_{t\ge 0}$ of  {\rm $\mathbf{( E_{\bar \mu})}$}, it holds $\bar X_t\sim \bar\mu_t$ for all $t\ge 0$. 
\end{lemma}
\begin{proof}
 We claim that $\mathscr F(\mu)\in  \mathcal C(\mathbf R_+,\mathcal P_1(\mathbf R^{d+1}))$, for all $\mu \in \mathcal P(\mathcal C(\mathbf R_+,\mathbf R^{d+1}))$. Let us prove this claim. 
 Let  $ (X_t)_{t\ge 0}$ be the  solution of {\rm  $\mathbf{(E_\mu)}$}. Then, by \textbf{I},  \textbf{II},  and   \textbf{A}, together with~\eqref{eq.boundKL},  there exists $c_0>0$ such that a.s. for all  $t\ge0$, 
$$| X_t |\le c_0(1+t)+c_0 \int_0^t| X_s|\di s.$$
Therefore, a.s., for all $T>0$ and $0\le t\le T$, by  Gronwall lemma, one has  $| X_t|\le c_0 (1+T)e^{c_0T}$. 
With this bound, one deduces  that there exists $c_1>0$ such that a.s.  for all $0\le s\le t\le T$, $| X_t- X_s|\le c_1(1+T)e^{c_1T}(t-s)$, which proves  the claim.

 Let $\mu\in\mathcal P(\mathbf R^{d+1})$. Define $\mathscr V[\mu]:\mathbf R^{d+1}\to\mathbf R^{d+1}$ by:
\begin{align}
\label{def V[mu]}
&\mathscr V[\mu](\theta)=
-\kappa\int_{\mathsf X\times\mathsf Y}\langle\phi(\cdot,\cdot,x)-y,\mu\otimes\gamma\rangle\langle\nabla_\theta\phi(\theta,\cdot,x),\gamma\rangle\pi(\di x,\di y)-\kappa \nabla_\theta \mathscr D_{{\rm KL}}(q_{\theta}^1|P_0^1).
\end{align}
By  the analysis carried out in Section B.3.2 in \cite{colt} (based on Th. 5.34 in~\cite{villani2021topics}),   $\bar \mu$  is the unique  \textit{weak solution}\footnote{See Section 4.1.2 in \cite{santambrogio2015optimal} for the definition.}   in $ \mathcal C(\mathbf R_+,\mathcal P_1(\mathbf R^{d+1}))$ of the measure-valued equation
\begin{align}\label{eq.measure}
\begin{cases}
\partial_t \mu^*_t=\mathrm{div}(\mathscr V[\bar\mu_t] \mu^*_t)\\
 \mu^*_0=\mu_0.
\end{cases}
\end{align}
 On the other hand, using  the equality $g( X_t)-g(X_0)= \int_0^t \nabla g( X_u) \cdot \frac{\di }{\di t} X_u \di u$ valid for   any $\mathcal C^1$ function $g$ with compact support,  together with {\rm  $\mathbf{(E_\mu)}$}, we deduce  that $\mathscr F( \bar  \mu)$ is a weak solution of  \eqref{eq.measure}. By uniqueness, $\bar \mu=\mathscr F( \bar  \mu)$. The proof is complete.  
\end{proof}

 Let us now  introduce $N$ independent processes $\bar X^i$, $i\in \{1,\ldots,N\}$, solution to {\rm $\mathbf{( E_{\bar \mu})}$}. It then  holds thanks to Lemma \ref{le.PS}, for all $i\in \{1,\ldots,N\}$ and $t\ge 0$:
$$
  { \rm  \mathbf{(S)} }
 \begin{cases}
 \di \bar X^i_t=-\kappa\int_{\mathsf X\times\mathsf Y}\langle\phi(\cdot,\cdot,x)-y,\bar\mu_t\otimes\gamma\rangle\langle\nabla_\theta\phi(\bar X^i_t,\cdot,x),\gamma\rangle\pi(\di x,\di y)\di t -\kappa\nabla_\theta\mathscr D_{\mathrm{KL}}(q^1_{\bar X^i_t}|P_0^1) \di t,\\
 \bar X^i_0\sim \mu_0, \ \bar X^i_t\sim \bar\mu_t.
 \end{cases}
$$ 
 Their empirical distribution is denoted by $\bar\mu_t^N=\frac 1N\sum_{i=1}^N\delta_{\bar X_t^i}$, for  $N\ge 1$ and $t\in \mathbf R_+$. 
Recall   that from the proof of Lemma \ref{le.PS}, there exists $c_1>0$ such that a.s.  for all $0\le s\le t\le T$ and all $i\in \{1,\dots, N\}$: 
\begin{equation}\label{eq.Xt}
|  \bar X^i_t|\le c_1(1+T)e^{c_1T} \text{and } | \bar X^i_t-  \bar X^i_s|\le c_1(1+T)e^{c_1T}(t-s). 
\end{equation}

 We now decompose $\eta^N$ using the  following two processes: 
\begin{equation}
\Upsilon^N:=\sqrt N(\mu^N-\bar\mu^N)\ \text{and}\ \Theta^N:=\sqrt N(\bar\mu^N-\bar\mu). 
\end{equation}
We denote by $\mathcal C^{\mathfrak J,\mathfrak j}(\mathbf R^{d+1})^*$  the dual space of $\mathcal C^{\mathfrak J,\mathfrak j}(\mathbf R^{d+1})$ ($\mathfrak J,\mathfrak j\ge 0$).  One the one hand, $\bar\mu^N\in\mathcal C(\mathbf R_+,\mathcal C^{1,\mathfrak j}(\mathbf R^{d+1})^*)$, $\mathfrak j\ge 0$. This  is indeed a direct consequence of   \eqref{eq.Xt}.  On the other hand,  for any $\mathfrak j\ge0$, $\mu^N\in\mathcal D(\mathbf R_+,\mathcal C^{0,\mathfrak j}(\mathbf R^{d+1})^*)$.  Hence, it holds for all $\mathfrak j\ge 0$ a.s. 
\begin{equation}\label{eq.up-app}
\Upsilon^N\in \mathcal D(\mathbf R_+,\mathcal C^{1,\mathfrak j}(\mathbf R^{d+1})^*).
\end{equation} 
Concerning $\Theta^N$, we have the following result. 


 \begin{lemma}\label{Theta^N-continu}
 Assume {\rm \textbf{A}}.  Then, for any $\mathfrak J>1+ (d+1)/2$ and $ \mathfrak k \ge 0$, $\bar\mu^N, \bar\mu \in \mathcal C(\mathbf R_+,\mathcal H^{-\mathfrak J,\mathfrak k}(\mathbf R^{d+1}))$.  Therefore, a.s. $\Theta^N \in \mathcal C(\mathbf R_+,\mathcal H^{-\mathfrak J, \mathfrak k}(\mathbf R^{d+1}))$. Finally, \eqref{eq.P2} also holds for any test function $f\in \mathcal H^{\mathfrak J,\mathfrak k}(\mathbf R^{d+1})$ ($\mathfrak J>1+ (d+1)/2$ and $ \mathfrak k \ge 0$).
 \end{lemma}

\begin{proof}
  Let $\mathfrak J>1+ (d+1)/2$ and $  \mathfrak k\ge 0$. It then holds  $\mathcal H^{\mathfrak J,\mathfrak k}(\mathbf R^{d+1})\hookrightarrow \mathcal C^{1,\mathfrak k}(\mathbf R^{d+1})$. 
This implies that $\mathcal C^{1,\mathfrak k}(\mathbf R^{d+1})^*\hookrightarrow \mathcal H^{-\mathfrak J,\mathfrak k}(\mathbf R^{d+1})$, and consequently, $\bar\mu^N\in   \mathcal C(\mathbf R_+,\mathcal H^{-\mathfrak J,\mathfrak k}(\mathbf R^{d+1}))$.

 Let us now prove that $\bar\mu\in \mathcal C(\mathbf R_+,\mathcal H^{-\mathfrak J,\mathfrak k}(\mathbf R^{d+1}))$ for $ \mathfrak k\ge 0$. Set $\mathfrak j=\mathfrak k+1$.  Recall that one can choose any   $\gamma_0> 1+ \frac{d+1}{2}$ in Theorem \ref{th.LLN}. Pick thus such a $\gamma_0$ such that $  \mathfrak j \le \gamma_0$. 
We then have $\mathcal H^{\mathfrak J,\mathfrak j-1}(\mathbf R^{d+1})\hookrightarrow \mathcal C^{1,\mathfrak j-1}(\mathbf R^{d+1}) \hookrightarrow \mathcal C^{1,\gamma_0-1}(\mathbf R^{d+1})\hookrightarrow \mathcal C^{0,\gamma_0-1}(\mathbf R^{d+1})$. 
 Since $\mu_0$ has compact support, $\mu_0\in \mathcal C^{0,\gamma_0-1}(\mathbf R^d)^* \hookrightarrow \mathcal H^{-\mathfrak J,\mathfrak j-1}(\mathbf{R}^d)$. Let $f\in \mathcal C_c^\infty(\mathbf R^{d+1})$ and $0\le s\le t\le T$.  
Thanks to \eqref{eq.boundKL} and Assumption \textbf{A}, we deduce that:
\begin{align*}
|\langle f,\bar \mu_t\rangle-\langle f,\bar \mu_s\rangle| \le &C|t-s| ( \Vert f \Vert_{\mathcal C^{1,\gamma_0}}+ \Vert f \Vert_{\mathcal C^{1,\gamma_0-1}}) \sup_{u\in [0,T]} |\langle 1+|\cdot |^{ \gamma_0},\bar \mu_u\rangle|\\
&\le C|t-s|   \Vert f \Vert_{\mathcal C^{1,\gamma_0-1}} \sup_{u\in [0,T]} |\langle 1+|\cdot |^{ \gamma_0},\bar \mu_u\rangle|\\
&\le C|t-s|   \Vert f \Vert_{\mathcal H^{\mathfrak J,\mathfrak j-1}} \sup_{u\in [0,T]} |\langle 1+|\cdot |^{ \gamma_0},\bar \mu_u\rangle|.
\end{align*} 
 We have that $\sup_{u\in [0,T]} |\langle 1+|\cdot |^{\gamma_0},\bar \mu_u\rangle|<+\infty$ since $u\ge 0\mapsto \langle 1+|\cdot |^{\gamma_0},\bar \mu_u\rangle\in  \mathcal D(\mathbf R_+, \mathbf R)$ (this follows from the fact that $\bar \mu \in \mathcal D(\mathbf R_+, \mathcal P_{\gamma_0}(\mathbf{R}^d))$ together with~ Th. 6.9 in \cite{villani2009optimal}).  We have thus proved that  $\bar \mu_t\in   \mathcal H^{-\mathfrak J,\mathfrak j-1}(\mathbf{R}^d)$ and  $|\langle f,\bar \mu_t\rangle-\langle f,\bar \mu_s\rangle| \le C|t-s| \Vert f \Vert_{\mathcal H^{\mathfrak J,\mathfrak j-1}}$. This proves that $\bar\mu\in \mathcal C(\mathbf R_+,\mathcal H^{-\mathfrak J,\mathfrak j-1}(\mathbf R^{d+1}))$.  The last claim  is obtained by a density argument and the fact that $\mathcal H^{\mathfrak J,\mathfrak j-1}(\mathbf R^{d+1})  \hookrightarrow\mathcal C^{1,\gamma_0-1}(\mathbf R^{d+1})$.
 \end{proof}


\begin{lemma}\label{lem-E[eta]<infty}
 Assume {\rm \textbf{A}}. 
For all $T>0$, we have 
$$\sup_{N\ge1}\sup_{t\in[0,T]}\mathbf E[\|\Theta_t^N\|_{\mathcal H^{-\mathfrak J_1,\mathfrak j_1}}^2+\|\Upsilon_t^N\|_{\mathcal H^{-\mathfrak J_1,\mathfrak j_1}}^2]<+\infty.$$
In particular, $\sup_{N\ge1}\sup_{t\in[0,T]}\mathbf E[\|\eta_t^N\|_{\mathcal H^{-\mathfrak J_1,\mathfrak j_1}}^2]<+\infty$.
\end{lemma}

\begin{proof}
\begin{sloppypar}
Let $T>0$. Pick $t\in [0,T]$, $N\ge 1$, and $f\in \mathcal H^{\mathfrak J_1,\mathfrak j_1}(\mathbf R^d)$. 
On the one hand, since  $(f(\bar{X}_t^j)-\langle f,\bar \mu_t\rangle)_{j=1,\ldots,N}$ are independent  centered random variables,  one deduces that $\textbf{E}[\langle f,\Theta_t^N\rangle^2]\le \frac{2}{N}\sum_{i=1}^N(\textbf{E}\big[|f(\bar{X}_t^i)|^2\big]+|\langle f,\bar \mu_t\rangle|^2) \le  C_T\|f\|_{\mathcal H^{\mathfrak J_0,\mathfrak j_0}}^2$, where the last inequality is a consequence of   \eqref{eq.Xt} together with  $\mathcal H^{\mathfrak J_0,\mathfrak j_0}(\mathbf R^{d+1})\hookrightarrow \mathcal C^{0,\mathfrak j_0}(\mathbf R^{d+1}) $ and $\bar\mu \in \mathcal C(\mathbf R_+,\mathcal H^{-\mathfrak J_0,\mathfrak j_0}(\mathbf R^{d+1}))$ (see Lemma \ref{Theta^N-continu}). 
Using also  the embedding $\mathcal H^{\mathfrak J_1,\mathfrak j_1}(\mathbf R^{d+1})\hookrightarrow_{\mathrm{H.S.}}\mathcal H^{\mathfrak J_0,\mathfrak j_0}(\mathbf R^{d+1})$ and considering an orthonormal basis  of $\mathcal H^{\mathfrak J_1,\mathfrak j_1}(\mathbf R^{d+1})$, one deduces the desired upper bound on $\Theta^N$. 
\end{sloppypar}

 Let us now derive  the bound on the second order moment of $\Upsilon^N$. 
To this end, introduce  an orthonormal basis $(f_a)_{a\ge1}$ of $\mathcal H^{\mathfrak J_1,\mathfrak j_1}(\mathbf R^{d+1})$. One then has: 
\begin{equation}\label{norme_Upsilon_def}
\|\Upsilon_t^N\|_{\mathcal H^{-\mathfrak J_1,\mathfrak j_1}}^2=\sum_{a\ge1}\langle f_a,\Upsilon_t^N\rangle^2. 
\end{equation}
Recall $\mathcal H^{\mathfrak J_0,\mathfrak j_0}(\mathbf R^{d+1})\hookrightarrow \mathcal C^{2,\mathfrak j_0}(\mathbf R^{d+1})$. 
  We have, by  \textbf{(S)} and the fact that $f\in \mathcal C^{2,\mathfrak j_0}(\mathbf R^{d+1})$,
\begin{align}\label{eq-<f,barmu_t^N}
\langle f,\bar\mu_t^N\rangle=\langle f,\bar\mu_0^N\rangle
&-\kappa\int_0^t\int_{\mathsf X\times\mathsf Y}\langle\phi(\cdot,\cdot,x)-y,\bar\mu_s\otimes\gamma\rangle\langle\nabla f\cdot\nabla_{\theta}\phi(\cdot,\cdot,x),\bar\mu_s^N\otimes\gamma\rangle\pi(\di x,\di y)\di s\nonumber\\
&-\kappa\int_0^t\langle\nabla f\cdot\nabla_{\theta}\mathscr D_{\mathrm{KL}}(q_\cdot^1|P_0^1),\bar\mu_s^N\rangle\di s.
\end{align}
We now set  for $k\ge 0$ and $g\in \mathcal C^{2,\mathfrak j }(\mathbf R^{d+1})$ ($\mathfrak j \ge 0$):
\begin{enumerate}
\item[\textbf 1.]  $
\mathbf D_{k}^N[g]:=-\frac{\kappa}{N^3}\sum_{i=1}^N\sum_{j=1,j\neq i}^N\int_{\mathsf X\times\mathsf Y}\big (\big \langle\phi(\theta_k^j,\cdot,x),\gamma\big \rangle-y\big )\big \langle\nabla_\theta g(\theta_k^i)\cdot\nabla_\theta\phi(\theta_k^i,\cdot,x),\gamma\big \rangle\pi(\di x,\di y) -\frac{\kappa}{N^2}\int_{\mathsf X\times\mathsf Y}\big \langle(\phi(\cdot,\cdot,x)-y)\nabla_\theta g\cdot\nabla_\theta\phi(\cdot,\cdot,x),\nu_k^N\otimes\gamma\big \rangle\pi(\di x,\di y)$.
\item[\textbf 2.] $\mathbf M_{k}^N[g]= -\frac{\kappa}{N^3}\sum_{i,j=1}^N   (\phi(\theta_k^j,\mathsf  Z_k^{j},x_k)-y_k)\nabla_\theta g(\theta_k^i)\cdot\nabla_\theta \phi(\theta_k^i,\mathsf Z_k^{i},x_k)-\mathbf D_{k}^N[g]$. 
\item[\textbf 3.]  $\mathbf R_k^N[g]:=\frac{1}{2N}\sum_{i=1}^N(\theta_{k+1}^i-\theta_k^i)^T\nabla^2g(\widehat{\theta_k^i})(\theta_{k+1}^i-\theta_k^i)$  is the rest of the second order Taylor expansion of $\frac 1N \sum_{k=1}^N f(\theta_{k+1}^i)-f(\theta_{k}^i)$ (the point $\widehat{\theta_k^i}$ lies in $[\theta_{k+1}^i,\theta_k^i]$). 
\end{enumerate}
Note that $\mathbf D_{k}^N[g]$ and $\mathbf M_{k}^N[g]$ are well defined for $g\in \mathcal C^{1,\mathfrak j}(\mathbf R^{d+1})$ ($\mathfrak j \ge 0$). For   $t\ge 0$, we also define:
\begin{equation}\label{eq.MN}
\mathbf R_t^N[g]:=\sum_{k=0}^{\lfloor Nt\rfloor-1}\mathbf R_k^N[g]  \ \ \text{and}\ \ \mathbf  M_t^{N}[g]:=\sum_{k=0}^{\lfloor Nt\rfloor-1}\mathbf M_{k}^{N}[g].
\end{equation}
Let $t\ge 0$. 
With these definitions, we recall that from Eq. (53) in \cite{colt}, there exist $\widehat{\theta_k^i}$ ($i=1,\ldots,N$ and $k=0, \ldots,\lfloor Nt\rfloor-1$)  such that for $g\in \mathcal C^{2,\mathfrak j_0}(\mathbf R^{d+1})$:
\begin{align}\label{eq.pre_limitz1zN}
\langle g,\mu_t^N\rangle-\langle g,\mu_0^N\rangle&=-\kappa\int_{0}^t\int_{\mathsf X\times\mathsf Y}\langle\phi(\cdot,\cdot,x)-y,\mu_s^N\otimes\gamma\rangle\langle\nabla_\theta g\cdot\nabla_\theta \phi(\cdot,\cdot,x),\mu_s^N\otimes\gamma\rangle  \pi(\di x,\di y)\di s\nonumber\\
&\quad -\kappa \int_0^t \big \langle\nabla_\theta g\cdot \nabla_\theta \mathscr D_{{\rm KL}}(q_{\,_\cdot }^1|P_0^1),\mu_s^N\big \rangle\di s  \nonumber\\
&\quad +\frac{\kappa}{N}\int_{0}^t\int_{\mathsf X\times\mathsf Y} \Big\langle\langle\phi(\cdot,\cdot,x)-y,\gamma\rangle\langle\nabla_\theta  g\cdot\nabla_\theta \phi(\cdot,\cdot,x),\gamma\rangle,\mu_s^N\Big\rangle \pi(\di x,\di y)\di s\nonumber\\
&\quad -\frac{\kappa}{N}\int_{0}^t\int_{\mathsf X\times\mathsf Y} \Big\langle(\phi(\cdot,\cdot,x)-y)\nabla_\theta g\cdot\nabla_\theta \phi(\cdot,\cdot,x),\mu_s^N\otimes\gamma\Big\rangle \pi(\di x,\di y)\di s\nonumber\\
&\quad +  \mathbf M_t^{N}[g] +\mathbf W_t^{N}[g]+ \mathbf R_t^N[g],
\end{align}
 where $
 \mathbf  W_t^{N}[f]:=-  \mathbf  V_t^{N}[f] + \kappa \int^t_{\frac{\lfloor Nt\rfloor}{N}}\big \langle\nabla_\theta f\cdot \nabla_\theta \mathscr D_{{\rm KL}}(q_{\,_\cdot }^1|P_0^1),\mu_s^N\big \rangle\di s$ 
 and 
 \begin{align*}
\mathbf V_t^{N}[f]&:=-\kappa\int^{t}_{\frac{\lfloor Nt\rfloor}{N}}\int_{\mathsf X\times\mathsf Y}\langle\phi(\cdot,\cdot,x)-y,\mu_s^N\otimes\gamma\rangle\langle\nabla_\theta f\cdot\nabla_\theta \phi(\cdot,\cdot,x),\mu_s^N\otimes\gamma\rangle  \pi(\di x,\di y)\di s\\
&\quad +\frac{\kappa}{N}\int^{t}_{\frac{\lfloor Nt\rfloor}{N}}\int_{\mathsf X\times\mathsf Y} \Big\langle\langle\phi(\cdot,\cdot,x)-y,\gamma\rangle\langle\nabla_\theta f\cdot\nabla_\theta \phi(\cdot,\cdot,x),\gamma\rangle,\mu_s^N\Big\rangle \pi(\di x,\di y)\di s\\
&\quad -\frac{\kappa}{N}\int^{t}_{\frac{\lfloor Nt\rfloor}{N}}\int_{\mathsf X\times\mathsf Y} \Big\langle(\phi(\cdot,\cdot,x)-y)\nabla_\theta f\cdot\nabla_\theta \phi(\cdot,\cdot,x),\mu_s^N\otimes\gamma\Big\rangle \pi(\di x,\di y)\di s.
\end{align*}
Hence, since by definition $\Upsilon^N=\sqrt N(\mu^N-\bar\mu^N)$, one has for all $t\in \mathbf R_+$, using \eqref{eq-<f,barmu_t^N} and \eqref{eq.pre_limitz1zN} together with the fact that $\langle f,\mu_0^N\rangle=\langle f,\bar \mu_0^N\rangle$:
\begin{align}\label{def<f,Upsilon>}
\langle f,\Upsilon_t^N\rangle&=
 -\kappa\int_0^t\int_{\mathsf X\times\mathsf Y}\langle\phi(\cdot,\cdot,x)-y,\mu_s^N\otimes\gamma\rangle\langle\nabla_\theta f\cdot\nabla_\theta\phi(\cdot,\cdot,x),\Upsilon_s^N\otimes\gamma\rangle\pi(\di x,\di y)\di s\nonumber\\
 &\quad -\kappa\int_0^t\int_{\mathsf X\times\mathsf Y}\langle\phi(\cdot,\cdot,x),\Upsilon_s^N\otimes\gamma\rangle\langle\nabla_\theta f\cdot\nabla_\theta\phi(\cdot,\cdot,x),\bar\mu_s^N\otimes\gamma\rangle\pi(\di x,\di y)\di s\nonumber\\
&\quad -\kappa\int_0^t\int_{\mathsf X\times\mathsf Y}\langle\phi(\cdot,\cdot,x),\sqrt N(\bar\mu_s^N-\bar\mu_s)\otimes\gamma\rangle\langle\nabla_\theta f\cdot\nabla_\theta\phi(\cdot,\cdot,x),\bar\mu_s^N\otimes\gamma\rangle\pi(\di x,\di y)\di s\nonumber\\
&\quad -\kappa\int_0^t\langle\nabla_\theta f\cdot\nabla_\theta\mathscr D_{\mathrm{KL}}(q^1_\cdot|P_0^1),\Upsilon_s^N\rangle\di s\nonumber\\
&\quad +\frac{\kappa}{\sqrt N}\int_0^t\int_{\mathsf X\times\mathsf Y}\Big\langle\langle\phi(\cdot,\cdot,x)-y,\gamma\rangle\langle\nabla_\theta f\cdot\nabla_\theta\phi(\cdot,\cdot,x),\gamma\rangle,\mu_s^N\Big\rangle\pi(\di x,\di y)\di s\nonumber\\
&\quad -\frac{\kappa}{\sqrt N}\int_0^t\int_{\mathsf X\times\mathsf Y}\Big\langle (\phi(\cdot,\cdot,x)-y)\nabla_\theta f\cdot\nabla_\theta\phi(\cdot,\cdot,x),\mu_s^N\otimes\gamma\Big\rangle\pi(\di x,\di y)\di s\nonumber\\
&\quad +\sqrt N\mathbf M_t^N[f]+\sqrt N\mathbf W_t^N[f]+\sqrt N\mathbf R_t^N[f].
\end{align}
 Using \textbf{II},  when $\mathfrak j>\frac{d+1}{2}$, one has $\mathfrak H(\cdot ,x)\in \mathcal H^{\mathfrak J,\mathfrak j}(\mathbf R^{d+1})$ for all $\mathfrak J\ge 0$, and it holds:
\begin{equation}\label{phi-normeH}
\sup_{x\in\mathsf X}\big\|\mathfrak H(.,x)\big\|_{\mathcal H^{\mathfrak J,\mathfrak j}}<+\infty. 
\end{equation}  
By Lemma B.3 in \cite{jmlr}, one has, for all $t\in\mathbf R_+$, 
\begin{equation}\label{UpsleqA+B}
\langle f,\Upsilon_t^N\rangle^2\le \mathbf A_t^N[f]+\mathbf B_t^N[f], 
\end{equation}
where  
\begin{align}\label{def_A^N}
\mathbf A_t^N[f]&=
 -2\kappa\int_0^t\int_{\mathsf X\times\mathsf Y}\langle f,\Upsilon_s^N\rangle \langle\phi(\cdot,\cdot,x)-y,\mu_s^N\otimes\gamma\rangle\langle\nabla_\theta f\cdot\nabla_\theta\phi(\cdot,\cdot,x),\Upsilon_s^N\otimes\gamma\rangle\pi(\di x,\di y)\di s\nonumber\\
 &\quad -2\kappa\int_0^t\int_{\mathsf X\times\mathsf Y}\langle f,\Upsilon_s^N\rangle\langle\phi(\cdot,\cdot,x),\Upsilon_s^N\otimes\gamma\rangle\langle\nabla_\theta f\cdot\nabla_\theta\phi(\cdot,\cdot,x),\bar\mu_s^N\otimes\gamma\rangle\pi(\di x,\di y)\di s\nonumber\\
&\quad -2\kappa\int_0^t\int_{\mathsf X\times\mathsf Y}\langle f,\Upsilon_s^N\rangle\langle\phi(\cdot,\cdot,x),\sqrt N(\bar\mu_s^N-\bar\mu_s)\otimes\gamma\rangle\langle\nabla_\theta f\cdot\nabla_\theta\phi(\cdot,\cdot,x),\bar\mu_s^N\otimes\gamma\rangle\pi(\di x,\di y)\di s\nonumber\\
&\quad -2\kappa\int_0^t\langle f,\Upsilon_s^N\rangle\langle\nabla_\theta f\cdot\nabla_\theta\mathscr D_{\mathrm{KL}}(q^1_\cdot|P_0^1),\Upsilon_s^N\rangle\di s\nonumber\\
&\quad +\frac{2\kappa}{\sqrt N}\int_0^t\int_{\mathsf X\times\mathsf Y}\langle f,\Upsilon_s^N\rangle\Big\langle\langle\phi(\cdot,\cdot,x)-y,\gamma\rangle\langle\nabla_\theta f\cdot\nabla_\theta\phi(\cdot,\cdot,x),\gamma\rangle,\mu_s^N\Big\rangle\pi(\di x,\di y)\di s\nonumber\\
&\quad -\frac{2\kappa}{\sqrt N}\int_0^t\int_{\mathsf X\times\mathsf Y}\langle f,\Upsilon_s^N\rangle\Big\langle (\phi(\cdot,\cdot,x)-y)\nabla_\theta f\cdot\nabla_\theta\phi(\cdot,\cdot,x),\mu_s^N\otimes\gamma\Big\rangle\pi(\di x,\di y)\di s
\end{align}
and  
\begin{align*}
\mathbf B_t^N[f] &= \sum_{k=0}^{\lfloor Nt\rfloor-1}\Big[2\langle f,\Upsilon_{\frac{k+1}{N}^-}^N\rangle\sqrt N\mathbf R_k^N[f]+3N\mathbf R_k^N[f]^2\Big]+\sum_{k=0}^{\lfloor Nt\rfloor-1}\Big[2\langle f,\Upsilon_{\frac{k+1}{N}^-}^N\rangle\sqrt N\mathbf M_k^N[f]+3N\mathbf M_k^N[f]^2\Big]\\
&\quad +\sum_{k=0}^{\lfloor Nt\rfloor-1}\Big[2\langle f,\Upsilon_{\frac{k+1}{N}^-}^N\rangle\mathbf a_k^N[f]+3\mathbf a_k^N[f]^2\Big]-2\sqrt N\int_0^t\langle f,\Upsilon_s^N\rangle\mathbf L_s^N[f]\di s,
\end{align*}
with, for $s\in[0,t]$, 
\begin{align*}
\mathbf L_s^N[f]&=-\kappa\int_{\mathsf X\times\mathsf Y}\langle\phi(\cdot,\cdot,x)-y,\mu_s^N\otimes\gamma\rangle\langle\nabla_\theta f\cdot\nabla_\theta\phi(\cdot,\cdot,x),\mu_s^N\otimes\gamma\rangle\pi(\di x,\di y)\\
&+\frac{\kappa}{N}\int_{\mathsf X\times\mathsf Y}\Big\langle\langle\phi(\cdot,\cdot,x)-y,\gamma\rangle\langle\nabla_\theta f\cdot\nabla_\theta\phi(\cdot,\cdot,x),\gamma\rangle,\mu_s^N\Big\rangle\pi(\di x,\di y)  \\
&-\frac{\kappa}{N}\int_{\mathsf X\times\mathsf Y}\Big\langle(\phi(\cdot,\cdot,x)-y)\nabla_\theta f\cdot\nabla_\theta\phi(\cdot,\cdot,x),\mu_s^N\otimes\gamma\Big\rangle\pi(\di x,\di y) \\
&-\kappa\langle\nabla_\theta f\cdot\nabla_\theta\mathscr D_{\mathrm{KL}}(q^1_{\cdot}|P_0^1),\mu_s^N\rangle
\end{align*}
and, for $0\le k<\lfloor Nt\rfloor$,  $
\mathbf a_k^N[f]=\sqrt N\int_{\frac{k}{N}}^{\frac{k+1}{N}}\mathbf L_s^N[f]\di s$. 
By \eqref{norme_Upsilon_def} and \eqref{UpsleqA+B},  
\begin{equation}\label{Upsilon_t^N^2le}
\|\Upsilon_t^N\|_{\mathcal H^{-\mathfrak J_1,\mathfrak j_1}}^2\le \sum_{a\ge1}\mathbf A_t^N[f_a]+\mathbf B_t^N[f_a].
\end{equation}
Using Lemma \ref{lem_borneA_et_B}, one deduces that:
\begin{equation}\label{sumE[A+B}
\sum_{a\ge1}\mathbf E[\mathbf A_t^N[f_a]+\mathbf B_t^N[f_a]] \le C_T+C_T\int_0^t\mathbf E[\|\Upsilon_s^N\|_{\mathcal H^{-\mathfrak J_1,\mathfrak j_1}}^2]\di s
\end{equation}
Hence, by \eqref{Upsilon_t^N^2le} and \eqref{sumE[A+B},
\begin{equation}
\mathbf E[\|\Upsilon_t^N\|_{\mathcal H^{-\mathfrak J_1,\mathfrak j_1}}^2]\le C_T+C_T\int_0^t\mathbf E[\|\Upsilon_s^N\|_{\mathcal H^{-\mathfrak J_1,\mathfrak j_1}}^2]\di s.
\end{equation}
Using  Gronwall's lemma yields the desired moment estimate on $\Upsilon^N$. 
\end{proof}

The following lemma provides the compact containment condition we need to prove that $(\eta^N)_{N\ge1}$ is relatively compact in $\mathcal D(\mathbf R_+,\mathcal H^{-\mathfrak J_3+1,\mathfrak j_3}(\mathbf R^{d+1}))$.
\begin{lemma}\label{lem-cc-eta}
Assume {\rm \textbf{A}}.
Then, for all $T>0$, $
\sup_{N\ge1}\mathbf E[\sup_{t\in[0,T]}\|\eta_t^N\|_{\mathcal H^{-\mathfrak J_2,\mathfrak j_2}}^2]<+\infty$. 
\end{lemma}
\begin{proof}
Let $T>0$ and $N\ge1$. Consider an orthonormal basis  $(f_a)_{a\ge1}$  of $\mathcal H^{\mathfrak J_2,\mathfrak j_2}(\mathbf R^{d+1})$ and $f\in\mathcal H^{\mathfrak J_2,\mathfrak j_2}(\mathbf R^{d+1})$. 
From \eqref{def<f,Upsilon>} and using Jensen's inequality, 
\begin{align}\label{def<f,Upsilon>carre}
\sup_{t\in[0,T]}\langle f,\Upsilon_t^N\rangle^2&\le
C\int_0^T\int_{\mathsf X\times\mathsf Y}|\langle\phi(\cdot,\cdot,x)-y,\mu_s^N\otimes\gamma\rangle\langle\nabla_\theta f\cdot\nabla_\theta\phi(\cdot,\cdot,x),\Upsilon_s^N\otimes\gamma\rangle|^2\pi(\di x,\di y)\di s\nonumber\\
 &\quad +C\int_0^T\int_{\mathsf X\times\mathsf Y}|\langle\phi(\cdot,\cdot,x),\Upsilon_s^N\otimes\gamma\rangle\langle\nabla_\theta f\cdot\nabla_\theta\phi(\cdot,\cdot,x),\bar\mu_s^N\otimes\gamma\rangle|^2\pi(\di x,\di y)\di s\nonumber\\
&\quad +C\int_0^T\int_{\mathsf X\times\mathsf Y}|\langle\phi(\cdot,\cdot,x),\sqrt N(\bar\mu_s^N-\bar\mu_s)\otimes\gamma\rangle\langle\nabla_\theta f\cdot\nabla_\theta\phi(\cdot,\cdot,x),\bar\mu_s^N\otimes\gamma\rangle|^2\pi(\di x,\di y)\di s\nonumber\\
&\quad +C\int_0^T|\langle\nabla_\theta f\cdot\nabla_\theta\mathscr D_{\mathrm{KL}}(q^1_\cdot|P_0^1),\Upsilon_s^N\rangle|^2\di s\nonumber\\
&\quad +\frac{C}{N}\int_0^T\int_{\mathsf X\times\mathsf Y}\Big\langle\langle\phi(\cdot,\cdot,x)-y,\gamma\rangle\langle\nabla_\theta f\cdot\nabla_\theta\phi(\cdot,\cdot,x),\gamma\rangle,\mu_s^N\Big\rangle^2\pi(\di x,\di y)\di s\nonumber\\
&\quad +\frac{C}{N}\int_0^T\int_{\mathsf X\times\mathsf Y}\Big\langle (\phi(\cdot,\cdot,x)-y)\nabla_\theta f\cdot\nabla_\theta\phi(\cdot,\cdot,x),\mu_s^N\otimes\gamma\Big\rangle^2\pi(\di x,\di y)\di s\nonumber\\
&\quad + N\sup_{t\in[0,T]}\mathbf M_t^N[f]^2+ N\sup_{t\in[0,T]}\mathbf W_t^N[f]^2+ N\sup_{t\in[0,T]}\mathbf R_t^N[f]^2.
\end{align}
Let us now provide upper bounds on each term appearing in the right-hand side of \eqref{def<f,Upsilon>carre}.
Let us consider the first term in the right-hand side of \eqref{def<f,Upsilon>carre}. By \textbf{II}, for all $\mathfrak J\ge1$ and $\mathfrak  j\ge 0$, 
\begin{equation}\label{eq-nabla-f-nablaphi-H}
\sup_{g\in \mathcal H^{\mathfrak J,\mathfrak j}(\mathbf R^{d+1}), \|g\|_{\mathcal H^{\mathfrak J,\mathfrak j}}=1} \ \sup_{x\in\mathsf X}\big\| \nabla_\theta g\cdot \mathfrak H(\cdot ,x)\big\|_{\mathcal H^{\mathfrak J-1,\mathfrak j}} <+\infty.
\end{equation}
 By \eqref{bound-phi-y}, \eqref{eq-nabla-f-nablaphi-H}, the embedding $\mathcal H^{\mathfrak J_2,\mathfrak j_2}(\mathbf R^{d+1})\hookrightarrow \mathcal H^{\mathfrak J_1+1,\mathfrak j_1}(\mathbf R^{d+1})$ together with  Lemma \ref{lem-E[eta]<infty}, we have, for all $s\in[0,T]$,  
\begin{align}\label{l-cc1}
&\mathbf E\Big[\int_{\mathsf X\times\mathsf Y}|\langle\phi(\cdot,\cdot,x)-y,\mu_s^N\otimes\gamma\rangle\langle\nabla_\theta f\cdot\nabla_\theta\phi(\cdot,\cdot,x),\Upsilon_s^N\otimes\gamma\rangle|^2\pi(\di x,\di y)\Big]\nonumber\\
&\le C\mathbf E\Big[\int_{\mathsf X\times\mathsf Y} \big\|\nabla_\theta f\cdot\mathfrak H(\cdot ,x) \big\|_{\mathcal H^{\mathfrak J_1,\mathfrak j_1}}^2\|\Upsilon_s^N\|_{\mathcal H^{-\mathfrak J_1,\mathfrak j_1}}^2\pi(\di x,\di y)\Big]\nonumber\\
&\le C\|f\|_{\mathcal H^{\mathfrak J_1+1,\mathfrak j_1}}^2\mathbf E[\|\Upsilon_s^N\|_{\mathcal H^{-\mathfrak J_1,\mathfrak j_1}}^2]\le  C\|f\|_{\mathcal H^{\mathfrak J_1+1,\mathfrak j_1}}^2.
\end{align}

Let us now deal with the second term in the right hand side of \eqref{def<f,Upsilon>carre}. Using \eqref{phi-normeH}, Lemma \ref{lem-E[eta]<infty} and~\eqref{bound_nablafnablaphi,barmu}, and  Sobolev embeddings, we have, for all $s\in[0,T]$,  
\begin{align*}
\mathbf E[|\langle\phi(\cdot,\cdot,x),\Upsilon_s^N\otimes\gamma\rangle\langle\nabla_\theta f\cdot\nabla_\theta\phi(\cdot,\cdot,x),\bar\mu_s^N\otimes\gamma\rangle|^2] \le C\|f\|^2_{\mathcal H^{\mathfrak J_1,\mathfrak j_1}}\mathbf E[\|\Upsilon_s^N\|_{\mathcal H^{-\mathfrak J_1,\mathfrak j_1}}^2]\le C\|f\|^2_{\mathcal H^{\mathfrak J_1,\mathfrak j_1}},
\end{align*}
which provides the required upper bound. 

We now consider the t third term in the r.h.s. of   \eqref{def<f,Upsilon>carre}. 
We have, using    \eqref{bound_nablafnablaphi,barmu} and  \eqref{bound_<phi,barmu-barmu>}, together with the embedding $\mathcal H^{\mathfrak J_0,\mathfrak j_0}(\mathbf R^{d+1})\hookrightarrow\mathcal C^{1,\mathfrak j_0}(\mathbf R^{d+1})$, for all $s\in[0,T]$,  
\begin{align*}
\mathbf E[ |\langle\phi(\cdot,\cdot,x),\sqrt N(\bar\mu_s^N-\bar\mu_s)\otimes\gamma\rangle\langle\nabla_\theta f\cdot\nabla_\theta\phi(\cdot,\cdot,x),\bar\mu_s^N\otimes\gamma\rangle|^2]\le C\|f\|_{\mathcal H^{\mathfrak J_0,\mathfrak j_0}}^2.
\end{align*}
We now turn to the fourth term in \eqref{def<f,Upsilon>carre}. Note first that by \eqref{eq.boundKL}, we have that $\nabla_\theta g\cdot\nabla_\theta\mathscr D_{\mathrm{KL}}(q^1_\cdot|P_0^1)\in \mathcal H^{\mathfrak J-1,\mathfrak j+1}(\mathbf R^{d+1})$ for all $g\in\mathcal H^{\mathfrak J,\mathfrak j}(\mathbf R^{d+1})$, $\mathfrak J\ge 1$, $\mathfrak j\ge0$.
Moreover, we have 
\begin{equation}\label{nablafnablaKLnormeHfini}
\sup_{g\in\mathcal H^{\mathfrak J,\mathfrak j}(\mathbf R^{d+1}), \|g\|_{\mathcal H^{\mathfrak J,\mathfrak j}}=1 }\  \|\nabla_\theta g\cdot\nabla_\theta\mathscr D_{\mathrm{KL}}(q^1_\cdot|P_0^1)\|_{\mathcal H^{\mathfrak J-1,\mathfrak j+1}}<+\infty.
\end{equation}
  Hence, using the embedding $\mathcal H^{\mathfrak J_2,\mathfrak j_2}(\mathbf R^{d+1})\hookrightarrow\mathcal H^{\mathfrak J_1+1,\mathfrak j_1-1}(\mathbf R^{d+1})$ (see the beginning of Section \ref{sec-proof-clt}) and Lemma \ref{lem-E[eta]<infty},  we obtain, for all $s\in[0,T]$, 
\begin{align*}
\mathbf E[|\langle\nabla_\theta f\cdot\nabla_\theta\mathscr D_{\mathrm{KL}}(q^1_\cdot|P_0^1),\Upsilon_s^N\rangle|^2]\le \mathbf E[\|\langle\nabla_\theta f\cdot\nabla_\theta\mathscr D_{\mathrm{KL}}(q^1_\cdot|P_0^1)\|_{\mathcal H^{\mathfrak J_1,\mathfrak j_1}}^2\|\Upsilon_s^N\|_{\mathcal H^{-\mathfrak J_1,\mathfrak j_1}}^2]\le C\|f\|_{\mathcal H^{\mathfrak J_1+1,\mathfrak j_1-1}}^2.
\end{align*}
 By \textbf{A} and  Lemma \ref{le.Bounds},  the fifth and sixth terms in the r.h.s. of  \eqref{def<f,Upsilon>carre} are bounded by $C\|f\|_{\mathcal C^{1,\mathfrak  j_0}}^2$ and thus by $C\|f\|_{\mathcal H^{\mathfrak J_0,\mathfrak  j_0}}^2$.

We now turn to the three last terms of \eqref{def<f,Upsilon>carre}. Note first that $t\in\mathbf R_+ \mapsto\mathbf M_t^N[f]$ is a $\mathfrak F_t^N$-martingale, where $\mathfrak F_t^N=\mathcal F_{\lfloor Nt\rfloor}^N$ (to see this, use the same computations as those used in the proof of Lemma 3.2 in \cite{jmlr}). Now, using Equations (65), (61) and (62) in  \cite{colt}, we obtain, using Doob's inequality and Sobolev embeddings,  
\begin{align}
\mathbf E[\sup_{t\in[0,T]}\mathbf M_t^N[f]^2] =\mathbf E[\mathbf M_T^N[f]^2]\le C\|f\|_{\mathcal H^{\mathfrak J_0,\mathfrak j_0}}^2/N, \label{E[sup_M_t^N}\\
\mathbf E[\sup_{t\in[0,T]}\mathbf W_t^N[f]^2] \le C\|f\|_{\mathcal H^{\mathfrak J_0,\mathfrak j_0}}^2/N,\nonumber\\
\mathbf E[\sup_{t\in[0,T]}\mathbf R_t^N[f]^2]  \le C\|f\|_{\mathcal H^{\mathfrak J_0,\mathfrak j_0}}^2/N^2.\label{E[sup R_t^2]}
\end{align}
Collecting these bounds, we obtain
\begin{equation}
\mathbf E\Big[\sup_{t\in[0,T]}\langle f,\Upsilon_t^N\rangle^2\Big]\le C(\|f\|_{\mathcal H^{\mathfrak J_1+1,\mathfrak j_1}}^2+\|f\|_{\mathcal H^{\mathfrak J_1+1,\mathfrak j_1-1}}^2+\|f\|_{\mathcal H^{\mathfrak J_1,\mathfrak j_1}}^2+\|f\|_{\mathcal H^{\mathfrak J_0,\mathfrak j_0}}^2).
\end{equation}
Hence, by Sobolev  embeddings together with  the embedding $\mathcal H^{\mathfrak J_2,\mathfrak j_2}(\mathbf R^{d+1})\hookrightarrow_{\mathrm{H.S.}}\mathcal H^{\mathfrak J_1+1,\mathfrak j_1-1}(\mathbf R^{d+1})$,  one deduces that:
\begin{equation}\label{EsupUpsilon}
\mathbf E\Big[\sup_{t\in[0,T]}\|\Upsilon_t^N\|_{\mathcal H^{-\mathfrak J_2,\mathfrak j_2}}^2\Big]\le C.
\end{equation}
We now turn to the study of $\mathbf E[\sup_{t\in[0,T]}\langle f,\Theta_t^N\rangle^2]$. Recall that $\Theta_t^N=\sqrt N(\bar\mu_t^N-\bar\mu_t)$. Using \eqref{eq-<f,barmu_t^N}  and \eqref{eq.P2}  (recall that by Lemma \ref{Theta^N-continu}, one can use test functions $f\in \mathcal H^{\mathfrak J_0,\mathfrak j_0}(\mathbf R^{d+1})$ in \eqref{eq.P2}), one has:
\begin{align}\label{eq-<f,Theta_t^N}
\langle f,\Theta_t^N\rangle=\langle f,\Theta_0^N\rangle
&-\kappa\int_0^t\int_{\mathsf X\times\mathsf Y}\langle\phi(\cdot,\cdot,x)-y,\bar\mu_s\otimes\gamma\rangle\langle\nabla f\cdot\nabla_{\theta}\phi(\cdot,\cdot,x),\Theta_s^N\otimes\gamma\rangle\pi(\di x,\di y)\di s\nonumber\\
&-\kappa\int_0^t\langle\nabla f\cdot\nabla_{\theta}\mathscr D_{\mathrm{KL}}(q_\cdot^1|P_0^1),\Theta_s^N\rangle\di s.
\end{align}
By Jensen's inequality, together with \eqref{bound-phi-y} and Lemma \ref{lem-E[eta]<infty}, we obtain 
\begin{align*}
&\mathbf E\Big[\sup_{t\in[0,T]}\langle f,\Theta_t^N\rangle^2\Big]\\
&\le C\mathbf E[\langle f,\Theta_0^N\rangle^2] + C\mathbf E\Big[\int_0^T\int_{\mathsf X\times\mathsf Y}\langle\nabla f\cdot\nabla_{\theta}\phi(\cdot,\cdot,x),\Theta_s^N\otimes\gamma\rangle^2\pi(\di x,\di y)\di s\Big]\\
&\quad+C \mathbf E\Big[ \int_0^T\langle\nabla f\cdot\nabla_{\theta}\mathscr D_{\mathrm{KL}}(q_\cdot^1|P_0^1),\Theta_s^N\rangle^2\di s\Big]\\
&\le C\|f\|_{\mathcal H^{\mathfrak J_1,\mathfrak j_1}}^2+ C\|f\|_{\mathcal H^{\mathfrak J_1+1,\mathfrak j_1}}^2\int_0^T\mathbf E[\|\Theta_s^N\|_{\mathcal H^{-\mathfrak J_1,\mathfrak j_1}}^2]\di s +C\|f\|_{\mathcal H^{\mathfrak J_1+1,\mathfrak j_1-1}}^2\int_0^T\mathbf E[\|\Theta_s^N\|_{\mathcal H^{-\mathfrak J_1,\mathfrak j_1}}^2]\di s\\
&\le C (\|f\|_{\mathcal H^{\mathfrak J_1,\mathfrak j_1}}^2+\|f\|_{\mathcal H^{\mathfrak J_1+1,\mathfrak j_1}}^2+\|f\|_{\mathcal H^{\mathfrak J_1+1,\mathfrak j_1-1}}^2).
\end{align*}
Hence, by Sobolev embeddings (see the very beginning of Section \ref{sec-proof-clt}), we deduce that:
\begin{equation}
\mathbf E\Big[\sup_{t\in[0,T]}\|\Theta_t^N\|^2_{ \mathcal H^{-\mathfrak J_2,\mathfrak j_2}}\Big]\le C.
\end{equation}
Together with \eqref{EsupUpsilon}, this completes the proof of the lemma. 
\end{proof}

The following lemma provides the regularity condition needed to prove that the sequence of  fluctuation processes $(\eta^N)_{N\ge1}$ is relatively compact in the space  $\mathcal D(\mathbf R_+,\mathcal H^{-\mathfrak J_3+1,\mathfrak j_3}(\mathbf R^{d+1}))$.

\begin{lemma}\label{lem-reg-cond-eta}
Assume {\rm \textbf{A}}.
For all $T>0$, there exist $C>0$ such that for all $N\ge1$, $\delta>0$, $0\le r<t\le T$ with $t-r\le \delta$ and $f\in\mathcal C^\infty_c(\mathbf R^{d+1})$,   $ 
\mathbf E[|\langle f,\eta_t^N\rangle-\langle f,\eta_r^N\rangle|]\le C(\sqrt \delta + \delta + (1+  \delta)/\sqrt N) \|f\|_{\mathcal H^{\mathfrak J_1+1,\mathfrak j_1-1}}$. 
\end{lemma} 

\begin{proof}
From \eqref{def<f,Upsilon>},  $\langle f,\Upsilon_t^N\rangle-\langle f,\Upsilon_r^N\rangle$ is equal to:
\begin{align}\label{<f,Upsilon_t>-<fUpsilon_r}
&
 -\kappa\int_r^t\int_{\mathsf X\times\mathsf Y}\langle\phi(\cdot,\cdot,x)-y,\mu_s^N\otimes\gamma\rangle\langle\nabla_\theta f\cdot\nabla_\theta\phi(\cdot,\cdot,x),\Upsilon_s^N\otimes\gamma\rangle\pi(\di x,\di y)\di s\nonumber\\
 &\quad -\kappa\int_r^t\int_{\mathsf X\times\mathsf Y}\langle\phi(\cdot,\cdot,x),\Upsilon_s^N\otimes\gamma\rangle\langle\nabla_\theta f\cdot\nabla_\theta\phi(\cdot,\cdot,x),\bar\mu_s^N\otimes\gamma\rangle\pi(\di x,\di y)\di s\nonumber\\
&\quad -\kappa\int_r^t\int_{\mathsf X\times\mathsf Y}\langle\phi(\cdot,\cdot,x),\sqrt N(\bar\mu_s^N-\bar\mu_s)\otimes\gamma\rangle\langle\nabla_\theta f\cdot\nabla_\theta\phi(\cdot,\cdot,x),\bar\mu_s^N\otimes\gamma\rangle\pi(\di x,\di y)\di s\nonumber\\
&\quad -\kappa\int_r^t\langle\nabla_\theta f\cdot\nabla_\theta\mathscr D_{\mathrm{KL}}(q^1_\cdot|P_0^1),\Upsilon_s^N\rangle\di s\nonumber\\
&\quad +\frac{\kappa}{\sqrt N}\int_r^t\int_{\mathsf X\times\mathsf Y}\Big\langle\langle\phi(\cdot,\cdot,x)-y,\gamma\rangle\langle\nabla_\theta f\cdot\nabla_\theta\phi(\cdot,\cdot,x),\gamma\rangle,\mu_s^N\Big\rangle\pi(\di x,\di y)\di s\nonumber\\
\nonumber
&\quad -\frac{\kappa}{\sqrt N}\int_r^t\int_{\mathsf X\times\mathsf Y}\Big\langle (\phi(\cdot,\cdot,x)-y)\nabla_\theta f\cdot\nabla_\theta\phi(\cdot,\cdot,x),\mu_s^N\otimes\gamma\Big\rangle\pi(\di x,\di y)\di s\nonumber\\
&\quad +\sqrt N(\mathbf M_t^N[f]-\mathbf M_r^N[f])+\sqrt N(\mathbf W_t^N[f]-\mathbf W_r^N[f]) +\sqrt N(\mathbf R_t^N[f]-\mathbf R_r^N[f]).
\end{align}
Using similar techniques as those used in the proof of Lemma \ref{lem-cc-eta}, we obtain the following bounds:  
\begin{align*}
&\mathbf E\Big[\Big|\int_r^t\int_{\mathsf X\times\mathsf Y}\langle\phi(\cdot,\cdot,x)-y,\mu_s^N\otimes\gamma\rangle\langle\nabla_\theta f\cdot\nabla_\theta\phi(\cdot,\cdot,x),\Upsilon_s^N\otimes\gamma\rangle \pi(\di x,\di y)\di s\Big|\Big]\le C\|f\|_{\mathcal H^{\mathfrak J_1+1,\mathfrak j_1}}(t-r), \\
&\mathbf E\Big[\Big|\int_r^t\int_{\mathsf X\times\mathsf Y}\langle\phi(\cdot,\cdot,x),\Upsilon_s^N\otimes\gamma\rangle\langle\nabla_\theta f\cdot\nabla_\theta\phi(\cdot,\cdot,x),\bar\mu_s^N\otimes\gamma\rangle\pi(\di x,\di y)\di s\Big|\Big]\le  C\|f\|_{\mathcal H^{\mathfrak J_1,\mathfrak j_1}}(t-r), \\
&\mathbf E\Big[\Big|\int_r^t\int_{\mathsf X\times\mathsf Y}\langle\phi(\cdot,\cdot,x),\sqrt N(\bar\mu_s^N-\bar\mu_s)\otimes\gamma\rangle\langle\nabla_\theta f\cdot\nabla_\theta\phi(\cdot,\cdot,x),\bar\mu_s^N\otimes\gamma\rangle\pi(\di x,\di y)\di s\Big|\Big] \le C\|f\|_{\mathcal H^{\mathfrak J_0,\mathfrak j_0}}(t-r), \\
&\mathbf E\Big[\Big| \int_r^t\langle\nabla_\theta f\cdot\nabla_\theta\mathscr D_{\mathrm{KL}}(q^1_\cdot|P_0^1),\Upsilon_s^N\rangle\di s\Big|\Big]\le C\|f\|_{\mathcal H^{\mathfrak J_1+1,\mathfrak j_1-1}}(t-r), \\
&\mathbf E\Big[\Big|\frac{1}{\sqrt N}\int_r^t\int_{\mathsf X\times\mathsf Y}\Big\langle\langle\phi(\cdot,\cdot,x)-y,\gamma\rangle\langle\nabla_\theta f\cdot\nabla_\theta\phi(\cdot,\cdot,x),\gamma\rangle,\mu_s^N\Big\rangle\pi(\di x,\di y)\di s\Big|\Big] \le  C\frac{\|f\|_{\mathcal H^{\mathfrak J_0,\mathfrak j_0}}}{\sqrt N}(t-r),\\
&\mathbf E\Big[\Big| \frac{1}{\sqrt N}\int_r^t\int_{\mathsf X\times\mathsf Y}\Big\langle (\phi(\cdot,\cdot,x)-y)\nabla_\theta f\cdot\nabla_\theta\phi(\cdot,\cdot,x),\mu_s^N\otimes\gamma\Big\rangle\pi(\di x,\di y)\di s \Big|\Big]\le C\frac{\|f\|_{\mathcal H^{\mathfrak J_0,\mathfrak j_0}}}{\sqrt N}(t-r).
\end{align*}
Let us now treat the three last terms appearing at the last line of Equation~\eqref{<f,Upsilon_t>-<fUpsilon_r}.  
From the proof of  Lemma 21 in \cite{colt}, we have:
\begin{align*}
\mathbf E[|\mathbf M_t^N[f]-\mathbf M_r^N[f]|]\le C\frac{\sqrt{N\delta+1}}{N}\|f\|_{\mathcal C^{1,\mathfrak j_0}} \text{ and } \mathbf E[|\mathbf R_t^N[f]-\mathbf R_r^N[f]|]\le C\frac{\|f\|_{\mathcal C^{2,\mathfrak j_0}}}{N}.
\end{align*}
Let us mention that  the upper bound on $\mathbf E[|\mathbf W_t^N[f]-\mathbf W_r^N[f]|]$ provided in the proof of Lemma 21 in\cite{colt} (which we recall implies that this term is control by $1/\sqrt N$) is not sharp enough.  With  straightforward computations, from the definition  of $\mathbf W_t^N[f]$, we  actually have: 
\begin{align*}
\mathbf E[|\mathbf W_t^N[f]-\mathbf W_r^N[f]|]\le \mathbf E[|\mathbf W_t^N[f]|]+\mathbf E[|\mathbf W_t^N[f]|]\le C\frac{\|f\|_{\mathcal C^{1,j_0}}}{N}.
\end{align*} 
In conclusion,  using  Sobolev embeddings (see the very beginning of Section \ref{sec-proof-clt}), we obtain 
\begin{equation}\label{EUt-Ur}
\mathbf E[|\langle f,\Upsilon_t^N\rangle-\langle f,\Upsilon_r^N\rangle|]\le  C(\sqrt \delta + \delta + (1+  \delta)/\sqrt N)\|f\|_{\mathcal H^{\mathfrak J_1+1,\mathfrak j_1-1}}.
\end{equation}
Let us now consider $\Theta^N_t-\Theta^N_r$. 
By \eqref{eq-<f,Theta_t^N}, one has:
\begin{align}
\langle f,\Theta_t^N\rangle-\langle f,\Theta_r^N\rangle=
&-\kappa\int_r^t\int_{\mathsf X\times\mathsf Y}\langle\phi(\cdot,\cdot,x)-y,\bar\mu_s\otimes\gamma\rangle\langle\nabla f\cdot\nabla_{\theta}\phi(\cdot,\cdot,x),\Theta_s^N\otimes\gamma\rangle\pi(\di x,\di y)\di s\nonumber\\
&-\kappa\int_r^t\langle\nabla f\cdot\nabla_{\theta}\mathscr D_{\mathrm{KL}}(q_\cdot^1|P_0^1),\Theta_s^N\rangle\di s.
\end{align}
By \eqref{bound-phi-y}, \eqref{eq-nabla-f-nablaphi-H} and  \eqref{nablafnablaKLnormeHfini}, together with Lemma \ref{lem-E[eta]<infty}, it then holds: 
\begin{align}
\nonumber
\mathbf E\big[|\langle f,\Theta_t^N\rangle-\langle f,\Theta_r^N\rangle|\big]&\le C\|f\|_{\mathcal H^{\mathfrak J_1+1,\mathfrak j_1}}\int_r^t\mathbf E[\|\Theta_s^N\|_{ \mathcal H^{-\mathfrak J_1,\mathfrak j_1}}]\di s+ C\|f\|_{\mathcal H^{\mathfrak J_1+1,\mathfrak j_1-1}}\int_r^t\mathbf E[\|\Theta_s^N\|_{ \mathcal H^{-\mathfrak J_1,\mathfrak j_1}}]\di s\\
\label{E[Theotat-Thetar}
&\le C\delta(\|f\|_{\mathcal H^{\mathfrak J_1+1,\mathfrak j_1}}+\|f\|_{\mathcal H^{\mathfrak J_1+1,\mathfrak j_1-1}}).
\end{align}
Hence, by \eqref{EUt-Ur} and \eqref{E[Theotat-Thetar}, and recalling that $\eta^N=\Upsilon^N+\Theta^N$, we get that   $\mathbf E[|\langle f,\eta_t^N\rangle-\langle f,\eta_r^N\rangle|]\le  C(\sqrt \delta + \delta + (1+  \delta)/\sqrt N) \|f\|_{\mathcal H^{\mathfrak J_1+1,\mathfrak j_1-1}}$.  
\end{proof}

\begin{lemma}\label{lem_borneA_et_B}
Assume {\rm \textbf{A}}.
Let $(f_a)_{a\ge1}$ be an orthonormal basis of $\mathcal H^{\mathfrak J_1,\mathfrak j_1}(\mathbf R^{d+1})$. Then, for all $T> 0$, there exists $C>0$ such that for all $0\le t\le T$,  
\begin{enumerate}[label=\textit{(\roman*)}]
\item\label{lem1_it1} \begin{align*}
&\sum_{a\ge 1}\mathbf E\Big[-2\kappa\int_0^t\int_{\mathsf X\times\mathsf Y}\langle f_a,\Upsilon_s^N\rangle \langle\phi(\cdot,\cdot,x)-y,\mu_s^N\otimes\gamma\rangle\langle\nabla_\theta f_a\cdot\nabla_\theta\phi(\cdot,\cdot,x),\Upsilon_s^N\otimes\gamma\rangle\pi(\di x,\di y)\di s\Big]\\
 &\leq C\int_0^t\mathbf E\Big[ \|\Upsilon_s^N\|_{\mathcal H^{-\mathfrak J_1,\mathfrak j_1}}^2\Big]\di s. 
\end{align*}

\item\label{lem1_it2} \begin{align*}
\sum_{a\ge1}\mathbf E\Big[ -2\kappa\int_0^t\langle f_a,\Upsilon_s^N\rangle\langle\nabla_\theta f_a\cdot\nabla_\theta\mathscr D_{\mathrm{KL}}(q^1_\cdot|P_0^1),\Upsilon_s^N\rangle\di s \Big] \le C\int_0^t\mathbf E\Big[ \|\Upsilon_s^N\|_{\mathcal H^{-\mathfrak J_1,\mathfrak j_1}}^2\Big]\di s.
\end{align*}

\item\label{lem1_it3} \begin{align*}
&\sum_{a\ge1}\mathbf E\Big[ -2\kappa\int_0^t\int_{\mathsf X\times\mathsf Y}\langle f_a,\Upsilon_s^N\rangle\langle\phi(\cdot,\cdot,x),\Upsilon_s^N\otimes\gamma\rangle\langle\nabla_\theta f_a\cdot\nabla_\theta\phi(\cdot,\cdot,x),\bar\mu_s^N\otimes\gamma\rangle\pi(\di x,\di y)\di s\nonumber\\
&\quad\quad -2\kappa\int_0^t\int_{\mathsf X\times\mathsf Y}\langle f_a,\Upsilon_s^N\rangle\langle\phi(\cdot,\cdot,x),\sqrt N(\bar\mu_s^N-\bar\mu_s)\otimes\gamma\rangle\langle\nabla_\theta f_a\cdot\nabla_\theta\phi(\cdot,\cdot,x),\bar\mu_s^N\otimes\gamma\rangle\pi(\di x,\di y)\di s\Big]\\
&\le C+C \int_0^t\mathbf E[\|\Upsilon_s^N\|_{\mathcal H^{-\mathfrak J_1,\mathfrak j_1}}^2]\di s.  
\end{align*}
\item\label{lem1_it4}\begin{align*}
\sum_{a\ge1}\mathbf E&\Big[\frac{2\kappa}{\sqrt N}\int_0^t\int_{\mathsf X\times\mathsf Y}\langle f_a,\Upsilon_s^N\rangle\Big\langle\langle\phi(\cdot,\cdot,x)-y,\gamma\rangle\langle\nabla_\theta f_a\cdot\nabla_\theta\phi(\cdot,\cdot,x),\gamma\rangle,\mu_s^N\Big\rangle\pi(\di x,\di y)\di s\nonumber\\
&\quad -\frac{2\kappa}{\sqrt N}\int_0^t\int_{\mathsf X\times\mathsf Y}\langle f_a,\Upsilon_s^N\rangle\Big\langle (\phi(\cdot,\cdot,x)-y)\nabla_\theta f_a\cdot\nabla_\theta\phi(\cdot,\cdot,x),\mu_s^N\otimes\gamma\Big\rangle\pi(\di x,\di y)\di s\Big]\\
&\le C+C\int_0^t\mathbf E\Big[|\|\Upsilon_s^N\|_{\mathcal H^{-\mathfrak J_1,\mathfrak j_1}}^2\Big]\di s.
\end{align*}
\item \label{lem1_it5} \begin{equation*}
\sum_{a\ge1}\mathbf E\Big[\sum_{k=0}^{\lfloor Nt\rfloor-1}\Big[2\langle f_a,\Upsilon_{\frac{k+1}{N}^-}^N\rangle\sqrt N\mathbf M_k^N[f_a]+3N\mathbf M_k^N[f_a]^2\Big]\Big]\le C.
\end{equation*}
\item \label{lem1_it6}\begin{align*}
\sum_{a\ge1}\mathbf E\Big[\sum_{k=0}^{\lfloor Nt\rfloor-1}\Big[2\langle f_a,\Upsilon_{\frac{k+1}{N}^-}^N\rangle\sqrt N\mathbf R_k^N[f_a]+3N\mathbf R_k^N[f_a]^2\Big]\Big] \le C+ \int_0^t\mathbf E\Big[\|\Upsilon_s^N\|_{\mathcal H^{-\mathfrak J_1,\mathfrak j_1}}^2\Big]\di s.
\end{align*}

\item\label{lem1_it7} \begin{align*}
\sum_{a\ge1}\mathbf E\Big[\sum_{k=0}^{\lfloor Nt\rfloor-1}\Big[2\langle f_a,\Upsilon_{\frac{k+1}{N}^-}^N\rangle\mathbf a_k^N[f_a]+3\mathbf a_k^N[f_a]^2\Big]-2\sqrt N\int_0^t\langle f_a,\Upsilon_s^N\rangle\mathbf L_s^N[f]\di s\Big]\le C.
\end{align*}
\end{enumerate}
\end{lemma}

\begin{proof}
Let $0\le t\le T$ and $N\ge1$. Consider an orthonormal basis $(f_a)_{a\ge1}$ of $\mathcal H^{\mathfrak J_1,\mathfrak j_1}(\mathbf R^{d+1})$ and  a function $f\in\mathcal H^{\mathfrak J_1,\mathfrak j_1}(\mathbf R^{d+1})$. In what follows,   $C>0$ will denote a constant    independent of $t$, $N$, $s\in [0,t]$, $f$ and $(f_a)_{a\ge1}$, which can change from one occurrence to another. 
Let us prove item \ref{lem1_it1}. Introduce for $x\in \mathsf X$,  the operator $\mathbf T_x: \mathcal H^{\mathfrak J_1,\mathfrak j_1}(\mathbf R^{d+1})\to \mathcal H^{\mathfrak J_1-1,\mathfrak j_1}(\mathbf R^{d+1})$ defined by 
\begin{equation}\label{operator-Tx}
\theta \in \mathbf R^{d+1}\mapsto \mathbf T_x (f)(\theta)=\nabla_\theta f(\theta)\cdot \nabla_\theta \int_{\mathbf R^d} \phi(\theta,z,x)\gamma(z)\di z=\nabla _\theta f\cdot \mathfrak H(\cdot, x),
\end{equation}
 where we recall that $\phi(\theta,z,x)=s(\Psi_\theta(z),x)$. Note that $\mathbf T_x$ is well defined  since the function $\mathfrak H(\cdot ,x): \theta \mapsto \int_{\mathbf R^d} \phi(\theta,z,x)\gamma(z)\di z=\langle \phi(\theta,\cdot,x),\gamma \rangle$ is smooth  and all its derivatives of non negative order are uniformly bounded w.r.t $x\in \mathsf X$ over $\mathbf R^{d+1}$  (this follows from  \textbf{A1} and \textbf{A3}).
Then, one has 
\begin{align*}
&\sum_{a\ge 1}-2\kappa\int_0^t\int_{\mathsf X\times\mathsf Y}\langle f_a,\Upsilon_s^N\rangle \langle\phi(\cdot,\cdot,x)-y,\mu_s^N\otimes\gamma\rangle\langle\nabla_\theta f_a\cdot\nabla_\theta\phi(\cdot,\cdot,x),\Upsilon_s^N\otimes\gamma\rangle\pi(\di x,\di y)\di s\\
&=-2\kappa\int_0^t\int_{\mathsf X\times\mathsf Y}\langle\phi(\cdot,\cdot,x)-y,\mu_s^N\otimes\gamma\rangle\sum_{a\ge1}\langle f_a,\Upsilon_s^N\rangle \langle\mathbf T_x f_a,\Upsilon_s^N\rangle\pi(\di x,\di y)\di s\\
&=-2\kappa\int_0^t\int_{\mathsf X\times\mathsf Y}\langle\phi(\cdot,\cdot,x)-y,\mu_s^N\otimes\gamma\rangle\langle\Upsilon_s^N,\mathbf T_x^*\Upsilon_s^N\rangle_{\mathcal H^{-\mathfrak J_1,\mathfrak j_1}}\pi(\di x,\di y)\di s.
\end{align*}
 Since the function $\phi$ is bounded and $\mathsf Y$ is compact, one has:
\begin{equation}\label{bound-phi-y}
\exists C>0, \forall\nu\in\mathcal P(\mathbf R^{d+1}), \forall (x,y)\in\mathsf X\times\mathsf Y,\  |\langle\phi(\cdot,\cdot,x)-y,\nu\otimes\gamma\rangle|\le C. 
\end{equation}
By \eqref{bound-phi-y} and using Lemma B.2 in \cite{jmlr}  (note that $\Upsilon^N\in\mathcal D(\mathbf R_+, \mathcal H^{-\mathfrak J_1+1,\mathfrak j}(\mathbf R^{d+1}))$ by \eqref{eq.up-app} together with the Sobolev embedding $\mathcal H^{\mathfrak J_1-1,\mathfrak j}(\mathbf R^{d+1})\hookrightarrow \mathcal C^{1,\mathfrak j_1}(\mathbf R^{d+1})$, $\mathfrak j\ge 0$), we have 
\begin{align*}
&\mathbf E\Big[\sum_{a\ge 1}-2\kappa\int_0^t\int_{\mathsf X\times\mathsf Y}\langle f_a,\Upsilon_s^N\rangle \langle\phi(\cdot,\cdot,x)-y,\mu_s^N\otimes\gamma\rangle\langle\nabla_\theta f_a\cdot\nabla_\theta\phi(\cdot,\cdot,x),\Upsilon_s^N\otimes\gamma\rangle\pi(\di x,\di y)\di s\Big]\\
& \leq C\int_0^t\mathbf E\Big[ \|\Upsilon_s^N\|_{\mathcal H^{-\mathfrak J_1,\mathfrak j_1}}^2\Big]\di s,
\end{align*}
which is the   desired estimate.

Introduce  the operator $\mathbf T: \mathcal H^{\mathfrak J_1,\mathfrak j_1}(\mathbf R^{d+1})\to  \mathcal H^{\mathfrak J_1-1,\mathfrak j_1+1}(\mathbf R^{d+1})$ defined by (see also \eqref{eq.boundKL})
\begin{equation}\label{operator-T}
\mathbf T(f): \theta \mapsto  \nabla_\theta f\cdot\nabla_\theta\mathscr D_{\mathrm{KL}}(q^1_\cdot|P_0^1), 
\end{equation}  
Item \ref{lem1_it2} is proved as the previous item, using now   Lemma \ref{lem_B1_avec_poids} below.

Item \ref{lem1_it3} is obtained with exactly the same arguments as those used to derive  the upper bounds on $\sum_{a\ge1}\mathbf J_t^N[f_a]$ and $\sum_{a\ge1}\mathbf K_t^N[f_a]$ in the proof of  Lemma 3.1 in \cite{jmlr}   (it suffices indeed to change $\sigma(\cdot ,x)$ there into $\mathfrak H(\cdot ,x)$). In particular, by \textbf{II} and \eqref{eq.Xt}, it holds:
\begin{align}\label{bound_nablafnablaphi,barmu}
\sup_{x\in\mathsf X}|\langle\nabla_\theta f\cdot\nabla_\theta\phi(\cdot,\cdot,x),\bar\mu_s^N\otimes\gamma\rangle|\le C\|f\|_{\mathcal C^{1,\mathfrak j_0}},
\end{align}
and (see Equation (3.20) in \cite{jmlr}), 
\begin{align}\label{bound_<phi,barmu-barmu>}
&\mathbf E[\langle\phi(\cdot,\cdot,x),(\bar\mu_s^N-\bar\mu_s)\otimes\gamma\rangle^2]\le  {C}/{N}.
\end{align}
 Note also that by Lemma \ref{le.Bounds} and \textbf{I}, it holds:
\begin{align}\label{eq:nablafnablaphi}
\mathbf E\big[\langle|\nabla_\theta f\cdot\nabla_\theta\phi(\cdot,\cdot,x)|,\mu_s^N\otimes\gamma\rangle^2\big]\le C\|f\|_{\mathcal C^{1,\mathfrak j_0}}^2
\end{align}
Item \ref{lem1_it4} follows from $\mathcal H^{\mathfrak J_0,\mathfrak j_0}(\mathbf R^{d+1})\hookrightarrow \mathcal C^{1,\mathfrak j_0}(\mathbf R^{d+1}) $ and  $\mathcal H^{\mathfrak J_1,\mathfrak j_1}(\mathbf R^{d+1})\hookrightarrow_{\mathrm{H.S.}}\mathcal H^{\mathfrak J_0,\mathfrak j_0}(\mathbf R^{d+1})$.

Let us prove item \ref{lem1_it5}. 
Since $\mathbf E[\mathbf M_k^N[f]|\mathcal F_k^N]=0$, we have with the same arguments as those used to derive Equation (B.1) in  \cite{jmlr}, 
$$\sum_{k=0}^{\lfloor Nt\rfloor-1}\mathbf E [\langle f,\Upsilon_{\frac{k+1}{N}^-}^N\rangle\sqrt N\mathbf M_k^N[f]]=0.$$ 
 Moreover,  we recall that by Lemma \ref{le.Bounds} (see Eqaution (60) in \cite{colt}), one has $
\mathbf E[\mathbf M_k^N[f]^2]\le C\|f\|_{\mathcal C^{1,j_0}}^2/N^2$. 
Hence, we conclude, using  again  $\mathcal H^{\mathfrak J_0,\mathfrak j_0}(\mathbf R^{d+1})\hookrightarrow \mathcal C^{1,\mathfrak j_0}(\mathbf R^{d+1}) $ and  $\mathcal H^{\mathfrak J_1,\mathfrak j_1}(\mathbf R^{d+1})\hookrightarrow_{\mathrm{H.S.}}\mathcal H^{\mathfrak J_0,\mathfrak j_0}(\mathbf R^{d+1})$, that 
\begin{equation}
\sum_{a\ge1}\mathbf E\Big[\sum_{k=0}^{\lfloor Nt\rfloor-1}\Big[2\langle f_a,\Upsilon_{\frac{k+1}{N}^-}^N\rangle\sqrt N\mathbf M_k^N[f_a]+3N\mathbf M_k^N[f_a]^2\Big]\Big]\le C\sum_{a\ge1}\|f_a\|^2_{\mathcal C^{1,j_0}}\le C.
\end{equation}
Let us prove item \ref{lem1_it6}.
We have
\begin{align*}
\sum_{k=0}^{\lfloor Nt\rfloor-1}\langle f,\Upsilon^N_{\frac{k+1}{N}^-}\rangle\sqrt N\mathbf R_k^N[f]\leq \sum_{k=0}^{\lfloor Nt\rfloor-1}\frac 1N \langle f,\Upsilon^N_{\frac{k+1}{N}^-}\rangle^2 + \sum_{k=0}^{\lfloor Nt\rfloor-1}N^2\mathbf R_k^N[f]^2.
\end{align*} 
Recall that from the analysis performed at the end of the proof of Lemma B.1 in \cite{colt}, $\mathbf E[\mathbf R_k^N[f]^2]\le C/N^4$ so that 
\begin{align*}
\mathbf E\Big[\sum_{k=0}^{\lfloor Nt\rfloor-1}N^2\mathbf R_k^N[f]^2\Big] \le  C\|f\|_{\mathcal C^{2,\mathfrak j_0}}^2/N.
\end{align*} 
Using   \eqref{eq.Xt} and Lemma \ref{le.Bounds}, the same computations as those of the proof of item \textit{(iv)} in Lemma B.1 in \cite{jmlr} yield:
\begin{align*}
\mathbf E\Big[\sum_{k=0}^{\lfloor Nt\rfloor-1}\frac 1N\langle f,\Upsilon_{\frac{k+1}{N}^-}^N\rangle^2\Big]\le C \|f\|^2_{\mathcal C^{2,\mathfrak j_0}}+ \mathbf E\Big[\int_0^t\langle f,\Upsilon_s^N\rangle^2\di s\Big]. 
\end{align*}
Hence, 
\begin{equation}
\sum_{k=0}^{\lfloor Nt\rfloor-1}\langle f,\Upsilon^N_{\frac{k+1}{N}^-}\rangle\sqrt N\mathbf R_k^N[f] \le C \|f\|^2_{\mathcal C^{2,\mathfrak j_0}}+ \mathbf E\Big[\int_0^t\langle f,\Upsilon_s^N\rangle^2\di s\Big]+ C\|f\|_{\mathcal C^{2,\mathfrak j_0}}^2/N.
\end{equation}
 Item \ref{lem1_it6} then follows from $\mathcal H^{\mathfrak J_0,\mathfrak j_0}(\mathbf R^{d+1})\hookrightarrow \mathcal C^{2,\mathfrak j_0}(\mathbf R^{d+1}) $ and  $\mathcal H^{\mathfrak J_1,\mathfrak j_1}(\mathbf R^{d+1})\hookrightarrow_{\mathrm{H.S.}}\mathcal H^{\mathfrak J_0,\mathfrak j_0}(\mathbf R^{d+1})$. 

Let us prove item \ref{lem1_it7}.
Using Jensen's inequality together with Lemma \ref{le.Bounds}  and \eqref{eq.boundKL}, we have, for all $0\le s \le t$,  
\begin{align}\label{bound_E[L_s]}
\mathbf E[|\mathbf L_s^N[f]|^2] \le C\|f\|^2_{\mathcal C^{1,j_0}}(1+1/N). 
\end{align}
On the other hand,    for all $s\in(\frac kN,\frac{k+1}{N})$, by  \eqref{eq.Xt} and the same computations as those used to derive Equation (B.5) in  \cite{jmlr}, we have:
\begin{equation}
|\langle f,\Upsilon_{\frac{k+1}{N}^-}^N\rangle-\langle f,\Upsilon_s^N\rangle| =\sqrt N\, \big|\langle f,\bar
\mu_s^N\rangle-\langle f,\bar \mu_{\frac{k+1}N}^N\rangle\big|\le C\| f\|_{\mathcal C^{2,\mathfrak j_0}}.
\end{equation} 
Hence, 
\begin{align}
\nonumber
&\mathbf E\Big[ \sum_{k=0}^{\lfloor Nt\rfloor-1}\langle f,\Upsilon_{\frac{k+1}{N}^-}^N\rangle\mathbf a_k^N[f]-\sqrt N\int_0^{\frac{\lfloor Nt\rfloor}{N}}\langle f,\Upsilon_s^N\rangle\mathbf L_s^N[f]\di s\Big]\\
&=\sqrt N\sum_{k=0}^{\lfloor Nt\rfloor-1}\int_{\frac{k}{N}}^{\frac{k+1}{N}}\mathbf E\Big[\Big(\langle f,\Upsilon_{\frac{k+1}{N}^-}^N\rangle-\langle f,\Upsilon_s^N\rangle\Big)\mathbf L_s^N[f]\Big]\di s\nonumber\\
\label{E[item7-1}
&\le C\|f\|_{\mathcal C^{2,\mathfrak j_0}}\int_0^{\frac{\lfloor Nt\rfloor}{N}}\mathbf E[|\mathbf L_s^N[f]|]\di s \le C\|f\|_{\mathcal C^{2,\mathfrak j_0}}^2.
\end{align} 
We also have, using Lemma \ref{le.Bounds}  and \eqref{eq.Xt}, it is straightforward to deduce that $\mathbf E[\langle f,\Upsilon_s^N\rangle^2]\le CN\|f\|_{\mathcal C^{1,j_0}}^2$. Consequently, one has:
\begin{align}\label{E[item7-2}
\mathbf E\Big[\sqrt N\Big|\int_{\frac{\lfloor Nt\rfloor}{N}}^t\langle f,\Upsilon_s^N\rangle\mathbf L_s^N[f]\di s\Big|\Big] \le \sqrt N\int_{\frac{\lfloor Nt\rfloor}{N}}^t\sqrt{\mathbf E[\langle f,\Upsilon_s^N\rangle^2]}\sqrt{\mathbf E[\mathbf L_s^N[f]^2]}\le C\|f\|_{\mathcal C^{1,j_0}}^2.
\end{align}
Finally, 
\begin{align}
\nonumber
\mathbf E\Big[\sum_{k=0}^{\lfloor Nt\rfloor-1}\mathbf a_k^N[f]^2\Big] = N\mathbf E\Big[ \sum_{k=0}^{\lfloor Nt\rfloor-1}\Big|\int_{\frac kN}^{\frac{k+1}{N}}\mathbf L_s^N[f]\di s\Big|^2\Big] &\le \sum_{k=0}^{\lfloor Nt\rfloor-1}\int_{\frac kN}^{\frac{k+1}{N}}\mathbf E[\mathbf L_s^N[f]^2]\di s\\
\label{E[a_k]}
&\le C\|f\|^2_{\mathcal C^{1,j_0}}(1+1/N).
\end{align}
Item \ref{lem1_it7} follows from \eqref{E[item7-1}, \eqref{E[item7-2} and  \eqref{E[a_k]}. The proof of the lemma is complete.  
\end{proof}

\begin{lemma}\label{lem_B1_avec_poids}
Let $\mathfrak J\ge 1$ and $\mathfrak j\ge 0$. Recall the definition of $\mathbf T \in \mathcal L( \mathcal H^{\mathfrak J,\mathfrak j}(\mathbf R^{d+1}),  \mathcal H^{\mathfrak J-1,\mathfrak j+1}(\mathbf R^{d+1}))$ in \eqref{operator-T}. Then, there exists $C>0$ such that for any $\Upsilon\in \mathcal H^{-\mathfrak J+1,\mathfrak j+1}(\mathbf R^{d+1})$,  
\begin{align}\label{eq lem upsilon}
|\langle\Upsilon,\mathbf T^*\Upsilon\rangle_{\mathcal H^{-\mathfrak J,\mathfrak j}}|\le C\|\Upsilon\|_{\mathcal H^{-\mathfrak J,\mathfrak j}}^2.
\end{align}
\end{lemma}
 Note that $\mathbf T^* \in \mathcal L( 
  \mathcal H^{-\mathfrak J+1,\mathfrak j+1}(\mathbf R^{d+1}), \mathcal H^{-\mathfrak J,\mathfrak j}(\mathbf R^{d+1}))$. Let us mention that the upper bound \eqref{eq lem upsilon} is much better than the one which would be obtained applying   the Cauchy-Schwarz inequality. 
\begin{proof}
 The proof is inspired from the one of Lemma B.2 in \cite{jmlr} (see also~ Lemma B1 in \cite{sirignano2020clt}). We will give  the proof in dimension $1$, i.e. when  $d=0$,  the other cases are treated the same way. Let $\Upsilon\in \mathcal H^{-\mathfrak J+1,\mathfrak j+1}(\mathbf R)\hookrightarrow \mathcal H^{-\mathfrak J,\mathfrak j}(\mathbf R)$. By the Riesz representation theorem, there exists a unique $\Psi\in \mathcal H^{\mathfrak J,\mathfrak j}(\mathbf R)$ such that
\begin{align*}
\langle f,\Upsilon\rangle=\langle f,\Psi\rangle_{\mathcal H^{\mathfrak J,\mathfrak j}},\ \forall f\in \mathcal H^{\mathfrak J,\mathfrak j}(\mathbf R). 
\end{align*} 
Define $F$ by $F(\Upsilon)=\Psi$. The density of $\mathcal C^\infty_c(\mathbf R)$ in $\mathcal H^{\mathfrak J,\mathfrak j}(\mathbf R)$ implies that $\{\Upsilon\in \mathcal H^{-\mathfrak J,\mathfrak j}(\mathbf R): F(\Upsilon)\in\mathcal C^\infty_c(\mathbf R)\}$ is dense in $\mathcal H^{-\mathfrak J,\mathfrak j}(\mathbf R)$. It is thus sufficient to show \eqref{eq lem upsilon}  when  $\Psi=F(\Upsilon)\in \mathcal C^\infty_c(\mathbf R)$. We have 
\begin{equation}
\langle\Upsilon,\mathbf T\Upsilon\rangle_{\mathcal H^{-\mathfrak J,\mathfrak j}}=\langle\Psi,\mathbf T^*\Upsilon\rangle=\langle \mathbf T\Psi,\Upsilon\rangle=\langle\mathbf T\Psi,\Psi\rangle_{\mathcal H^{\mathfrak J,\mathfrak j}}.
\end{equation}
Hence, to prove \eqref{eq lem upsilon}, it is enough to show   $|\langle \mathbf T\Psi,\Psi\rangle_{\mathcal H^{\mathfrak J,\mathfrak j}}|\le C\|\Psi\|_{\mathcal H^{\mathfrak J,\mathfrak j}}^2$ for $\Psi\in \mathcal C^\infty_c(\mathbf R)$. 
We will only consider the case when  
$\mathfrak J=\mathfrak j=1$, the other cases being treated very similarly. Recall the upper bounds \eqref{eq.boundKL}.
Let $\Psi\in\mathcal C^\infty_c(\mathbf R)$. 
We have, by integration by parts and using the fact that $\Psi$ is compactly supported, 
\begin{align}\label{lem_Upsilon_eq1}
\langle \mathbf T\Psi,\Psi\rangle_{\mathcal H^{1,1}}&
=\int_{\mathbf R} \Psi'(\theta)\mathscr D_{\mathrm{KL}}'(q^1_\theta|P_0^1)\frac{\Psi(\theta)}{1+\theta^2}\di\theta+\int_{\mathbf R} (\Psi'(\theta)\mathscr D_{\mathrm{KL}}'(q^1_\theta|P_0^1))'\frac{\Psi'(\theta)}{1+\theta^2}\di\theta \nonumber\\
\nonumber
&= \int_{\mathbf R} \Psi'(\theta)\mathscr D_{\mathrm{KL}}'(q^1_\theta|P_0^1)\frac{\Psi(\theta)}{1+\theta^2}\di\theta
+\int_{\mathbf R} \Psi''(\theta)\mathscr D_{\mathrm{KL}}'(q^1_\theta|P_0^1)\frac{\Psi'(\theta)}{1+\theta^2}\di\theta
\\
\nonumber
&\quad +\int_{\mathbf R} \mathscr D_{\mathrm{KL}}''(q^1_\theta|P_0^1)\frac{\Psi'(\theta)^2}{1+\theta^2}\di\theta\nonumber\\
\nonumber
&=-\frac 12\int_{\mathbf R}\Psi(\theta)^2\frac{\di}{\di\theta}\Big(\frac{\mathscr D_{\mathrm{KL}}'(q^1_\theta|P_0^1)}{1+\theta^2}\Big)\di \theta 
- \frac 12\int_{\mathbf R}\Psi'(\theta)^2\frac{\di}{\di\theta}\Big(\frac{\mathscr D_{\mathrm{KL}}'(q^1_\theta|P_0^1)}{1+\theta^2}\Big)\di \theta
\\
&\quad +\int_{\mathbf R} \mathscr D_{\mathrm{KL}}''(q^1_\theta|P_0^1)\frac{\Psi'(\theta)^2}{1+\theta^2}\di\theta.
\end{align}
To bound the first two terms of \eqref{lem_Upsilon_eq1}, we use the bounds \eqref{eq.boundKL}. 
More precisely, for all $\theta\in\mathbf R$, 
\begin{align*}
\Big|\frac{\di}{\di\theta}\Big(\frac{\mathscr D_{\mathrm{KL}}'(q^1_\theta|P_0^1)}{1+\theta^2}\Big)\Big|\le   \frac{|\mathscr D_{\mathrm{KL}}''(q^1_\theta|P_0^1)(1+\theta^2)|+2|\theta\mathscr D_{\mathrm{KL}}'(q^1_\theta|P_0^1)|}{(1+\theta^2)^2} \le \frac{C}{1+\theta^2}+ \frac{C|\theta|(1+\theta|)}{(1+\theta^2)^2}\le \frac{C}{1+\theta^2}. 
\end{align*}
Hence, we obtain, plugging this bound in \eqref{lem_Upsilon_eq1}, 
\begin{align*}
|\langle \mathbf T\Psi,\Psi\rangle_{\mathcal H^{1,1}}|\le C\Big(\int_\mathbf R\frac{\Psi(\theta)^2}{1+\theta^2}\di\theta+\int_\mathbf R\frac{\Psi'(\theta)^2}{1+\theta^2}\di\theta\Big)\leq C\|\Psi\|_{\mathcal H^{1,1}}^2. 
\end{align*} 
This completes the proof of the lemma. 
\end{proof}

We now collect the previous results to prove Proposition \ref{p-rc-eta}.
\begin{proof}[Proof of Proposition \ref{p-rc-eta}] 
The proof consists in applying Th. 4.6 in \cite{jakubowski1986skorokhod} with $E= \mathcal H^{-\mathfrak J_3+1,\mathfrak j_3}(\mathbf R^{d+1})$ and $\mathbb F= \{\mathsf H_f, \ f\in\mathcal C^{\infty}_c(\mathbf R^{d+1})\}$ where 
$$\mathsf H_f: \nu\in \mathcal H^{-\mathfrak J_3+1,\mathfrak j_3}(\mathbf R^{d+1})\mapsto\langle f,\nu\rangle.$$
Note that $\mathcal H^{\mathfrak J_3-1,\mathfrak j_3}(\mathbf R^{d+1})$ is compactly embedded in $\mathcal H^{\mathfrak J_2,\mathfrak j_2}(\mathbf R^{d+1})$. Hence, by Schauder's theorem, $\mathcal H^{-\mathfrak J_2,\mathfrak j_2}(\mathbf R^{d+1})$ is compactly embedded in $\mathcal H^{-\mathfrak J_3+1,\mathfrak j_3}(\mathbf R^{d+1})$. Thus, for all $C>0$, the set $\{ h\in \mathcal H^{-\mathfrak J_3+1,\mathfrak j_3}(\mathbf R^{d+1}), \ \|h\|_{\mathcal H^{-\mathfrak J_2,\mathfrak j_2}}\le C\}$ is compact. Hence, Condition (4.8) in Th. 4.6 in \cite{jakubowski1986skorokhod} follows from Lemma \ref{lem-cc-eta} and Markov's inequality. 
Let us now show that Condition (4.9) in \cite{jakubowski1986skorokhod} is verified, i.e., that for all $f\in\mathcal C^\infty_c(\mathbf R^{d+1})$, the sequence $(\langle f,\eta^N\rangle)_{N\ge1}$ is relatively compact in $\mathcal D(\mathbf R_+,\mathbf R)$. To do this, it suffices to use Lemma \ref{lem-reg-cond-eta}
and Prop. A.1 in \cite{jmlr} (with $\mathcal H_1=\mathcal H_2=\mathbf R$ there). In conclusion, according to Th. 4.6 in \cite{jakubowski1986skorokhod}, the sequence $(\eta^N)_{N\ge1}$ is relatively compact in $\mathcal D(\mathbf R_+,\mathcal H^{-\mathfrak J_3+1,\mathfrak j_3}(\mathbf R^{d+1}))$.
\end{proof}

\subsection{Relative compactness of $(\sqrt N\mathbf M^N)_{N\ge1}$ and regularity of the limit points}
Throughout this section, we that the $\{\theta^i_k, i\in \{1,\ldots,N\}\}$'s are  generated    by the algorithm  \eqref{eq.algo-batch} (with straightforward modifications, one can check that all the results of  this section are valid when the  $\{\theta^i_k, i\in \{1,\ldots,N\}\}$'s are  generated    by   the algorithms \eqref{eq.algo-ideal} and \eqref{eq.algo-z1z2}). 
\begin{lemma}\label{lem-ccM^N}
Assume {\rm \textbf{A}}.
Then, for all $T>0$,  $\sup_{N\ge1}\mathbf E\Big[\sup_{t\in[0,T]}\|\sqrt N\mathbf M^N_t\|_{\mathcal H^{-\mathfrak J_1,\mathfrak j_1}}^2\Big]<+\infty$. 
\end{lemma}

\begin{proof}
\begin{sloppypar}
Recall that by \eqref{E[sup_M_t^N}, there exists $C>0$ such that for all $f\in\mathcal H^{\mathfrak J_1,\mathfrak j_1}(\mathbf R^{d+1})$ and $N\ge1$, 
$$\mathbf E[\sup_{t\in[0,T]}|\sqrt N\mathbf M_t^N[f]|^2] \le C\|f\|_{\mathcal H^{\mathfrak J_0,\mathfrak j_0}}^2. $$
Considering an orthonormal basis of $\mathcal H^{\mathfrak J_1,\mathfrak j_1}(\mathbf R^{d+1})\hookrightarrow_{\mathrm{H.S.}}\mathcal H^{\mathfrak J_0,\mathfrak j_0}(\mathbf R^{d+1})$, one gets that $\mathbf E [\sup_{t\in[0,T]}\|\sqrt N\mathbf M^N_t\|_{\mathcal H^{-\mathfrak J_1,\mathfrak j_1}}^2 ]\le C$ uniformly in $N\ge 1$.  \end{sloppypar}
\end{proof}

We now turn to the regularity condition on the sequence $\{t\in\mathbf R_+\mapsto\sqrt N\mathbf M_t^N[f]\}_{N\ge1}$, for $f\in\mathcal C^\infty_c(\mathbf R^{d+1})$. 

\begin{lemma}\label{lem-reg-M^N}
Assume {\rm \textbf{A}}.
Then, for all $T>0$, there exists $C>0$ such that for all $N\ge1$, $\delta>0$, $0\le r<t\le T$ such that $t-r\le \delta$ and $f\in\mathcal C^\infty_c(\mathbf R^{d+1})$,  it holds
$$\mathbf E\big [|\sqrt N\mathbf M_t^N[f]-\sqrt N\mathbf M_r^N[f]|\big ]\le C\sqrt{N\delta+1}\frac{\|f\|_{\mathcal C^{1,\mathfrak j_0}}}{\sqrt N}. $$ 
\end{lemma}

\begin{proof}
 From the proof of~Lemma 21 in \cite{colt}, it holds 
$$\mathbf E\big [| \mathbf M_t^N[f]-\mathbf M_r^N[f]|^2\big ]\le C {(N\delta+1)}\frac{\|f\|^2_{\mathcal C^{1,\mathfrak j_0}}}{N^2}.$$ 
This leads the desired result. 
\end{proof}

\begin{proposition}\label{p-rc-m^N}
Assume {\rm \textbf{A}}.
Then, the sequence $\{t\in\mathbf R_+\mapsto\sqrt N\mathbf M_t^N\}_{N\ge1}$ is relatively compact in $\mathcal D(\mathbf R_+,\mathcal H^{-\mathfrak J_3,\mathfrak j_3}(\mathbf R^{d+1})$. 
\end{proposition} 
 \begin{proof} 
 Recall  that $\mathcal H^{\mathfrak J_3,\mathfrak j_3}(\mathbf R^{d+1})\hookrightarrow_{\mathrm{H.S.}}\mathcal H^{\mathfrak J_1,\mathfrak j_1}(\mathbf R^{d+1})$. The same arguments as those used to prove Proposition \ref{p-rc-eta} together with   Lemmata \ref{lem-ccM^N}  and \ref{lem-reg-M^N}  imply the result.
\end{proof}

 We now turn to the regularity of the limit points of the sequence $(\eta^N)_{N\ge1}$.


\begin{lemma}\label{lem-reg-LP}
\begin{sloppypar}
Assume {\rm \textbf{A}}.
Then, for all $T>0$,  
 \begin{equation}\label{eq1-lem-reg}
\lim_{N\to\infty}\mathbf E\Big[\sup_{t\in[0,T]}\|\eta_t^N-\eta_{t^-}^N\|_{\mathcal H^{-\mathfrak J_3+1 ,\mathfrak j_3}}^2\Big] +   \mathbf E\Big[\sup_{t\in[0,T]}\|\sqrt N \mathbf M_t^N -\sqrt N \mathbf M_{t^-}^N\|_{\mathcal H^{-\mathfrak J_3 ,\mathfrak j_3}}^2\Big]=0. 
\end{equation}
Any limit point of $(\eta^N)_{N\ge1}$ (resp. of $(\sqrt N\mathbf M^N)_{N\ge1}$) in $\mathcal D(\mathbf R_+,\mathcal H^{-\mathfrak J_3+1,\mathfrak j_3}(\mathbf R^{d+1}))$ (resp.  in $\mathcal D(\mathbf R_+,\mathcal H^{-\mathfrak J_3,\mathfrak j_3}(\mathbf R^{d+1}))$) belongs a.s.   to $\mathcal C(\mathbf R_+,\mathcal H^{-\mathfrak J_3+1 ,\mathfrak j_3}(\mathbf R^{d+1}))$ (resp. to $\mathcal C(\mathbf R_+,\mathcal H^{-\mathfrak J_3,\mathfrak j_3}(\mathbf R^{d+1}))$). \end{sloppypar}
\end{lemma}


\begin{proof} 
 Let $T>0$.  Let us first consider the sequence $(\eta^N)_{N\ge1}$.  
In what follows, $C>0$ is a constant independent of $N\ge 1$, $k\in\{1,\dots,\lfloor NT\rfloor\}$,  and  $f\in \mathcal H^{\mathfrak J_3-1,\mathfrak j_3}(\mathbf R^{d+1})$.  We have 
\begin{equation}
\sup_{t\in[0,T]}\|\eta_t^N-\eta_{t^-}^N\|_{\mathcal H^{-\mathfrak J_3+1,\mathfrak j_3}}^2\le 2\sup_{t\in[0,T]}\|\Upsilon_t^N-\Upsilon_{t^-}^N\|_{\mathcal H^{-\mathfrak J_3+1,\mathfrak j_3}}^2+2\sup_{t\in[0,T]}\|\Theta_t^N-\Theta_{t^-}^N\|_{\mathcal H^{-\mathfrak J_3+1,\mathfrak j_3}}^2.
\end{equation}
According to Lemma \ref{Theta^N-continu}, one has, for all $t\in\mathbf R_+$ and $N\ge1$, $\|\Theta_t^N-\Theta_{t^-}^N\|_{\mathcal H^{-\mathfrak J_3+1,\mathfrak j_3}}=0$.
In addition, since a.s. $\bar\mu^N\in\mathcal C(\mathbf R_+,\mathcal H^{-\mathfrak J_0,\mathfrak j_0}(\mathbf R^{d+1}))$, it follows, by definition of $\Upsilon^N$, that a.s. for all $N\ge 1$,  
\begin{equation}
\sup_{t\in[0,T]}\langle f,\Upsilon_t^N-\Upsilon_{t^-}^N\rangle^2=N\sup_{t\in[0,T]}\langle f,\mu_t^N-\mu_{t^-}^N\rangle^2. 
\end{equation}
 The function $t\in[0,T]\mapsto\langle f,\mu_t^N\rangle$ has exactly $\lfloor NT\rfloor$ discontinuities located at times $t_k= k/N$ ($k\in \{1,\ldots, \lfloor NT\rfloor\}$).
In addition, from \eqref{def<f,Upsilon>}, for  $k\in\{1,\dots,\lfloor NT\rfloor\}$, its  $k$-th discontinuity   is bounded by
\begin{align*}
\boldsymbol \delta_k^N[f]&:= |\mathbf M_{k-1}^N[f]| +|\mathbf R_{k-1}^N[f]|\\
&\quad + \kappa\Big|\int_{\frac{k-1}{N}}^{\frac{k}{N}}\int_{\mathsf X\times\mathsf Y}\langle\phi(\cdot,\cdot,x)-y,\mu_s^N\otimes\gamma\rangle\langle\nabla_\theta f\cdot\nabla_\theta\phi(\cdot,\cdot,x),\mu_s^N\otimes\gamma\rangle\pi(\di x,\di y)\di s\Big|  \\
&\quad + \frac{\kappa}{N}\Big|\int_{\frac{k-1}{N}}^{\frac{k}{N}}\int_{\mathsf X\times\mathsf Y}\Big\langle\langle\phi(\cdot,\cdot,x)-y,\gamma\rangle\langle\nabla_\theta f\cdot\nabla_\theta\phi(\cdot,\cdot,x),\gamma\rangle,\mu_s^N\Big\rangle\pi(\di x,\di y)\di s\Big|   \\
&\quad + \frac{\kappa}{N}\Big|\int_{\frac{k-1}{N}}^{\frac{k}{N}}\int_{\mathsf X\times\mathsf Y}\Big\langle(\phi(\cdot,\cdot,x)-y)\nabla_\theta f\cdot\nabla_\theta\phi(\cdot,\cdot,x),\mu_s^N\otimes\gamma\Big\rangle\pi(\di x,\di y)\di s\Big| \\
&\quad + \kappa\Big|\int_{\frac{k-1}{N}}^{\frac{k}{N}}\langle\nabla_\theta f\cdot\nabla_\theta\mathscr D_{\mathrm{KL}}(q^1_\cdot|P_0^1),\mu_s^N\rangle\di s\Big|.
\end{align*}
Thus,
\begin{equation}
\sup_{t\in[0,T]}\langle f,\mu_t^N-\mu_{t^-}^N\rangle^2\le \max\{|\boldsymbol \delta_{k+1}^N[f]|^2,\ 0\le k<\lfloor NT\rfloor\}. 
\end{equation}
Using the bounds provided by the proof of Lemma 19 in \cite{colt}, we obtain, for $0\le k<\lfloor NT\rfloor$, 
\begin{equation}\label{bound-E[M_k^4]}
\mathbf E[|\mathbf M_k^N[f]|^4]\le C\frac{\|f\|^4_{\mathcal C^{1,\mathfrak j_0}}}{N^4}\le C\frac{\|f\|^4_{\mathcal H^{\mathfrak J_0,\mathfrak j_0}}}{N^4}, \ \ \ 
\mathbf E[|\mathbf R_k^N[f]|^4]\le C\frac{\|f\|^4_{\mathcal H^{\mathfrak J_0,\mathfrak j_0}}}{N^8},
\end{equation}
and 
\begin{align*}
&\mathbf E\Big[\Big|\int_{\frac{k }{N}}^{\frac{k+1}{N}}\int_{\mathsf X\times\mathsf Y}\langle\phi(\cdot,\cdot,x)-y,\mu_s^N\otimes\gamma\rangle\langle\nabla_\theta f\cdot\nabla_\theta\phi(\cdot,\cdot,x),\mu_s^N\otimes\gamma\rangle\pi(\di x,\di y)\di s\Big|^4\\
&+ \frac{1}{N}\Big| \int_{\frac{k }{N}}^{\frac{k+1}{N}}\int_{\mathsf X\times\mathsf Y}\Big\langle\langle\phi(\cdot,\cdot,x)-y,\gamma\rangle\langle\nabla_\theta f\cdot\nabla_\theta\phi(\cdot,\cdot,x),\gamma\rangle,\mu_s^N\Big\rangle\pi(\di x,\di y)\di s\Big|^4   \\
&\quad + \frac{1}{N}\Big|\int_{\frac{k }{N}}^{\frac{k+1}{N}}\int_{\mathsf X\times\mathsf Y}\Big\langle(\phi(\cdot,\cdot,x)-y)\nabla_\theta f\cdot\nabla_\theta\phi(\cdot,\cdot,x),\mu_s^N\otimes\gamma\Big\rangle\pi(\di x,\di y)\di s\Big|^4\Big] \le C\frac{\|f\|^4_{\mathcal H^{\mathfrak J_0,\mathfrak j_0}}}{ N^4}.
\end{align*}
 In addition, one also has (see Equation (57) in \cite{colt}): 
\begin{align*}
\mathbf E\Big[\Big|\int_{\frac{k-1}{N}}^{\frac{k}{N}}\langle\nabla_\theta f\cdot\nabla_\theta\mathscr D_{\mathrm{KL}}(q^1_\cdot|P_0^1),\mu_s^N\rangle\di s\Big|^4\Big]\le C\frac{\|f\|_{\mathcal H^{\mathfrak J_0,\mathfrak j_0}}^4}{ N^4}.
\end{align*}
Consequently, it holds:
\begin{align*}
\mathbf E[ \max\{|\boldsymbol \delta_{k+1}^N[f]^2,\ 0\le k<\lfloor NT\rfloor\}] \le \Big|\sum_{k=0}^{\lfloor NT\rfloor-1}\mathbf E[\boldsymbol \delta_k^N[f]^4]\Big|^{1/2}\le C\frac{\|f\|_{\mathcal H^{\mathfrak J_0,\mathfrak j_0}}^2}{N^{3/2}}.
\end{align*}
Hence
\begin{equation}
\mathbf E\Big[N\sup_{t\in[0,T]}\langle f,\mu_t^N-\mu_{t^-}^N\rangle^2\Big]\le C\frac{\|f\|_{\mathcal H^{\mathfrak J_0,\mathfrak j_0}}^2}{\sqrt N}. 
\end{equation}
 Since  $\mathcal H^{\mathfrak J_3-1,\mathfrak j_3}(\mathbf R^{d+1})\hookrightarrow_{\mathrm{H.S.}}\mathcal H^{\mathfrak J_0,\mathfrak j_0}(\mathbf R^{d+1})$, one deduces that $ \mathbf E [\sup_{t\in[0,T]}\|\eta_t^N-\eta_{t^-}^N\|_{\mathcal H^{-\mathfrak J_3+1 ,\mathfrak j_3}}^2 ]\to 0$ as $N\to +\infty$. The fact that any limit points of $(\eta^N)_{N\ge1}$ is a.s. continuous follows from Condition 3.28 in Proposition 3.26 of \cite{jacod2003skorokhod}.

 The case of the sequence $(\sqrt N\mathbf M^N)_{N\ge1}$ is treated very similarly. The proof of the lemma is complete.
\end{proof}

\subsection{Convergence of $(\sqrt N\mathbf M^N)_{N\ge1}$ to a $\mathfrak G$-process}
\label{sec-conv-G-process}

 In this section, we prove that the sequence $(\sqrt N\mathbf M^N)_{N\ge1}$ converges towards   a $\mathfrak G$-process (see Definition \ref{d-G-process}), see   Proposition \ref{p-conv-to-g-process}. The case when the $\{\theta^i_k, i\in \{1,\ldots,N\}\}$'s are  generated   by the algorithm 
\eqref{eq.algo-batch} requires extra analysis compared to the cases when the $\{\theta^i_k, i\in \{1,\ldots,N\}\}$'s are  generated   by the algorithms  \eqref{eq.algo-ideal} or  \eqref{eq.algo-z1z2} (see indeed the second part of the proof of  Proposition \ref{p-conv-to-g-process} and Lemma \ref{prop-conv_rhoN'} below).

\begin{proposition}\label{prop-conv_G-proc-f}
 Assume that the $\{\theta^i_k, i\in \{1,\ldots,N\}\}$'s are  generated  either  by the algorithm \eqref{eq.algo-ideal} or  by the algorithm \eqref{eq.algo-batch}. Then, for every $f\in \mathcal C^{1,\mathfrak j_0}(\mathbf R^{d+1})$, the sequence $\{t\in\mathbf R_+\mapsto\sqrt N\mathbf M_t^N[f]\}_{N\ge1}$ converges in distribution in $\mathcal D(\mathbf R_+,\mathbf R)$ towards a process $\mathcal X^f\in\mathcal C(\mathbf R_+,\mathbf R)$ that has independent Gaussian increments. Moreover, for all $t\in\mathbf R_+$, 
$$\mathbf E[\mathcal X_t^f]=0 \text{ and } \Var(\mathcal X_t^f)=\kappa^2\int_0^t\Var_\pi(\mathscr Q[f](x,y,\bar\mu_s))\di s,  $$
where we recall $\mathscr Q[f](x,y,\bar\mu_v)=\langle\phi(\cdot,\cdot,x)-y,\bar\mu_v\otimes\gamma\rangle\langle\nabla_\theta f\cdot\nabla_\theta\phi(\cdot,\cdot,x),\bar\mu_v\otimes\gamma\rangle$  (see Theorem \ref{thm-clt-ideal}).
\end{proposition}

\begin{proof}
We treat separately the two cases when the $\{\theta^i_k, i\in \{1,\ldots,N\}\}$'s are  generated   by the algorithm \eqref{eq.algo-ideal} or by the algorithm~\eqref{eq.algo-batch}. Let~$f\in \mathcal C^{1,\mathfrak j_0}(\mathbf R^{d+1})$.

\paragraph{The case of the Idealized algorithm \eqref{eq.algo-ideal}.}  Let us assume that the $\{\theta^i_k, i\in \{1,\ldots,N\}\}$'s are  generated   by the algorithm \eqref{eq.algo-ideal}.
To prove the desired result, we apply the martingale central limit theorem 5.1.4 in \cite{ethier2009markov} to the sequence $\{t\in\mathbf R_+\mapsto\sqrt N\mathbf M_t^N[f]\}_{N\ge1}$.  Let us first show that Condition (a) in Th. 7.1.4 in  \cite{ethier2009markov} holds. First of all, by Remark 7.1.5 in \cite{ethier2009markov}, the covariation matrix of $\sqrt N\mathbf M_t^N[f]$  is 
\begin{equation}\label{def_a}
 \mathfrak a_t^N[f]=N\sum_{k=0}^{\lfloor Nt\rfloor-1}\mathbf M_k^N[f]^2
\end{equation} 
In particular,  $\mathfrak a_t^N[f]-\mathfrak a_s[f]\ge 0$ when $t\ge s$.  On the other hand, by \eqref{bound-E[M_k^4]} (which, we recall, also holds when the $\{\theta^i_k, i\in \{1,\ldots,N\}\}$'s are  generated   by the algorithm \eqref{eq.algo-ideal}), we have 
 for all  $T\ge 0$:
\begin{equation}
\lim_{N\to+\infty}\mathbf E\Big[\sup_{t\in[0,T]}|\sqrt N\mathbf M_t^N[f]-\sqrt N\mathbf M_{t^-}^N[f]|\Big]=0.
\end{equation}
Thus Condition (a) in Th. 7.1.4 in \cite{ethier2009markov} is satisfied. Let us prove the last required condition in Theorem 7.1.4 of \cite{ethier2009markov}, namely that for all $t\in\mathbf R_+$, $\lim_N\mathfrak a_t^N[f]=\mathfrak c_t[f]$ in $\mathbf P$-probability, where $\mathfrak c$ satisfies the assumptions of Th. 7.1.1 in  \cite{ethier2009markov} (i.e., $t\in\mathbf R_+\mapsto\mathfrak c_t[f]$ is continuous, $\mathfrak c_0[f]=0$, and $\mathfrak c_t[f]-\mathfrak c_s[f]\ge 0$ if $t\ge s$). 
 Let us consider and fix $t\ge 0$. 
We recall that  when the $\{\theta^i_k, i\in \{1,\ldots,N\}\}$'s are  generated   by the algorithm \eqref{eq.algo-ideal}, one has  that for $k\ge0$ (see Equation (21) in \cite{colt}), 
\begin{align*}
 \mathbf D_{k}^N[f]&=-\frac{\kappa}{N^3}\sum_{i=1}^N\sum_{j=1,j\neq i}^N\int_{\mathsf X\times\mathsf Y}\big \langle\phi(\theta_k^j,\cdot,x)-y,\gamma\big\rangle\big \langle\nabla_\theta f(\theta_k^i)\cdot\nabla_\theta\phi(\theta_k^i,\cdot,x),\gamma\big \rangle\pi(\di x,\di y)\nonumber\\
&\quad-\frac{\kappa}{N^2}\int_{\mathsf X\times\mathsf Y}\big \langle(\phi(\cdot,\cdot,x)-y)\nabla_\theta f\cdot\nabla_\theta\phi(\cdot,\cdot,x),\nu_k^N\otimes\gamma\big \rangle\pi(\di x,\di y)\\
&=-\frac{\kappa}{N^3}\sum_{i=1}^N\sum_{j=1}^N\int_{\mathsf X\times\mathsf Y}\big \langle\phi(\theta_k^j,\cdot,x)-y,\gamma\big \rangle\big \langle\nabla_\theta f(\theta_k^i)\cdot\nabla_\theta\phi(\theta_k^i,\cdot,x),\gamma\big \rangle\pi(\di x,\di y)\nonumber\\
&\quad +\frac{\kappa}{N^3}\sum_{i=1}^N\int_{\mathsf X\times\mathsf Y}\big \langle\phi(\theta_k^i,\cdot,x)-y,\gamma\big \rangle\big \langle\nabla_\theta f(\theta_k^i)\cdot\nabla_\theta\phi(\theta_k^i,\cdot,x),\gamma\big \rangle\pi(\di x,\di y)\nonumber\\
&\quad-\frac{\kappa}{N^2}\int_{\mathsf X\times\mathsf Y}\big \langle(\phi(\cdot,\cdot,x)-y)\nabla_\theta f\cdot\nabla_\theta\phi(\cdot,\cdot,x),\nu_k^N\otimes\gamma\big \rangle\pi(\di x,\di y),
\end{align*}
and 
\begin{align*}
 \mathbf  M_{k}^{N}[f]&:=-\frac{\kappa}{N^3}\sum_{i=1}^N\sum_{j=1,j\neq i}^N(\langle\phi(\theta_k^j,\cdot,x_k),\gamma\rangle-y_k)\langle\nabla_ \theta f(\theta_k^i)\cdot\nabla_\theta\phi(\theta_k^i,\cdot,x_k),\gamma\rangle\nonumber\\
&\quad-\frac{\kappa}{N^2}\langle(\phi(\cdot,\cdot,x_k)-y_k)\nabla_\theta f\cdot\nabla_\theta\phi(\cdot,\cdot,x_k),\nu_k^N\otimes\gamma\rangle-\mathbf D_{k}^N[f].
\end{align*}
Let us introduce, for any $\nu\in\mathcal H^{\mathfrak J_0,\mathfrak j_0}(\mathbf R^{d+1})$, 
\begin{align*}
 \mathfrak Q[f](\nu)=\int_{\mathsf X\times\mathsf Y}\big (\big \langle\phi(\cdot,\cdot,x),\nu\otimes\gamma\big \rangle-y\big )\big \langle\nabla_\theta f\cdot\nabla_\theta\phi(\cdot,\cdot,x),\nu\otimes\gamma\big \rangle\pi(\di x,\di y).  
\end{align*}
Let us also define for $k\ge0$ and $N\ge1$,  
\begin{align*}
\mathfrak R_k^N[f]&:= \frac{\kappa}{N^3}\sum_{i=1}^N(\langle\phi(\theta_k^i,\cdot,x_k),\gamma\rangle-y_k)\langle\nabla_\theta f(\theta_k^i)\cdot\nabla_\theta\phi(\theta_k^i,\cdot,x_k),\gamma\rangle\\
&\quad-\frac{\kappa}{N^2}\langle(\phi(\cdot,\cdot,x_k)-y_k)\nabla_\theta f\cdot\nabla_\theta\phi(\cdot,\cdot,x_k),\nu_k^N\otimes\gamma\rangle\\
&\quad- \frac{\kappa}{N^3}\sum_{i=1}^N\int_{\mathsf X\times\mathsf Y}\big (\big \langle\phi(\theta_k^i,\cdot,x),\gamma\big \rangle-y\big )\big \langle\nabla_\theta f(\theta_k^i)\cdot\nabla_\theta\phi(\theta_k^i,\cdot,x),\gamma\big \rangle\pi(\di x,\di y)\nonumber\\
&\quad+\frac{\kappa}{N^2}\int_{\mathsf X\times\mathsf Y}\big \langle(\phi(\cdot,\cdot,x)-y)\nabla_\theta f\cdot\nabla_\theta\phi(\cdot,\cdot,x),\nu_k^N\otimes\gamma\big \rangle\pi(\di x,\di y).
\end{align*}
It then holds for all $k\ge0$ and $N\ge1$:
\begin{align}\label{M_k_simple}
\mathbf M_k^N[f]&=-\frac\kappa N\mathscr Q[f](x_k,y_k,\nu_k^N)+\frac\kappa N \mathfrak Q[f](\nu_k^N)+\mathfrak R_k^N[f].
\end{align}
Hence, by \eqref{def_a}  and \eqref{M_k_simple}, for all $t\in\mathbf R_+$, 
\begin{align}\label{a_t_calcule}
\mathfrak a_t^N[f]&=\frac{\kappa^2}{N}\sum_{k=0}^{\lfloor Nt\rfloor-1}\big[\mathscr Q[f](x_k,y_k,\nu_k^N)- \mathfrak Q[f](\nu_k^N)\big]^2 + 2\kappa\sum_{k=0}^{\lfloor Nt\rfloor-1}\mathfrak R_k^N[f]\big [ \mathfrak Q[f](\nu_k^N)-\mathscr Q[f](x_k,y_k,\nu_k^N)\big ]\nonumber\\
&\quad + N\sum_{k=0}^{\lfloor Nt\rfloor-1}\mathfrak{R}_k^N[f]^2.
\end{align}
Fix $t\ge 0$. Recall that we want to identify the limit  of $(\mathfrak a_t^N[f])_{N\ge1}\in  \mathbf R^{\mathbf N^*}$ in $\mathbf P$-probability. 
Using the  following two upper bounds (which can be easily derived using \textbf A and Lemma \ref{le.Bounds})
$$\mathbf E\big [|\mathfrak R_k^N[f]|^2\big ]\le C\|f\|_{\mathcal C^{1,\mathfrak j_0}}^2/N^4\text{ and } \mathbf E\big [| \mathfrak Q[f](\nu_k^N)|^2\big ]+\mathbf E\big [|\mathscr Q[f](x_k,y_k,\nu_k^N)|^2\big ]\le C  \|f\|_{\mathcal C^{1,\mathfrak j_0}}^2,$$
 one deduces that the two last terms of \eqref{a_t_calcule} converge to zero in $L^1$. 
Therefore, one just needs to determine the limit in $\mathbf P$-probability of  
\begin{align}
\nonumber
\frac{\kappa^2}{N}\sum_{k=0}^{\lfloor Nt\rfloor-1}[\mathscr Q[f](x_k,y_k,\nu_k^N)- \mathfrak Q[f](\nu_k^N)]^2&=\frac{\kappa^2}{N}\sum_{k=0}^{\lfloor Nt\rfloor-1}\Var_\pi(\mathscr Q[f](x,y,\nu_k^N))\\
\nonumber
&\quad +\frac{\kappa^2}{N}\sum_{k=0}^{\lfloor Nt\rfloor-1}(\mathscr Q[f](x_k,y_k,\nu_k^N)- \mathfrak Q[f](\nu_k^N))^2\\
\label{eq.Res}
&\quad - \frac{\kappa^2}{N}\sum_{k=0}^{\lfloor Nt\rfloor-1}\Var_\pi(\mathscr Q[f](x,y,\nu_k^N)).
\end{align}
On the one hand,  using Theorem \ref{th.LLN} together with the continuous mapping theorem and the dominated convergence theorem, one deduces that for all $t\ge 0$\footnote{This is indeed the same proof as the one made just after  Eq. (3.63) in \cite{jmlr}, changing $\sigma$ there by  $\mathfrak H=\langle\nabla _\theta \phi, \gamma\rangle$.}:
\begin{align*}
\frac{\kappa^2}{N}\sum_{k=0}^{\lfloor Nt\rfloor-1}\Var_\pi(\mathscr Q[f](x,y,\nu_k^N))&= \kappa^2\sum_{k=0}^{\lfloor Nt\rfloor-1}\int_{\frac kN}^{\frac{k+1}{N}}\Var_\pi(\mathscr Q[f](x,y,\mu_s^N))\di s\\
&=\kappa^2\int_0^t \Var_\pi(\mathscr Q[f](x,y,\mu_s^N))\di s-\kappa^2 \int_{\frac{\lfloor Nt\rfloor}{N}}^t\Var_\pi(\mathscr Q[f](x,y,\mu_s^N))\di s\\
&\xrightarrow[N\to +\infty]{\mathbf P}\kappa^2\int_0^t \Var_\pi(\mathscr Q[f](x,y,\bar\mu_s))\di s.
\end{align*}
 Let us now deal with the two remainders terms in \eqref{eq.Res}. Denoting by  $\mathscr L_k^N= [\mathscr Q[f](x_k,y_k,\nu_k^N)- \mathfrak Q[f](\nu_k^N)]^2$, we notice that  $\Var_\pi(\mathscr Q[f](x,y,\nu_k^N))= \mathbf E_{(x,y)\sim \pi} [  \mathscr L_k^N]$. Moreover if $j<k$,  since $\mathscr L_j^N$ is $\mathcal F_k^N$-measurable (see \eqref{eq.Fk1}) as well as $\nu_k^N$, and $(x_k,y_k)\indep \mathcal F_k^N$, one has:
\begin{align*}
\mathbf E\Big[\Big(\mathscr L_k^N-\mathbf E_\pi [  \mathscr L_k^N] \Big)  \Big(\mathscr L_j^N-\mathbf E_\pi [\mathscr   L_j^N] \Big)\Big]&=\mathbf E\Big[\Big(\mathscr L_j^N-\mathbf E_\pi [  \mathscr L_j^N] \Big)  \mathbf E\Big[\Big(\mathscr L_k^N-\mathbf E_\pi [  \mathscr L_k^N] \Big)| \mathcal F_k^N\Big]\Big]\\
&=\mathbf E\Big[\Big(\mathscr L_j^N-\mathbf E_\pi [  \mathscr L_j^N] \Big)  \mathbf E_\pi\Big[\Big(\mathscr L_k^N-\mathbf E_\pi [  \mathscr L_k^N] \Big)\Big]\\
&=\mathbf E\Big[\Big(\mathscr L_j^N-\mathbf E_\pi [  \mathscr L_j^N] \Big) \times 0\Big]=0.
\end{align*}
Thus,   it holds:
\begin{align*}
&\mathbf E\Big[\Big|\frac{\kappa^2}{N}\sum_{k=0}^{\lfloor Nt\rfloor-1}[\mathscr Q[f](x_k,y_k,\nu_k^N)- \mathfrak Q[f](\nu_k^N)]^2 - \frac{\kappa^2}{N}\sum_{k=0}^{\lfloor Nt\rfloor-1}\Var_\pi(\mathscr Q[f](x,y,\nu_k^N))\Big|^2\Big]\\
&=\frac{\kappa^4}{N^2}\sum_{k=0}^{\lfloor Nt\rfloor-1}\mathbf E\Big[\Big|[\mathscr Q[f](x_k,y_k,\nu_k^N)- \mathfrak Q[f](\nu_k^N)]^2-\Var_\pi(\mathscr Q[f](x,y,\nu_k^N))\Big|^2\Big]\\
&\leq \frac{C}{N^2}\sum_{k=0}^{\lfloor Nt\rfloor-1}\mathbf E[|\mathscr Q[f](x_k,y_k,\nu_k^N)|^4]\leq \frac CN   \|f\|_{\mathcal C^{1,\mathfrak j_0}}^4\to  0.
\end{align*}
We have thus shown that  
 for all $t\ge 0$, $\mathfrak a_t^N[f]\to \kappa^2\int_0^t \Var_\pi(\mathscr Q[f](x,y,\bar\mu_s))\di s$ in $\mathbf P$-probability and as $N\to +\infty$. 
Therefore,  for $t\ge 0$,  $\mathfrak c_t[f]= \kappa^2\int_0^t \Var_\pi(\mathscr Q[f](x,y,\bar\mu_s))\di s$. This ends the proof of  the proposition when the the $\{\theta^i_k, i\in \{1,\ldots,N\}\}$'s are  generated   by the algorithm \eqref{eq.algo-ideal}.

\paragraph{The case of the BbB algorithm \eqref{eq.algo-batch}. }  Let us assume that the $\{\theta^i_k, i\in \{1,\ldots,N\}\}$'s are  generated   by the algorithm \eqref{eq.algo-batch}.
We will also apply the central limit theorem 7.1.4 in \cite{ethier2009markov} to the sequence $\{t\in\mathbf R_+\mapsto\sqrt N\mathbf M_t^N[f]\}_{N\ge1}$.
Again, we define, as in \eqref{def_a},  
$$
\mathfrak a_t^N[f]=N\sum_{k=0}^{\lfloor Nt\rfloor-1}\mathbf M_k^N[f]^2.$$
 Condition (a) in Th. 7.1.4 in \cite{ethier2009markov} is satisfied and we will now prove  the last required condition in Th. 7.1.4 in \cite{ethier2009markov}.   Let us introduce the following  random probability measures over $\mathbf R^{d+1}\times\mathbf R^d$:
\begin{equation}\label{new_empir_distrib}
\mathbf r_k^N=\frac 1N\sum_{i=1}^N\delta_{(\theta_k^i,\mathsf Z_k^i)} \text{ and } \rho_t^N =\mathbf  r_{\lfloor Nt\rfloor}^N, \ k\ge0, \ t\ge 0.
\end{equation}
We also set, for $(x,y)\in\mathsf X\times\mathsf Y$ and $\rho\in\mathcal P(\mathbf R^{d+1}\times\mathbf R^d)$, 
$$\mathfrak Q[f](x,y,\rho)=\langle\phi(\cdot,\cdot,x)-y,\rho\rangle\langle\nabla_\theta f( \pi_{\mathbf R^{d+1}}(\cdot))\cdot\nabla_\theta \phi(\cdot,\cdot,x),\rho\rangle,$$
where,  for $(\theta,\mathsf Z)\in\mathbf R^{d+1}\times \mathbf R^d$,  $\pi_{\mathbf R^{d+1}}$ is the projection onto $\mathbf R^{d+1}$: $ \pi_{\mathbf R^{d+1}}(\theta,\mathsf Z)=\theta\in \mathbf R^{d+1}$. 
By Item \textbf 2 in the proof of Lemma \ref{lem-E[eta]<infty}, one has for $k\ge 0$, 
\begin{align*}
\mathbf M_{k}^N[f]
&=-\frac{\kappa}{N}\langle\phi(\cdot,\cdot,x_k)-y_k,\mathbf r_k^N\rangle\langle\nabla_\theta f(\pi_{\mathbf R^{d+1}}(\cdot))\cdot\nabla_\theta \phi(\cdot,\cdot,x_k),\mathbf r_k^N\rangle -\mathbf D_{k}^N[f]\\
&= -\frac{\kappa}{N}\mathfrak Q[f](x_k,y_k,\mathbf r_k^N)-\mathbf D_{k}^N[f] = \mathbf F^N(x_k,y_k,\mathbf r_k^N)-\mathbf D_{k}^N[f]
\end{align*}
where
 $$\mathbf F^N(x_k,y_k,\mathbf r_k^N)=-\frac{\kappa}{N}\mathfrak Q[f](x_k,y_k,\mathbf r_k^N).$$
 Fix $t\ge0$. Let us identify the limit  in probability as $N\to +\infty$ of   the sequence $(\mathfrak a_t^N[f])_{N\ge1}\subset \mathbf R$. 
 We define at iteration $k\ge 1$ a larger $\sigma$-algebra than $\mathcal F_k^N$ (see \eqref{eq.Fk2}),  in which, contrary to $\mathcal F_k^N$,  the sequence $\{\mathsf Z^j_k, j=1,\ldots,N\}$ is considered:
\begin{equation*}
 \Sigma_k^N=\boldsymbol\sigma \Big (\theta_{0}^i ,   \mathsf Z^{j}_{q'},(x_q,y_q),  1\leq i,j\leq N,  0\le q\le k-1, 0\le q'\le k\big \} \Big ).
\end{equation*}  
We rewrite $\mathfrak a_t^N[f]$ as follows:
\begin{align}\label{eq-a_t^Nalgo6}
\mathfrak a_t^N[f]=N\sum_{k=0}^{\lfloor Nt\rfloor-1} \Big(\mathbf E\big[\mathbf M_{k}^N[f]^2 \big|\Sigma_k^N \big] + \mathbf M_{k}^N[f]^2  - \mathbf E\big[\mathbf M_{k}^N[f]^2 \big|\Sigma_k^N \big]\Big).
\end{align}
By \eqref{bound-E[M_k^4]}, it holds: 
\begin{align*}
\mathbf E\Big[\Big(N\sum_{k=0}^{\lfloor Nt\rfloor-1}\mathbf M_k^N[f]^2-\mathbf E[\mathbf M_k^N[f]^2|\Sigma_k^N]\Big)^2\Big]
&=N^2\sum_{k=0}^{\lfloor Nt\rfloor-1}\mathbf E\Big[\Big(\mathbf M_k^N[f]^2-\mathbf E[\mathbf M_k^N[f]^2|\Sigma_k^N]\Big)^2\Big]\\
&\leq CN^2\sum_{k=0}^{\lfloor Nt\rfloor-1}\mathbf E[\mathbf M_k^N[f]^4]\leq CN^2\|f\|_{\mathcal C^{1,\mathfrak j_0}}^4/N^3\to 0.
\end{align*}
Hence, the two last terms of \eqref{eq-a_t^Nalgo6} converge to zero in $L^2$, i.e.:
\begin{equation}
N\sum_{k=0}^{\lfloor Nt\rfloor-1}\mathbf M_{k}^N[f]^2  - \mathbf E\big[\mathbf M_{k}^N[f]^2 \big|\Sigma_k^N \big]\xrightarrow[N\to\infty]{L^2}0.
\end{equation}
Therefore, the limit in $\mathbf P$-probability $\mathfrak c_t[f]$ of $\mathfrak a_t^N[f]$ is given by the limit in $\mathbf P$-probability of 
$$N\sum_{k=0}^{\lfloor Nt\rfloor-1} \mathbf E\big[\mathbf M_{k}^N[f]^2 \big|\Sigma_k^N \big]=N\sum_{k=0}^{\lfloor Nt\rfloor-1}\Var_{\pi}(\mathbf F^N(x,y,\mathbf r_k^N)),$$
where the equality holds since $(x_k,y_k)\indep \Sigma_k^N$ and the $(\theta_k^j,\mathsf Z_j^k)$'s are $ \Sigma_k^N$-measurable. 
We then write:
\begin{align}\label{eq_sum_var}
\nonumber
&N\sum_{k=0}^{\lfloor Nt\rfloor-1}\Var_{\pi}(\mathbf F^N(x,y,r_k^N))\\
\nonumber
&=\frac{\kappa^2}{N}\sum_{k=0}^{\lfloor Nt\rfloor-1}\Var_\pi(\mathfrak Q[f](x,y,r_k^N))\nonumber\\
&=\kappa^2\sum_{k=0}^{\lfloor Nt\rfloor-1}\int_{\frac kN}^{\frac{k+1}{N}}\Var_{\pi}(\mathfrak Q[f](x,y,\rho_s^N))\di s\nonumber\\
&=\kappa^2\int_0^t\Var_{\pi}(\mathfrak Q[f](x,y,\rho_s^N))\di s  - \kappa^2 \int_{\frac{\lfloor Nt\rfloor}{N}}^t\Var_{\pi}(\mathfrak Q[f](x,y,\rho_s^N))\di s.
\end{align}
 For this fix time $t\ge 0$, we would like now   to pass to  the limit $N\to+\infty$ (in $\mathbf P$-probability) in  \eqref{eq_sum_var}. We recall the standard result: $(X^N)_{N\ge 1}$ converges to $X$ in $\mathbf P$-probability if for any subsequence $N'$ there exists a subsequence $N^\star$ of $N'$ such that a.s. $X^{N^\star}\to X$. We will use such a result.  
Let us thus consider  a subsequence $N'$. Let us show that there exists a subsequence $N^\star$ of $N'$  such that  a.s.
$$N^\star \sum_{k=0}^{\lfloor N^\star t\rfloor-1}\Var_{\pi}(\mathrm F^{N^\star }(x,y,\mathbf r_k^{N^\star }))\to \kappa^2\int_0^t\Var_{\pi}(\mathscr Q[f](x,y,\bar\mu_s))\di s.$$
 Since $\mathfrak q_0:= 2\max(\mathfrak j_0, \mathfrak p_0)>1+(d+1)/2$, by Theorem \ref{th.LLN},   in $\mathbf P$-probability, $\lim_{N'}\mu^{N'}=\bar\mu$ in the space $\mathcal D(\mathbf R_+,\mathcal P_{\mathfrak q_0}(\mathbf R^{d+1}))$. Hence, there exists a subsequence $N''$ of $N'$ such that $\mu^{N''}$ converges a.s. to $\bar\mu$ in $\mathcal D(\mathbf R_+,\mathcal P_{\mathfrak q_0}(\mathbf R^{d+1}))$. By Lemma \ref{prop-conv_rhoN'} below,  it holds a.s.  for all $s\ge 0$,  
 \begin{equation}\label{eq.Cvrho}
 \rho_s^{N''}\to \bar\mu_s\otimes\gamma \text{ as $N''\to +\infty$   in  $\mathcal P_{\mathfrak q_0}(\mathbf R^{d+1}\times\mathbf R^d)$}.
 \end{equation} 
We now claim that   a.s. for all $s\ge 0$
\begin{equation}\label{eq.barQ}
\Var_\pi(\mathfrak Q[f](x,y,\rho_s^{N''}))\to\Var_\pi(\mathfrak Q[f](x,y,\bar\mu_s\otimes\gamma)) \text{ as  $  N''\to +\infty$}.
\end{equation}
 Let us prove this claim. We recall that by definition:
\begin{align}\label{eq.Vpi}
\Var_\pi(\mathfrak Q[f](x,y,\rho_s^{N''}))=\mathbf E_{(x,y)\sim\pi} [ |\mathfrak Q[f](x,y,\rho_s^{N''})|^2]-\mathbf E_{(x,y)\sim\pi}[ \mathfrak Q[f](x,y,\rho_s^{N''})]^2,
\end{align}
 where 
\begin{align}\label{eq.QF}
\mathfrak Q[f](x,y,\rho_s^{N''})=\langle\phi(\cdot,\cdot,x)-y,\rho_s^{N''}\rangle\langle\nabla_\theta f( \pi_{\mathbf R^{d+1}}(\cdot))\cdot\nabla_\theta \phi(\cdot,\cdot,x),\rho_s^{N''}\rangle.
\end{align}
Since $(\theta,z)\mapsto\phi(\theta,z,x)-y$ is continuous and bounded (uniformly over $\theta,z,x,y$), it holds a.s. for all  $s\ge 0$, $x,y\in \mathsf X\times \mathsf Y$, $\langle\phi(\cdot,\cdot,x)-y,\rho_s^{N''}\rangle\to \langle\phi(\cdot,\cdot,x)-y,\bar\mu_s\otimes\gamma\rangle$ as $N''\to +\infty$. On the other hand, since the function $(\theta,z)\mapsto \langle\nabla_\theta f( \theta)\cdot\nabla_\theta \phi(\theta,z,x)$ is continuous and bounded by $C\Vert f \Vert _{\mathcal C^{1,\mathfrak j_0}}(1+|\theta|^{\mathfrak j_0}) \mathfrak b(z)$. Since $(1+|\theta|^{\mathfrak j_0}) \mathfrak b(z)$ is bounded   by the function $\mathfrak D_{\mathfrak q_0}(\theta,z)=1+|\theta|^{\mathfrak q_0} +|z|^{\mathfrak q_0}$ (recall that by \textbf{A1}, $\mathfrak b(z)=1+|z|^{\mathfrak p_0}$), one has from \eqref{eq.Cvrho},  as $N''\to +\infty$, a.s. for all $s\ge 0$, $x\in \mathsf X$,  
$$\langle\nabla_\theta f( \pi_{\mathbf R^{d+1}}(\cdot))\cdot\nabla_\theta \phi(\cdot,\cdot,x),\rho_s^{N''}\rangle\to \langle\nabla_\theta f( \pi_{\mathbf R^{d+1}}(\cdot))\cdot\nabla_\theta \phi(\cdot,\cdot,x),\bar\mu_s\otimes\gamma\rangle.$$ 
Note also that by the previous analysis, we have a.s. for all $s\ge 0$, $x,y\in \mathsf X\times \mathsf Y$,
\begin{align}\label{eq.Eq-ll}
|\mathfrak Q[f](x,y,\rho_s^{N''})|\le \sup_{x,y}|\mathfrak Q[f](x,y,\rho_s^{N''})|&\le C \Vert f \Vert _{\mathcal C^{1,\mathfrak j_0}} \langle\mathfrak D_{\mathfrak q_0},\rho_s^{N''}\rangle\\
\nonumber
&\le  C \Vert f \Vert _{\mathcal C^{1,\mathfrak j_0}}\sup_{N''\ge 1}\langle\mathfrak D_{\mathfrak q_0},\rho_s^{N''}\rangle<+\infty
\end{align}  
where the last inequality   follows e.g. from the fact that  $(\langle\mathfrak D_{\mathfrak q_0},\rho_s^{N''}\rangle)_{N''}$ is a converging sequence. Together with  the dominated convergence theorem, one deduces  \eqref{eq.barQ}.


Let us now consider the random variable $\int_0^t\Var_{\pi}(\mathfrak Q[f](x,y,\rho_s^{N''}))\di s$ appearing in the r.h.s. of~\eqref{eq_sum_var}. By \eqref{eq.Vpi}, \eqref{eq.QF}, \eqref{eq.Eq-ll}, and  \eqref{eq.De}, it holds a.s.  for all $s\ge 0$ and $x,y\in \mathsf X\times \mathsf Y$, 
\begin{align*}
\Var_{\pi}(\mathfrak Q[f](x,y,\rho_s^{N''}))&\le  C\Vert f \Vert^2 _{\mathcal C^{1,\mathfrak j_0}} |\langle\mathfrak D_{\mathfrak q_0},\rho_s^{N''} \rangle|^2\\
&\le C\Vert f \Vert^2_{\mathcal C^{1,\mathfrak j_0}}\sup_{N''\ge 1}\sup_{s\in[0,t]}|\langle\mathfrak D_{\mathfrak q_0},\rho_s^{N''} \rangle|^2<+\infty.
\end{align*}
Therefore, using also \eqref{eq.barQ} and  the dominated convergence theorem,  for this fix $t\ge 0$, one has:
$$\kappa^2\int_0^t\Var_{\pi}(\mathfrak Q[f](x,y,\rho_s^{N''}))\di s \xrightarrow[N''\to\infty]{a.s.}\kappa^2\int_0^t\Var_\pi(\mathfrak Q[f](x,y,\bar\mu_s\otimes\gamma))\di s.$$
Let us now consider  the last term in \eqref{eq_sum_var}.  
We have using \eqref{eq.De2}, 
\begin{align*}
\mathbf E\Big[\Big|\int_{\frac{\lfloor N''t\rfloor}{N''}}^t\Var_{\pi}(\mathfrak Q[f](x,y,\rho_s^{N''}))\di s\Big|\Big]&= \mathbf E\Big[\Big|\int_{0}^t\Var_{\pi}(\mathfrak Q[f](x,y,\rho_s^{N''})) \mathbf 1_{s\in  \big [\frac{\lfloor N''t\rfloor}{N''}, t\big ]}\di s\Big|\Big] \\
&\le\frac{1}{N''}\mathbf E\Big[ \sup_{s\in  \big [\frac{\lfloor N''t\rfloor}{N''}, t\big ]} \Var_{\pi}(\mathfrak Q[f](x,y,\rho_s^{N''}))  \Big] \\
&\le\frac{C\|f\|^2_{\mathcal C^{1,\mathfrak j_0}}}{N''} \mathbf E\Big[ \sup_{s\in  \big [\frac{\lfloor N''t\rfloor}{N''}, t\big ]} |\langle\mathfrak D_{\mathfrak q_0},\rho_s^{N''} \rangle|^2   \Big] \\
&\le \frac{C\|f\|^2_{\mathcal C^{1,\mathfrak j_0}}}{N''}\xrightarrow[N''\to\infty]{}0.
\end{align*}
 Therefore, there exists $N^\star \subset N''$ such that  
$$\int_{\frac{\lfloor N^\star t\rfloor}{N^\star}}^t\Var_{\pi}(\mathfrak Q[f](x,y,\rho_s^{N^\star}))\di s\xrightarrow[N^\star\to\infty]{a.s.}0.$$
  Thus, we have found a subsequence $N^\star \subset N'$  such that a.s. 
 $$N^\star\sum_{k=0}^{\lfloor N^\star t\rfloor-1}\Var_{\pi}(\mathbf F^{N^\star}(x,y,\mathbf r_k^{N^\star}))\xrightarrow[N^\star\to\infty]{a.s.}\mathfrak c_t[f]:=\kappa^2\int_0^t\Var_\pi(\mathfrak Q[f](x,y,\bar\mu_s\otimes\gamma))\di s.$$
 Consequently 
 $$N \sum_{k=0}^{\lfloor N  t\rfloor-1}\Var_{\pi}(\mathbf F^{N }(x,y,\mathbf r_k^{N }))\xrightarrow[N \to\infty]{\mathbf P}\mathfrak c_t[f].$$
 This is the desired result since $\mathfrak Q[f] (x,y,\bar\mu_s\otimes\gamma)=\mathscr Q[f](x,y,\bar\mu_s)$. The proof of the proposition is complete. 
\end{proof}

\begin{lemma}\label{prop-conv_rhoN'}
 Assume that the $\{\theta^i_k, i\in \{1,\ldots,N\}\}$'s are  generated  by the algorithm   \eqref{eq.algo-batch}. Assume also {\rm \textbf{A}} and let $\mathfrak q_0\in 2\mathbf N$ such that $\mathfrak q_0>1+(d+1)/2$. 
 Assume that along some subsequence $\mathfrak N$, $(\mu^{\mathfrak N})_{\mathfrak N}$ converges a.s. to $\bar\mu$ in $\mathcal D(\mathbf R_+,\mathcal P_{\mathfrak q_0}(\mathbf R^{d+1}))$. Then,   it holds  a.s.  for  all $s\ge 0$:
\begin{equation*}
 \lim_{\mathfrak N\to +\infty}\rho_s^{\mathfrak N}=\bar\mu_s\otimes\gamma \ \text{   in}\ \mathcal P_{\mathfrak q_0}(\mathbf R^{d+1}\times\mathbf R^d).
\end{equation*}
\end{lemma}

\begin{proof} 
In the following, we simply denote $\mathfrak N$ by $N$.
Assume that $\mu^N\xrightarrow{a.s.}\bar\mu \ \text{in}\ \mathcal D(\mathbf R_+,\mathcal P_{\mathfrak q_0}(\mathbf R^{d+1}))$.  Recall that  $\mathfrak D_{\mathfrak q_0}(\theta,z)=1+|\theta|^{\mathfrak q_0} +|z|^{\mathfrak q_0}$. According to~Th. 6.0 in \cite{villani2009optimal}, to prove the lemma it is enough to show that  a.s. for all $s\ge 0$,  
\begin{equation}\label{eq.cv1}
 \lim_{ N\to +\infty}\rho_s^{N}=\bar\mu_s\otimes\gamma \ \text{    in } \mathcal P(\mathbf R^{d+1}\times\mathbf R^d) \text{ and }
 \lim_{ N\to +\infty}\langle \mathfrak D_{\mathfrak q_0},\rho_s^{ N}\rangle = \langle \mathfrak D_{\mathfrak q_0},\bar\mu_s\otimes\gamma\rangle.  
\end{equation}
We have for any continuous fonction $h:\mathbf R^{d+1}\times\mathbf R^d\to \mathbf R$ and $s\ge 0$, 
\begin{align}\label{decompo-rho_t}
\langle h,\rho_s^{N}\rangle-\langle h,\bar\mu_s\otimes\gamma\rangle&= \frac 1N \sum_{i=1}^N\Big(h(\theta_{\lfloor Ns\rfloor}^i,\mathsf Z_{\lfloor Ns\rfloor}^i)-\int_{\mathbf R^d} h(\theta_{\lfloor Ns\rfloor}^i,z)\gamma(z)\di z\Big)\nonumber\\
&\quad +\frac 1N\sum_{i=1}^N\int_{\mathbf R^d} h(\theta_{\lfloor Ns\rfloor}^i,z)\gamma(z)\di z-\langle h,\bar\mu_s\otimes\gamma\rangle,
\end{align}
as soon as the $\int_{\mathbf R^d} h(\theta,z)\gamma(z)\di z$'s ($\theta \in \mathbf R^{d+1}$) and $\langle h,\bar\mu_s\otimes\gamma\rangle$ are well defined.
\medskip

\noindent
\textbf{Step 1.} We start by proving the first statement in \eqref{eq.cv1}.    
Let $t\ge 0$. We pick $g\in \mathcal C_b(\mathbf R^{d+1}\times\mathbf R^d)$. Note that in this case \eqref{decompo-rho_t} holds with $h=g$.  
 For ease of notation, we set $\mathscr S^i_k(g)=g(\theta_k^i,\mathsf Z_k^i)-\int_{\mathbf R^d} g(\theta_k^i,z)\gamma(z)\di z$, and we will also simply denote $\mathscr S^i_k(g)$ by $\mathscr S^i_k$.  Note that since $g$ is bounded,  for all $m\in \mathbf N^*$,  $\mathbf E [|\mathscr S^i_k|^m ]\le C$ for some $C>0$ independent of $i\in \{1,\ldots,N\}$, $N\ge 1$, and $k\ge 0$. Let us consider $i_j\in \{0,\ldots,6\}$, $j=1,\ldots, 6$ such that $\sum_{j=1}^6 i_j=6$. Assume that    there exists $j_0\in \{1,\ldots,6\}$ such that $i_{j_0}=1$  and $i_{j_0}\neq i_l$ for all $l\neq j_0$. Then,  it holds:
$$\mathbf E \Big [\prod_{j=1}^6 \mathscr S_k^{i_j} \Big ] =0.$$

Therefore, it holds:
 \begin{align*}
&\mathbf E\Big[\sup_{s\in[0,t]}\Big|\frac 1N \sum_{i=1}^Ng(\theta_{\lfloor Ns\rfloor}^i,\mathsf Z_{\lfloor Ns\rfloor}^i)-\int_{\mathbf R^d} g(\theta_{\lfloor Ns\rfloor}^i,z)\gamma(z)\di z\Big|^6\Big]\\
&\leq \sum_{k=0}^{\lfloor Nt\rfloor}\mathbf E\Big[\Big|\frac 1N \sum_{i=1}^N\mathscr S^i_k\Big|^6\Big]\\
&=\frac{1}{N^6}\sum_{k=0}^{\lfloor Nt\rfloor}\sum_{i=1}^N\mathbf E [|\mathscr S^i_k|^6 ] + \frac{1}{N^6}\sum_{k=0}^{\lfloor Nt\rfloor}\sum_{i\neq j}\mathbf E [(\mathscr S^i_k)^3(\mathscr S^j_k)^3 ] + \frac{1}{N^6}\sum_{k=0}^{\lfloor Nt\rfloor}\sum_{i\neq j}\mathbf E [|\mathscr S^i_k|^4|\mathscr S^j_k|^2 ]\\
&\quad + \frac{1}{N^6}\sum_{k=0}^{\lfloor Nt\rfloor}\sum_{i\neq j\neq\ell}\mathbf E [|\mathscr S^i_k|^2|\mathscr S^j_k|^2|\mathscr S^\ell_k|^2 ] \le \frac{ C}{N^2},
\end{align*}
 where $\sum_{i\neq j\neq\ell}$ is a short notation for the sum over the triples $(i,j,\ell)$ such that $i\neq j$, $j\neq \ell$, and $\ell \neq i$.
By Borel-Cantelli lemma, one deduces that,   for all $t\ge 0$ it holds a.s.
\begin{equation}\label{Borel-cantelli}
\sup_{s\in[0,t]}\Big|\frac 1N \sum_{i=1}^Ng(\theta_{\lfloor Ns\rfloor}^i,\mathsf Z_{\lfloor Ns\rfloor}^i)-\int_{\mathbf R^d} g(\theta_{\lfloor Ns\rfloor}^i,z)\gamma(z)\di z\Big|\to 0 \text{ as $N\to +\infty$}.
\end{equation}
 Considering $t\in \mathbf N$, on deduces that a.s. for all $t\ge 0$, \eqref{Borel-cantelli} holds. 
Let us now show that a.s. for all $s\in\mathbf R_+$,
\begin{equation}\label{eq_2e terme}
\frac 1N\sum_{i=1}^N\int_{\mathbf R^d} g(\theta_{\lfloor Ns\rfloor}^i,z)\gamma(z)\di z-\langle g,\bar\mu_s\otimes\gamma\rangle\to 0  \text{ as $N\to +\infty$}.
\end{equation}
Since $\mathsf W_1\leq \mathsf W_{\mathfrak q_0}$, we have that $
\mu^N\xrightarrow{a.s.}\bar\mu \ \text{in}\ \mathcal D(\mathbf R_+,\mathcal P_1(\mathbf R^{d+1}))$. 
As $\bar\mu\in \mathcal C(\mathbf R_+,\mathcal P_1(\mathbf R^{d+1}))$, it holds a.s. for all $t\in\mathbf R_+$,  $\mu^N_t\to \bar\mu_t$ in  $\mathcal P_1(\mathbf R^{d+1})$. Let us define the function  $G: \theta\in\mathbf R^{d+1}\mapsto\int_{\mathbf R^d}g(\theta,z)\gamma(z)\di z$, which is bounded continuous. We have a.s. for all $s\in\mathbf R_+$, $\langle G,\mu_s^N\rangle\to\langle G,\bar\mu_s\rangle$. This is exactly \eqref{eq_2e terme}. 

 Considering \eqref{decompo-rho_t} together with \eqref{Borel-cantelli} and \eqref{eq_2e terme}, we have shown that  for all $g\in \mathcal C_b(\mathbf R^{d+1}\times\mathbf R^d)$, it holds    a.s.  for all $s\ge 0$:
 \begin{equation}\label{g,rho_t^Nconverge}
\langle g,\rho_s^{N}\rangle\to \langle g,\bar\mu_s\otimes\gamma\rangle.
\end{equation}
 We now would like to prove that   it holds a.s. for all  $g\in \mathcal C_b(\mathbf R^{d+1}\times\mathbf R^d)$ and all $s\ge 0$: $\langle g,\rho_s^{N}\rangle\to \langle g,\bar\mu_s\otimes\gamma\rangle$ (which would exactly implies the  first statement in~\eqref{eq.cv1}).
To this end, by~Remark 5.1.6 in \cite{ambrosio2008gradient}, it is sufficient to show that a.s. for all $s\ge 0$ and $g\in \mathcal C_c(\mathbf R^{d+1}\times\mathbf R^d)$ (the space of continuous functions with compact support), $\langle g,\rho_s^{N}\rangle\to_{N\to\infty}\langle g,\bar\mu_s\otimes\gamma\rangle$. Since the space $\mathcal C_c(\mathbf R^{d+1}\times\mathbf R^d)$ is separable, this last statement follows from 
\eqref{g,rho_t^Nconverge} and a standard continuity argument. 
Hence, we have proved  that a.s. for all $s\ge 0$, $\rho_s^N\to\bar\mu_s\otimes\gamma$.   The proof of the first statement in \eqref{eq.cv1} is complete. 
\medskip

\noindent
\textbf{Step 2.} Let us now prove the second statement in \eqref{eq.cv1}. 
Fix $t\ge 0$.  Note first that by \textbf{A1}, $\gamma$ has moments of every order. Thus, $\int_{\mathbf R^d} \mathfrak D_{\mathfrak q_0}(\theta,z)\gamma(z)\di z= 1+|\theta|^{\mathfrak q_0}+ \langle  |\cdot |^{\mathfrak q_0}, \gamma\rangle$ and $\langle \mathfrak D_{\mathfrak q_0},\bar\mu_t\otimes\gamma\rangle=1+ \langle  |\cdot |^{\mathfrak q_0}, \bar\mu_t\rangle+\langle |\cdot |^{\mathfrak q_0}, \gamma\rangle$ are well defined. Thus  \eqref{decompo-rho_t} holds with  $h= \mathfrak D_{\mathfrak q_0}$.  From the analysis carried out in the first step, \eqref{Borel-cantelli} holds with $g$ is replaced by $\mathfrak D_{\mathfrak q_0}$ if for all $m\ge 1$, $i\in \{1,\ldots,N\}$ and $k\in \{1,\ldots,\lfloor Nt\rfloor\}$, $\mathbf E [|\mathscr S^i_k( \mathfrak D_{\mathfrak q_0})|^m ]\le C$ ($C>0$ independent of $i,k$, and $N$), which is the case if 
  $$\mathbf E [|\mathfrak D_{\mathfrak q_0}(\theta_k^i,\mathsf Z_k^i)|^m]+\mathbf E \Big [\Big | \int_{\mathbf R^d} \mathfrak D_{\mathfrak q_0}(\theta_k^i,z)\gamma(z)\di z\Big |^m\Big ]\le C.$$
On the one hand, we have
$\mathbf E [|\mathfrak D_{\mathfrak q_0}(\theta_k^i,\mathsf Z_k^i)|^m]= \mathbf E [|1+|\theta_k^i|^{\mathfrak q_0} +|\mathsf Z_k^i|^{\mathfrak q_0}|^m]\le C_m (1+\mathbf E [|\theta_k^i|^{\mathfrak q_0m}]+\mathbf E [|\mathsf Z_k^i|^{\mathfrak q_0m}] \le C$ (see Lemma \ref{le.Bounds}). With similar computations, $\mathbf E   [  | \int_{\mathbf R^d} \mathfrak D_{\mathfrak q_0}(\theta_k^i,z)\gamma(z)\di z  |^m  ]<+\infty$. Thus, \eqref{Borel-cantelli}  holds with $g$ is replaced by $\mathfrak D_{\mathfrak q_0}$, i.e.     it holds a.s. for all $t\ge 0$:
\begin{equation}\label{eq.C1C}
\sup_{s\in[0,t]}\Big|\frac 1N \sum_{i=1}^N\mathfrak D_{\mathfrak q_0}(\theta_{\lfloor Ns\rfloor}^i,\mathsf Z_{\lfloor Ns\rfloor}^i)-\int_{\mathbf R^d} \mathfrak D_{\mathfrak q_0}(\theta_{\lfloor Ns\rfloor}^i,z)\gamma(z)\di z\Big|\to 0 \text{ as $N\to +\infty$}.
\end{equation}
 Let us now prove that \eqref{eq_2e terme} holds with $g$ replaced there by $\mathfrak D_{\mathfrak q_0}$. Consider the function $D_0: \theta\in\mathbf R^{d+1}\mapsto\int_{\mathbf R^d}\mathfrak D_{\mathfrak q_0}(\theta,z)\gamma(z)\di z= 1+|\theta|^{\mathfrak q_0}+ \langle |\cdot |^{\mathfrak q_0}, \gamma\rangle$. The function $D_0$ is continuous over $\mathbf R^{d+1}$
and clearly   $\theta\mapsto D_0(\theta)/(1+|\theta|^{\mathfrak q_0})$ is bounded.
Consequently, since $\mu^N\xrightarrow{a.s.}\bar\mu \ \text{in}\ \mathcal D(\mathbf R_+,\mathcal P_{\mathfrak q_0}(\mathbf R^{d+1}))$ and $\bar\mu \in \mathcal C(\mathbf R_+,\mathcal P_{\mathfrak q_0}(\mathbf R^{d+1}))$,  it holds a.s. for all $s\in\mathbf R_+$, $\langle D_0,\mu_s^N\rangle\to\langle D_0,\bar\mu_s\rangle$, which is exactly \eqref{eq_2e terme} when $g$ is replaced by $\mathfrak D_{\mathfrak q_0}$. This achieves the proof of  the second statement in \eqref{eq.cv1}. The proof of the lemma is therefore complete. 

  We end the proof of the lemma by deriving two extra  estimates (namely \eqref{eq.De} and \eqref{eq.De2} below) which will be useful in the proof of Proposition \ref{prop-conv_G-proc-f} when the algorithm \eqref{eq.algo-batch} is considered. Since $\mu^N\xrightarrow{a.s.}\bar\mu \ \text{in}\ \mathcal D(\mathbf R_+,\mathcal P_{\mathfrak q_0}(\mathbf R^{d+1}))$,  using e.g.  Proposition 5.3 in Chapter 3 of \cite{ethier2009markov}, one has a.s.  for all $t\ge 0$,
$$\sup_{N\ge 1} \sup_{s\in [0,t]} |\langle D_0,\mu_s^N\rangle|<+\infty.$$
Say that the previous inequality holds for all   $\omega\in  \Omega^*$ where $\mathbf P(\Omega^*)=1$. 
By \eqref{eq.C1C},   there exists $\Omega'$ with $\mathbf P(\Omega')=1$ and such that for all $\omega\in \Omega'$ and $t\ge 0$,  it holds as $N\to +\infty$
$$\sup_{s\in[0,t]}\big | \langle \mathfrak D_{\mathfrak q_0}, \rho_s^{N}(\omega) \rangle -\langle D_0,\mu_s^N(\omega)\rangle \big|\to 0.$$ 
Therefore, for all $\omega\in \Omega'\cap \Omega^*$, there   exists $N_1(\omega)\ge 1$ such that that for all $N\ge N_1(\omega)$ and $t\ge 0$, 
\begin{align*}
\sup_{s\in[0,t]}\big | \langle \mathfrak D_{\mathfrak q_0}, \rho_s^{N}(\omega) \rangle \big |&\le 1+ \sup_{s\in [0,t]} |\langle D_0,\mu_s^N(\omega)\rangle\\
&\le 1+\sup_{N\ge 1} \sup_{s\in [0,t]} |\langle D_0,\mu_s^N(\omega)\rangle<+\infty.
\end{align*}
Therefore, for all $\omega\in \Omega'\cap \Omega^*$  and $t\ge 0$  
\begin{equation}\label{eq.De}
\sup_{N\ge 1}\sup_{s\in[0,t]}| \langle \mathfrak D_{\mathfrak q_0}, \rho_s^{N}(\omega) \rangle \big |<+\infty,
\end{equation}  
i.e. \eqref{eq.De} holds a.s. for all $t\ge 0$ (since   $\mathbf P(\Omega'\cap \Omega^*)=1$). 
 Finally, it holds  that for all  $m\ge 1$ and  $0\le t_1\le t_2$,
    \begin{align*} 
   \sup_{s\in[t_1,t_2]} \frac 1N \sum_{i=1}^N|\mathsf Z_{\lfloor Ns\rfloor}^i|^m&\le    \sum_{k=\lfloor Nt_1\rfloor}^{\lfloor Nt_2\rfloor}  \frac 1N \sum_{i=1}^N|\mathsf Z_k^i|^{m}.
  \end{align*} 
 Since the $\mathsf Z_k^i$'s are i.i.d. with moments of all order (see \textbf{A1}), one deduces that:
      \begin{align*} 
      \mathbf E\Big[ \sup_{s\in[t_1,t_2]} \frac 1N \sum_{i=1}^N|\mathsf Z_{\lfloor Ns\rfloor}^i|^m\Big]&\le (\lfloor Nt_2\rfloor-\lfloor Nt_1\rfloor+1 ) \mathbf E_\gamma[|\mathsf Z|^m]. 
  \end{align*} 
Consequently, using also~ Lemma 19 in \cite{colt}, one has:
\begin{align} 
\nonumber
 &\mathbf E\Big[\sup_{s\in[t_1,t_2]}\Big |\frac 1N \sum_{i=1}^N\mathfrak D_{\mathfrak q_0}(\theta_{\lfloor Ns\rfloor}^i ,\mathsf Z_{\lfloor Ns\rfloor}^i ) \Big|^m\Big]  \\
 \nonumber 
 &\le  \mathbf E\Big[\sup_{s\in[t_1,t_2]} \frac 1N  \sum_{i=1}^NC_m[1+|\theta_{\lfloor Ns\rfloor}^i |^{\mathfrak q_0m}+ |\mathsf Z_{\lfloor Ns\rfloor}^i  |^{\mathfrak q_0m}]\Big]\\
 \nonumber 
&\le C_m  \mathbf E\Big[\sup_{s\in[0,t_2]} \langle 1+|.|^ {\mathfrak q_0m}, \mu_t^N\rangle\Big]+C_m (\lfloor Nt_2\rfloor-\lfloor Nt_1\rfloor+1 ) \mathbf E_\gamma[|\mathsf Z|^{\mathfrak q_0m}]\\
\nonumber
&\le C+ C(\lfloor Nt_2\rfloor-\lfloor Nt_1\rfloor+1 ) \mathbf E_\gamma[|\mathsf Z|^{\mathfrak q_0m}],
\end{align}  
where $C>0$ is independent of $N\ge 1$. In particular, when  $t_2-t_1\le 1/N$, it holds 
\begin{align} 
\label{eq.De2}
\mathbf E\Big[\sup_{s\in[t_1,t_2]}\big | \langle \mathfrak D_{\mathfrak q_0}, \rho_s^{N} \rangle \big |^m\Big] \le C,
 \end{align}
 where $C>0$ is independent of $N\ge 1$.
\end{proof}

With the same arguments as those used to prove Proposition \ref{prop-conv_G-proc-f}  when  the $\{\theta^i_k, i\in \{1,\ldots,N\}\}$'s are  generated    by the algorithm~\eqref{eq.algo-ideal}, we obtain

\begin{proposition}\label{p-conv-g-process-z1z2}
 Assume  that the $\{\theta^i_k, i\in \{1,\ldots,N\}\}$'s are  generated  by the algorithm~\eqref{eq.algo-z1z2}.  Assume also {\rm \textbf{A}}. Then, for every $f\in \mathcal C^{2,\mathfrak j_0}(\mathbf R^{d+1})$, the sequence $\{t\in\mathbf R_+\mapsto\sqrt N\mathbf M_t^N[f]\}_{N\ge1}$ converges in distribution in $\mathcal D(\mathbf R_+,\mathbf R)$ towards a process $\mathcal X^f\in\mathcal C(\mathbf R_+,\mathbf R)$ that has independent Gaussian increments. Moreover, for all $t\in\mathbf R_+$, 
$$\mathbf E[\mathcal X_t^f]=0 \text{ and } \Var(\mathcal X_t^f)=\kappa^2\int_0^t\Var_{\pi\otimes\gamma^{\otimes 2}}(\mathscr Q[f](x,y,z^1,z^2,\bar\mu_s))\di s,  $$
where we recall $\mathscr Q[f](x,y,z^1,z^2,\bar\mu_v)=\langle\phi(\cdot,z^1,x)-y,\bar\mu_v\rangle\langle\nabla_\theta f\cdot\nabla_\theta\phi(\cdot,z^2,x),\bar\mu_v\rangle$  (see Theorem \ref{thm-clt-ideal}).
\end{proposition}

\begin{proposition}\label{p-conv-to-g-process}
  Assume  that the $\{\theta^i_k, i\in \{1,\ldots,N\}\}$'s are  generated  either  by the algorithm~\eqref{eq.algo-ideal},~\eqref{eq.algo-batch},  or~\eqref{eq.algo-z1z2}. Assume also {\rm \textbf{A}}.
Then,  $(\sqrt N\mathbf M^N)_{N\ge1}$ converges in distribution in $\mathcal D(\mathbf R_+,\mathcal H^{-\mathfrak J_3,\mathfrak j_3}(\mathbf R^{d+1}))$ to a $\mathfrak G$-process $\mathscr G\in\mathcal C(\mathbf R_+,\mathcal H^{-\mathfrak J_3,\mathfrak j_3}(\mathbf R^{d+1}))$ (see Definition \ref{d-G-process}) with covariance structure given by: for all $1\le i,j\le k$,  $f_1,\dots, f_k\in\mathcal H^{\mathfrak J_3,\mathfrak j_3}(\mathbf R^{d+1})$ and $0\le s\le t$, 
\begin{itemize}
\item When the $\{\theta^i_k, i\in \{1,\ldots,N\}\}$'s are  generated  either  by the algorithm \eqref{eq.algo-ideal} and \eqref{eq.algo-batch},
\begin{align*}
\Cov(\mathscr G_t[f_i],\mathscr G_s[f_j])=\eta^2\int_0^s\Cov(\mathscr Q[f_i](x,y,\bar\mu_v),\mathscr Q[f_j](x,y,\bar\mu_v))\di v,
\end{align*}
where we recall $\mathscr Q[f](x,y,\bar\mu_v)=\langle\phi(\cdot,\cdot,x)-y,\bar\mu_v\otimes\gamma\rangle\langle\nabla_\theta f\cdot\nabla_\theta\phi(\cdot,\cdot,x),\bar\mu_v\otimes\gamma\rangle$  (see Theorem~\ref{thm-clt-ideal}).
\item When the~$\{\theta^i_k, i\in \{1,\ldots,N\}\}$'s are  generated  either  by the algorithm~\eqref{eq.algo-z1z2},
\begin{align*}
\Cov(\mathscr G_t[f_i],\mathscr G_s[f_j])=\eta^2\int_0^s\Cov(\mathscr Q[f_i](x,y,z^1,z^2,\bar\mu_v),\mathscr Q[f_j](x,y,z^1,z^2,\bar\mu_v))\di v,
\end{align*}
where we recall $\mathscr Q[f](x,y,z^1,z^2,\bar\mu_v)=\langle\phi(\cdot,z^1,x)-y,\bar\mu_v\rangle\langle\nabla_\theta f\cdot\nabla_\theta\phi(\cdot,z^2,x),\bar\mu_v\rangle$  (see Theorem~\ref{thm-clt-ideal}).
\end{itemize}
\end{proposition} 
\begin{proof}
The proof of Proposition \ref{p-conv-to-g-process} relies on the same arguments as those used to prove  Prop. 3.13 in \cite{jmlr}. 
\end{proof}

\subsection{On the limit points of $(\eta^N,\sqrt N\mathbf M^N)_{N\ge1}$}
In this section, we come back to the case   when the $\{\theta^i_k, i\in \{1,\ldots,N\}\}$'s are  generated    by  the algorithm  \eqref{eq.algo-batch}. The other two cases (namely \eqref{eq.algo-ideal} and \eqref{eq.algo-z1z2}) are treated similarly, and all the results of this section also holds for each of these other two algorithms. 
 
Let us derive  the pre-limit equation for the fluctuation process $\eta^N$, see  \eqref{pre-lim-eta} just below. 
On the one hand, one has for all $N\ge1$, $t\ge 0$ and $f\in \mathcal H^{\mathfrak J_0,\mathfrak j_0}(\mathbf R^{d+1})$,  
\begin{align*}
&\sqrt N\int_0^t\int_{\mathsf X\times\mathsf Y}\langle\phi(\cdot,\cdot,x)-y,\bar\mu_s\otimes\gamma\rangle\langle\nabla_\theta f\cdot\nabla_\theta\phi(\cdot,\cdot,x),\bar\mu_s\otimes\gamma\rangle\pi(\di x,\di y)\di s\\
&-\sqrt N\int_0^t\int_{\mathsf X\times\mathsf Y}\langle\phi(\cdot,\cdot,x)-y,\mu_s^N\otimes\gamma\rangle\langle\nabla_\theta f\cdot\nabla_\theta\phi(\cdot,\cdot,x),\mu_s^N\otimes\gamma\rangle\pi(\di x,\di y)\di s\\
&= -\int_0^t\int_{\mathsf X\times\mathsf Y}\langle\phi(\cdot,\cdot,x)-y,\bar\mu_s\otimes\gamma\rangle\langle\nabla_\theta f\cdot\nabla_\theta\phi(\cdot,\cdot,x),\eta_s^N\otimes\gamma\rangle\pi(\di x,\di y)\di s\\
&-\int_0^t\int_{\mathsf X\times\mathsf Y}\langle\phi(\cdot,\cdot,x),\eta_s^N\otimes\gamma\rangle\langle\nabla_\theta f\cdot\nabla_\theta\phi(\cdot,\cdot,x),\bar\mu_s\otimes\gamma\rangle\pi(\di x,\di y)\di s\\
&-\frac{1}{\sqrt N}\int_0^t\int_{\mathsf X\times\mathsf Y}\langle\phi(\cdot,\cdot,x),\eta_s^N\otimes\gamma\rangle\langle\nabla_\theta f\cdot\nabla_\theta\phi(\cdot,\cdot,x),\eta_s^N\otimes\gamma\rangle\pi(\di x,\di y)\di s.
\end{align*}
Hence, using \eqref{eq.pre_limitz1zN} and \eqref{eq.P2}, we obtain the following pre-limit equation for $\eta^N$ : 
\begin{align}\label{pre-lim-eta}
\langle f,\eta_t^N\rangle-\langle f,\eta_0^N\rangle&= -\kappa\int_0^t\int_{\mathsf X\times\mathsf Y}\langle\phi(\cdot,\cdot,x)-y,\bar\mu_s\otimes\gamma\rangle\langle\nabla_\theta f\cdot\nabla_\theta\phi(\cdot,\cdot,x),\eta_s^N\otimes\gamma\rangle\pi(\di x,\di y)\di s\nonumber\\
&-\kappa\int_0^t\int_{\mathsf X\times\mathsf Y}\langle\phi(\cdot,\cdot,x),\eta_s^N\otimes\gamma\rangle\langle\nabla_\theta f\cdot\nabla_\theta\phi(\cdot,\cdot,x),\bar\mu_s\otimes\gamma\rangle\pi(\di x,\di y)\di s\nonumber\\
&-\frac{\kappa}{\sqrt N}\int_0^t\int_{\mathsf X\times\mathsf Y}\langle\phi(\cdot,\cdot,x),\eta_s^N\otimes\gamma\rangle\langle\nabla_\theta f\cdot\nabla_\theta\phi(\cdot,\cdot,x),\eta_s^N\otimes\gamma\rangle\pi(\di x,\di y)\di s\nonumber\\
&-\kappa\int_0^t\langle\nabla_\theta f\cdot\nabla_\theta\mathscr D_{\mathrm{KL}}(q^1_\cdot|P_0^1),\eta_s^N\rangle\di s\nonumber\\
&+ \frac{\kappa}{\sqrt N}\int_0^t\int_{\mathsf X\times\mathsf Y}\Big\langle\langle\phi(\cdot,\cdot,x)-y,\gamma\rangle\langle\nabla_\theta f\cdot\nabla_\theta\phi(\cdot,\cdot,x),\gamma\rangle,\mu_s^N\Big\rangle\pi(\di x,\di y)\di s \nonumber\\
&-\frac{\kappa}{\sqrt N}\int_0^t\int_{\mathsf X\times\mathsf Y}\Big\langle(\phi(\cdot,\cdot,x)-y)\nabla_\theta f\cdot\nabla_\theta\phi(\cdot,\cdot,x),\mu_s^N\otimes\gamma\Big\rangle\pi(\di x,\di y)\di s\nonumber\\
&+ \sqrt N\mathbf M_t^N[f]+\sqrt N\mathbf W_t^N[f] + \sqrt N\mathbf R_t^N[f]. 
\end{align}
The aim of this section is to pass to the limit $N\to +\infty$ in \eqref{pre-lim-eta}.  We start with the following lemma whose proof, identical to the one of Lemma 3.16 in \cite{jmlr}, is omitted. 
\begin{lemma}\label{lem-eta_0}
Assume {\rm \textbf{A}}.
Then, the sequence $(\eta_0^N)_{N\ge1}$ converges in distribution in $\mathcal H^{-\mathfrak J_3+1,\mathfrak j_3}(\mathbf R^{d+1})$ towards a variable $\nu_0$ which is the unique (in distribution) $\mathcal H^{-\mathfrak J_3+1,\mathfrak j_3}(\mathbf R^{d+1})$-valued random variable such that for all $k\ge 1$ and $f_1,\dots,f_k\in\mathcal H^{\mathfrak J_3-1,\mathfrak j_3}(\mathbf R^{d+1})$, $(\langle f_1,\nu_0\rangle,\dots,\langle f_k,\nu_0\rangle)^T\sim\mathfrak N(0,\mathfrak C(f_1,\dots,f_k))$, where $\mathfrak C(f_1,\dots,f_k)$ is the covariance matrix of the vector $(f_1(\theta_0^1),\dots,f_k(\theta_0^1))^T$. 
\end{lemma}
Let us now set
\begin{equation}\label{def-E}
\mathscr E=\mathcal D(\mathbf R_+,\mathcal H^{-\mathfrak J_3+1,\mathfrak j_3}(\mathbf R^{d+1}))\times\mathcal D(\mathbf R_+,\mathcal H^{-\mathfrak J_3,\mathfrak j_3}(\mathbf R^{d+1})).
\end{equation}
According to Propositions \ref{p-rc-eta} and \ref{p-rc-m^N}, $(\eta^N,\sqrt\mathbf M^N)$ is tight in $\mathscr E$. Let $(\eta^\star,\mathscr G^*)$ be one of its limit point in $\mathscr E$. Along some subsequence $N'$, it holds: 
\begin{align*}
(\eta^{N'},\sqrt{N'}\mathbf M^{N'})\to (\eta^\star,\mathscr G^\star), \text{ as } N'\to\infty.
\end{align*} 
Considering the marginal distributions, and according to Lemma \ref{lem-reg-LP}, it holds a.s.
\begin{equation}\label{eta*-and-G*-continuous}
\eta^\star\in\mathcal C(\mathbf R_+,\mathcal H^{-\mathfrak J_3+1,\mathfrak j_3}(\mathbf R^{d+1})) \text{ and }\mathscr G^\star\in\mathcal C(\mathbf R_+,\mathcal H^{ -\mathfrak J_3 ,\mathfrak j_3}(\mathbf R^{d+1})). 
\end{equation}
By uniqueness of the limit in distribution, using Lemma \ref{lem-eta_0} (together with the fact that the function $m\in\mathcal D(\mathbf R_+,\mathcal H^{-\mathfrak J_3+1,\mathfrak j_3}(\mathbf R^{d+1}))\mapsto m_0\in\mathcal H^{-\mathfrak J_3+1,\mathfrak j_3}(\mathbf R^{d+1})$ is continuous) and Proposition \ref{p-conv-to-g-process}, it also holds:
\begin{equation}\label{eta*0=nu0}
\eta_0^\star\overset{\mathscr L}{=}\nu_0 \text{ and } \mathscr G^\star\overset{\mathscr L}{=}\mathscr G.
\end{equation}
\begin{proposition}\label{p-eta*-weak-sol}
Assume {\rm \textbf{A}}.
Then, $\eta^\star$ is a weak solution of {\rm \textbf{(EqL)}}  with initial distribution $\nu_0$. 
\end{proposition}
\begin{proof}
Let us introduce, for $\Phi\in\mathcal H^{-\mathfrak J_3+1,\mathfrak j_3}(\mathbf R^{d+1})$, $f\in\mathcal H^{ \mathfrak J_3,\mathfrak j_3-1}(\mathbf R^{d+1})$  , and $s\ge 0$: 
\begin{equation}\label{def-frakU}
\mathfrak U_s[f](\Phi)= \kappa\int_{\mathsf X\times\mathsf Y}\langle\phi(\cdot,\cdot,x)-y,\bar\mu_s\otimes\gamma\rangle\langle\nabla_\theta f\cdot\nabla_\theta\phi(\cdot,\cdot,x),\Phi\otimes\gamma\rangle\pi(\di x,\di y),
\end{equation}
\begin{equation}\label{def-frakV}
\mathfrak V_s[f](\Phi)=\kappa\int_{\mathsf X\times\mathsf Y}\langle\phi(\cdot,\cdot,x),\Phi\otimes\gamma\rangle\langle\nabla_\theta f\cdot\nabla_\theta\phi(\cdot,\cdot,x),\bar\mu_s\otimes\gamma\rangle\pi(\di x,\di y),
\end{equation}
and 
\begin{equation}\label{def-frakW}
\mathfrak W_s[f](\Phi) =\kappa\langle\nabla_\theta f\cdot\nabla_\theta\mathscr D_{\mathrm{KL}}(q^1_\cdot|P_0^1),\Phi\rangle
\end{equation}
The term $\mathfrak U_s[f](\Phi)$ is well defined because $f\in\mathcal H^{-\mathfrak J_3,\mathfrak j_3-1}(\mathbf R^{d+1}) \hookrightarrow\mathcal H^{\mathfrak J_3,\mathfrak j_3}(\mathbf R^{d+1})$.
 Since $\mathfrak j_3>(d+1)/2$, using  \eqref{phi-normeH} and because $\bar\mu_s\in\mathcal P_{\mathfrak j_0}(\mathbf R^{d+1})$ $(f\in \mathcal C^{1,\mathfrak j_0}(\mathbf R^{d+1})$),  $\mathfrak V_s[f](\Phi)$ is well defined. 
 The term $\mathfrak W_s[f](\Phi)$ is well defined because of \eqref{nablafnablaKLnormeHfini}. 
Equation   \eqref{pre-lim-eta} can  be rewritten as follows:
\begin{equation}\label{eq-prelim-eta-operators}
\langle f,\eta_t^N\rangle-\langle f,\eta_0^N\rangle + \int_0^t(\mathfrak U_s[f](\eta_s^N)+\mathfrak V_s[f](\eta_s^N)+\mathfrak W_s[f](\eta_s^N))\di s-\sqrt N\mathbf M_t^N[f] = \mathbf e_t^N[f],  
\end{equation}
where 
\begin{align*}
\mathfrak{Re}_t^N[f] &= -\frac{\kappa}{\sqrt N}\int_0^t\int_{\mathsf X\times\mathsf Y}\langle\phi(\cdot,\cdot,x),\eta_s^N\otimes\gamma\rangle\langle\nabla_\theta f\cdot\nabla_\theta\phi(\cdot,\cdot,x),\eta_s^N\otimes\gamma\rangle\pi(\di x,\di y)\di s\\
&+ \frac{\kappa}{\sqrt N}\int_0^t\int_{\mathsf X\times\mathsf Y}\Big\langle\langle\phi(\cdot,\cdot,x)-y,\gamma\rangle\langle\nabla_\theta f\cdot\nabla_\theta\phi(\cdot,\cdot,x),\gamma\rangle,\mu_s^N\Big\rangle\pi(\di x,\di y)\di s \nonumber\\
&-\frac{\kappa}{\sqrt N}\int_0^t\int_{\mathsf X\times\mathsf Y}\Big\langle(\phi(\cdot,\cdot,x)-y)\nabla_\theta f\cdot\nabla_\theta\phi(\cdot,\cdot,x),\mu_s^N\otimes\gamma\Big\rangle\pi(\di x,\di y)\di s\nonumber\\
&\quad +\sqrt N\mathbf W_t^N[f] + \sqrt N\mathbf R_t^N[f]. 
\end{align*}
Fix $f\in\mathcal H^{\mathfrak J_3,\mathfrak j_3-1}(\mathbf R^{d+1})$   and $t\in\mathbf R_+$. 
\paragraph{Step 1.} In this step we study the continuity of the mapping
\begin{equation}
\mathfrak B_t[f]: m\in\mathcal D(\mathbf R_+,\mathcal H^{-\mathfrak J_3+1,\mathfrak j_3}(\mathbf R^{d+1}))\mapsto\langle f,m_t\rangle+\int_0^t(\mathfrak U_s[f](m_s)+\mathfrak V_s[f](m_s)+\mathfrak W_s[f](m_s))\di s
\end{equation}
Let $(m^N)_{N\ge1}$ such that $m^N\to m$ in $\mathcal D(\mathbf R_+,\mathcal H^{-\mathfrak J_3+1,\mathfrak j_3}(\mathbf R^{d+1}))$.
Using \eqref{eq-nabla-f-nablaphi-H}, it holds, for all $N\ge1$, $s\in[0,t]$ and $x\in\mathsf X$, 
\begin{align*}
&|\langle\phi(\cdot,\cdot,x)-y,\bar\mu_s\otimes\gamma\rangle\langle\nabla_\theta f\cdot\nabla_\theta\phi(\cdot,\cdot,x),m_s^N\otimes\gamma\rangle|\\
&\le C\big\| \nabla_\theta f\cdot\nabla_\theta \mathfrak H(\cdot,x)\big\|_{\mathcal H^{\mathfrak J_3-1,\mathfrak j_3}}\sup_{N\ge1}\sup_{s\in[0,t]}\|m_s^N\|_{\mathcal H^{-\mathfrak J_3+1,\mathfrak j_3}}\\
&\le C\|f\|_{\mathcal H^{\mathfrak J_3,\mathfrak j_3}}\sup_{N\ge1}\sup_{s\in[0,t]}\|m_s^N\|_{\mathcal H^{-\mathfrak J_3+1,\mathfrak j_3}}<+\infty.
\end{align*}
We also have, by \eqref{phi-normeH} and  the embedding $f\in\mathcal H^{\mathfrak J_3,\mathfrak j_3-1}(\mathbf R^{d+1})\hookrightarrow\mathcal C^{1,\mathfrak j_0}(\mathbf R^{d+1})$ and the fact that $\bar\mu\in\mathcal C(\mathbf R_+,\mathcal P_{j_0}(\mathbf R^{d+1}))$, 
\begin{align*}
|\langle\phi(\cdot,\cdot,x),m_s^N\otimes\gamma\rangle\langle\nabla_\theta f\cdot\nabla_\theta\phi(\cdot,\cdot,x),\bar\mu_s\otimes\gamma\rangle|&\le C\sup_{N\ge1}\sup_{s\in[0,t]}\|m_s^N\|_{\mathcal H^{-\mathfrak J_3+1,\mathfrak j_3}}\\
&\quad \times \|f\|_{\mathcal C^{1,\mathfrak j_0}}\sup_{s\in[0,t]}\langle 1+|\cdot|^{\mathfrak j_0},\bar\mu_s\rangle<+\infty.
\end{align*}
Finally, using \eqref{nablafnablaKLnormeHfini}, 
\begin{align*}
|\langle\nabla_\theta f\cdot\nabla_\theta\mathscr D_{\mathrm{KL}}(q^1_\cdot|P_0^1),m_s^N\rangle|&\le  \|\nabla_\theta f\cdot\nabla_\theta\mathscr D_{\mathrm{KL}}(q^1_\cdot|P_0^1)\|_{\mathcal H^{\mathfrak J_3-1,\mathfrak j_3}}\sup_{N\ge1}\sup_{s\in[0,t]}\|m_s^N\|_{\mathcal H^{-\mathfrak J_3+1,\mathfrak j_3}}\\
&\le C\|f\|_{\mathcal H^{\mathfrak J_3,\mathfrak j_3-1}}\sup_{N\ge1}\sup_{s\in[0,t]}\|m_s^N\|_{\mathcal H^{-\mathfrak J_3+1,\mathfrak j_3}}<+\infty.
\end{align*}
These bounds allow to apply the dominated convergence theorem to obtain that $\mathfrak B_t[f](m^N)\to\mathfrak B_t[f](m)$, as soon as $t$ is a continuity point of $m$. Consequently, using \eqref{eta*-and-G*-continuous} and the continuous mapping theorem 2.7 in  \cite{billingsley2013convergence}, it holds, for all $t\in\mathbf R_+$ and $f\in \mathcal H^{\mathfrak J_3,\mathfrak j_3-1}(\mathbf R^{d+1})$, 
\begin{equation}\label{B-conv-distrib}
\mathfrak B_t[f](\eta^{N'})-\langle f,\eta_0^{N'}\rangle-\sqrt{N'}\mathbf M_t^{N'}[f]\xrightarrow[N'\to \infty]{\mathscr L}\mathfrak B_t[f](\eta^{*})-\langle f,\eta_0^{*}\rangle-\mathscr G^*_t[f].
\end{equation}
\paragraph{Step 2.}  In this step, we prove that for any $t\in\mathbf R_+$ and $f\in\mathcal H^{\mathfrak J_3,\mathfrak j_3-1}(\mathbf R^{d+1})$: 
\begin{equation}\label{E[e]to0}
\mathbf E\big [|\mathfrak{Re}_t^N[f]|\big ]\to_{N\to\infty} 0.
\end{equation}
By \eqref{eq-nabla-f-nablaphi-H}-\eqref{phi-normeH}, the embedding $\mathcal H^{-\mathfrak J_1,\mathfrak j_1}(\mathbf R^{d+1})\hookrightarrow\mathcal H^{-\mathfrak J_3+1,\mathfrak j_3}(\mathbf R^{d+1})$  and Lemma \ref{lem-E[eta]<infty}, it holds
\begin{align*}
&\mathbf E\Big[\Big| \frac{1}{\sqrt N}\int_0^t\int_{\mathsf X\times\mathsf Y}\langle\phi(\cdot,\cdot,x),\eta_s^N\otimes\gamma\rangle\langle\nabla_\theta f\cdot\nabla_\theta\phi(\cdot,\cdot,x),\eta_s^N\otimes\gamma\rangle\pi(\di x,\di y)\di s\Big|\Big]\\
&\le \frac{C\Vert f\Vert_{\mathcal H^{ \mathfrak J_3-1,\mathfrak j_3}}}{\sqrt N}\int_0^t\mathbf E\Big[\|\eta_s^N\|_{\mathcal H^{-\mathfrak J_3+1,\mathfrak j_3}}^2\Big]\di s \le \frac{C\Vert f\Vert_{\mathcal H^{ \mathfrak J_3-1,\mathfrak j_3}}}{\sqrt N}\int_0^t\mathbf E\Big[\|\eta_s^N\|_{\mathcal H^{-\mathfrak J_1,\mathfrak j_1}}^2\Big]\di s \le \frac{C\Vert f\Vert_{\mathcal H^{ \mathfrak J_3-1,\mathfrak j_3}}}{\sqrt N}.
\end{align*}
By Lemma \ref{le.Bounds}, we have 
\begin{align*}
&\mathbf E\Big[ \frac{1}{\sqrt N}\int_0^t\int_{\mathsf X\times\mathsf Y}\Big|\Big\langle\langle\phi(\cdot,\cdot,x)-y,\gamma\rangle\langle\nabla_\theta f\cdot\nabla_\theta\phi(\cdot,\cdot,x),\gamma\rangle,\mu_s^N\Big\rangle\Big|\pi(\di x,\di y)\di s \nonumber\\
&+\frac{\kappa}{\sqrt N}\int_0^t\int_{\mathsf X\times\mathsf Y}\Big|\Big\langle(\phi(\cdot,\cdot,x)-y)\nabla_\theta f\cdot\nabla_\theta\phi(\cdot,\cdot,x),\mu_s^N\otimes\gamma\Big\rangle\Big|\pi(\di x,\di y)\di s \Big]\le \frac{C \Vert f\Vert_{\mathcal C^{ 1,\mathfrak j_0}}}{\sqrt N}.
\end{align*}
In addition, from  \eqref{E[sup R_t^2]}, $\mathbf E[  \sqrt N|\mathbf R_t^N[f]|]\le \|f\|_{\mathcal H^{\mathfrak J_0,\mathfrak j_0}}/\sqrt N$. Moreover, it is straightforward to prove that  $\mathbf E[ |\mathbf W_t^N[f]|]\le \|f\|_{\mathcal H^{\mathfrak J_0,\mathfrak j_0}}/ N$. Hence, we have proved \eqref{E[e]to0}. 


\paragraph{Step 3.} End of the proof of Proposition \ref{p-eta*-weak-sol}. By \eqref{eq-prelim-eta-operators},  \eqref{B-conv-distrib} and \eqref{E[e]to0}, we deduce that for all $f\in\mathcal H^{\mathfrak J_3,\mathfrak j_3-1}(\mathbf R^{d+1})$, and $t\in\mathbf R_+$, it holds a.s. $\mathfrak B_t[f](\eta^\star)-\langle f,\eta_0^\star\rangle-\mathscr G^\star_t[f]=0$. Since $\mathcal H^{\mathfrak J_3,\mathfrak j_3-1}(\mathbf R^{d+1})$ and $\mathbf R_+$ are separable, we conclude by a standard continuity argument (and using that every Hilbert-Schmidt embedding is  continuous) that a.s. for all $f\in\mathcal H^{\mathfrak J_3,\mathfrak j_3-1}(\mathbf R^{d+1})$ and $t\in\mathbf R_+$, $\mathfrak B_t[f](\eta^\star)-\langle f,\eta_0^\star\rangle-\mathscr G^\star_t[f]=0$. Hence, $\eta^\star$ is a weak solution of {\rm \textbf{(EqL)}} with initial distribution $\nu_0$ (see \eqref{eta*0=nu0}). This ends the proof of Proposition \ref{p-eta*-weak-sol}. 
\end{proof}

\subsection{Pathwise uniqueness and   proof of Theorem \ref{thm-clt-ideal}}
Throughout this section, we consider algorithm \eqref{eq.algo-batch}, but we recall that all our statements are valid for algorithms \eqref{eq.algo-ideal} and \eqref{eq.algo-z1z2}.
\begin{proposition}\label{p-pathwise-uniqu}
Assume {\rm \textbf{A}}.
Then strong (pathwise) uniqueness holds for {\rm \textbf{(EqL)}}. Namely, on a fixed probability space, given a $\mathcal H^{-\mathfrak J_3+1,\mathfrak j_3}(\mathbf R^{d+1})$-valued random variable $\nu$ and a $\mathfrak G$-process $\mathscr G\in\mathcal C(\mathbf R_+,\mathcal H^{-\mathfrak J_3,\mathfrak j_3}(\mathbf R^{d+1}))$, there exists at most one $\mathcal C(\mathbf R_+,\mathcal H^{-\mathfrak J_3+1,\mathfrak j_3}(\mathbf R^{d+1}))$-valued process $\eta$ solution to {\rm \textbf{(EqL)}} with $\eta_0=\nu$ almost surely. 
\end{proposition}
\begin{proof}
\begin{sloppypar}
By linearity of the involved operators in {\rm \textbf{(EqL)}}, it is enough to consider a $\mathcal C(\mathbf R_+,\mathcal H^{-\mathfrak J_3+1,\mathfrak j_3}(\mathbf R^{d+1}))$-valued process $\eta$ solution to  {\rm \textbf{(EqL)}} when a.s. $\nu=0$ and $\mathscr G=0$, i.e., for every $f\in\mathcal H^{\mathfrak J_3,\mathfrak j_3-1}(\mathbf R^{d+1})$ and $t\in\mathbf R_+$, 
\end{sloppypar}
\begin{equation}\label{syst-uniqueness}
\begin{cases}
\langle f,\eta_t\rangle + \int_0^t(\mathfrak U_s[f](\eta_s)+\mathfrak V_s[f](\eta_s)+\mathfrak W_s[f](\eta_s))\di s =0,\\
\langle f,\eta_0\rangle=0,
\end{cases}
\end{equation}
where we recall that $\mathfrak U$, $\mathfrak V$ and $\mathfrak W$ are defined respectively in \eqref{def-frakU}, \eqref{def-frakV} and \eqref{def-frakW}.
Pick $T>0$. By~\eqref{syst-uniqueness}, 
we have, a.s. for all $f\in \mathcal H^{\mathfrak J_3,\mathfrak j_3-1}(\mathbf R^{d+1})$ and $t\in[0,T]$, 
\begin{equation}\label{f,eta^2-uniqueness}
\langle f,\eta_t\rangle^2= -2 \int_0^t(\mathfrak U_s[f](\eta_s)+\mathfrak V_s[f](\eta_s)+\mathfrak W_s[f](\eta_s))\langle f,\eta_s\rangle\di s.\\
\end{equation}
Since $\sup_{s\in[0,T]}\langle 1+|\cdot|^{\mathfrak j_0},\bar\mu_s\rangle<+\infty$, and using \eqref{phi-normeH}, 
\begin{align*}
&-2\int_0^t\mathfrak V_s[f](\eta_s)\langle f,\eta_s\rangle\di s\\
&\le 2\kappa\int_0^t\Big[\langle f,\eta_s\rangle^2+\int_{\mathsf X\times\mathsf Y}|\langle\phi(\cdot,\cdot,x),\eta_s\otimes\gamma\rangle|^2|\langle\nabla_\theta f\cdot\nabla_\theta\phi(\cdot,\cdot,x),\bar\mu_s\otimes\gamma\rangle|^2\pi(\di x,\di y)\Big]\di s\\
&\le C\int_0^t\Big[\langle f,\eta_s\rangle^2+ \|\eta_s\|_{\mathcal H^{-\mathfrak J_3,\mathfrak j_3}}^2\|f\|_{\mathcal C^{1,\mathfrak j_0}}^2\Big]\di s \le C\int_0^t\Big[\langle f,\eta_s\rangle^2+ \|\eta_s\|_{\mathcal H^{-\mathfrak J_3+1,\mathfrak j_3}}^2\|f\|_{\mathcal H^{\mathfrak J_0,\mathfrak j_0}}^2\Big]\di s.
\end{align*}
Consider   an orthonormal basis $\{f_a\}_{a\ge1}$ of $\mathcal H^{-\mathfrak J_3,\mathfrak j_3-1}(\mathbf R^{d+1})$. 
Recall  that $\mathbf T_x : f\in \mathcal H^{-\mathfrak J_3,\mathfrak j_3-1}(\mathbf R^{d+1})\mapsto\int_{\mathbf R^d}\nabla_\theta f\cdot\nabla_\theta\phi(\cdot,z,x)\gamma(z)\di z\in\mathcal H^{\mathfrak J_3-1,\mathfrak j_3-1}(\mathbf R^{d+1})$ (see \eqref{operator-Tx}). By Lemma B.2 in  \cite{jmlr}, one deduces that:
\begin{align*}
-2\sum_{a\ge1} \int_0^t\mathfrak U_s[f_a](\eta_s)\langle f_a,\eta_s\rangle\di s&=-2\kappa\int_0^t \int_{\mathsf X\times\mathsf Y}\langle\phi(\cdot,\cdot,x)-y,\bar\mu_s\otimes\gamma\rangle\sum_{a\ge1}\langle\mathbf T_x f_a,\eta_s\rangle\langle f_a,\eta_s\rangle\pi(\di x,\di y)\di s\\
&= -2\kappa\int_0^t \int_{\mathsf X\times\mathsf Y}\langle\phi(\cdot,\cdot,x)-y,\bar\mu_s\otimes\gamma\rangle\langle\eta_s,\mathbf T_x^*\eta_s\rangle_{\mathcal H^{-\mathfrak J_3,\mathfrak j_3-1}}\pi(\di x,\di y)\di s\\
&\le C\int_0^t\|\eta_s\|_{\mathcal H^{-\mathfrak J_3,\mathfrak j_3-1}}^2\di s.
\end{align*}
Using the operator $\mathbf T: f\in \mathcal H^{\mathfrak J_3,\mathfrak j_3-1}(\mathbf R^{d+1})\mapsto \nabla_\theta f\cdot\nabla_\theta\mathscr D_{\mathrm{KL}}(q^1_\cdot|P_0^1)\in \mathcal H^{\mathfrak J_3-1,\mathfrak j_3}(\mathbf R^{d+1})$ (see \eqref{operator-T}) together with Lemma \ref{lem_B1_avec_poids}, we obtain 
\begin{align*}
\sum_{a\ge1}-2\int_0^t\mathfrak W_s[f_a](\eta_s)\langle f_a,\eta_s\rangle\di s = -2\kappa\int_0^t\sum_{a\ge1}\langle\mathbf  T f_a,\eta_s\rangle\langle f_a,\eta_s\rangle\di s&=-2\kappa\int_0^t\langle \eta_s,\mathbf T^*\eta_s\rangle_{\mathcal H^{-\mathfrak J_3,\mathfrak j_3-1}}\di s\\
\le C\int_0^t\|\eta_s\|_{\mathcal H^{-\mathfrak J_3,\mathfrak j_3-1}}^2\di s
\end{align*}
Hence,  using \eqref{f,eta^2-uniqueness}, one deduces that a.s. for all $t\in[0,T]$,  
\begin{align*}
\|\eta_t\|_{\mathcal H^{-\mathfrak J_3,\mathfrak j_3-1}}^2=\sum_{a\ge1}\langle f_a,\eta_t\rangle^2\le C\int_0^t\|\eta_s\|_{\mathcal H^{-\mathfrak J_3,\mathfrak j_3-1}}^2\di s.
\end{align*}
By Gronwall's lemma, a.s. for all $t\in[0,T]$, $\|\eta_t\|_{\mathcal H^{-\mathfrak J_3,\mathfrak j_3-1}}=0$.  This concludes the proof of Proposition \ref{p-pathwise-uniqu}. 
\end{proof}

 We are now in position to conclude the proof of   Theorem \ref{thm-clt-ideal}.

\begin{proof}[Proof of Theorem \ref{thm-clt-ideal}]
Let us consider the case when the $\theta_k^i$'s are generated by the algorithm \eqref{eq.algo-batch} (the proofs of Theorem \ref{thm-clt-ideal} are exactly the same when they are generated by the algorithms  \eqref{eq.algo-ideal} or the algorithm~\eqref{eq.algo-z1z2}). By Proposition \ref{p-rc-eta}, $(\eta^{N})$ admits a limit point. Assume that it admits two limit points.  
Let $\ell\in\{1,2\}$ and $N_\ell$ be such that in distribution $\eta^{N_\ell}\to\eta^\ell$ in $\mathcal D(\mathbf R_+,\mathcal H^{-\mathfrak J_3+1,\mathfrak j_3}(\mathbf R^{d+1}))$. Recall that from Lemma \ref{lem-reg-LP}, we have a.s. $\eta^\ell\in\mathcal C(\mathbf R_+,\mathcal H^{-\mathfrak J_3+1,\mathfrak j_3}(\mathbf R^{d+1}))$. Let us now consider a limit point  $(\eta^{\ell,\star},\mathscr G^{\ell,\star})$ of  $(\eta^{N_\ell},\sqrt{N_\ell}\mathbf M^{N_\ell})$ in $\mathscr E$ (see \eqref{def-E}).  
Up to extracting a subsequence from $N_\ell$, we assume 
$$(\eta^{N_\ell},\sqrt{N_\ell}\mathbf M^{N_\ell})\xrightarrow[N_\ell\to\infty]{\mathscr L}(\eta^{\ell,\star},\mathscr G^{\ell,\star})\text{ in }\mathscr E.$$
Considering the marginal distributions, we then have by uniqueness of the limit in distribution, for $\ell=1,2$,
\begin{equation}\label{eq1-endofproof}
\eta^{\ell,\star}\overset{\mathscr L}{=}\eta^\ell \text{ and } \mathscr G^{\ell,\star}\overset{\mathscr L}{=}\mathscr G.
\end{equation}
where $\mathscr G$ is a G-process given by Proposition \ref{p-conv-to-g-process}. 
Recall also that from Proposition \ref{p-eta*-weak-sol}, both $\eta^{1,\star}$ and $\eta^{2,\star}$ are two weak solutions of {\rm \textbf{(EqL)}} with initial distribution $\nu_0$ (see also Lemma \ref{lem-eta_0}). 
Since strong uniqueness for {\rm \textbf{(EqL)}} (see Proposition \ref{p-pathwise-uniqu}) implies weak uniqueness for {\rm \textbf{(EqL)}}, we deduce that $\eta^{1,\star}=\eta^{2,\star}$ in law. By \eqref{eq1-endofproof}, this implies $\eta^{1}=\eta^{2}$ in law. Consequently, the whole sequence $(\eta^N)_{N\ge1}$ converges in distribution in $\mathcal D(\mathbf R_+,\mathcal H^{-\mathfrak J_3+1,\mathfrak j_3}(\mathbf R^{d+1}))$. Denoting by $\eta^\star$ its limit, we have proved that $\eta^\star$ has the same distribution as the unique weak solution of {\rm \textbf{(EqL)}} with initial distribution~$\nu_0$. The proof Theorem  \ref{thm-clt-ideal} is complete. 
\end{proof}

\end{document}